\documentclass[11pt]{article}
\usepackage[margin=1in]{geometry}
\usepackage{amsmath,amsthm,amssymb,amsfonts,cancel}
\usepackage{thmtools,enumitem,subfigure}
\usepackage{algorithm}
\usepackage[noend]{algpseudocode}
\usepackage{mathtools}
\usepackage{dsfont}
\usepackage{tcolorbox}
\usepackage{bbm}
\usepackage{todonotes}
\usepackage{tensor}
\usepackage{caption}
\usepackage{bm}
\usepackage{caption}
\usepackage{hyperref}
\usepackage{placeins}
\usepackage{authblk}
\usepackage{epsfig}

\newcommand{\RR}{\mathbb{R}}
\renewcommand{\S}{\mathcal{S}}
\renewcommand{\P}{\mathcal{P}}
\newcommand{\D}{\mathcal{D}}
\newcommand{\A}{\mathcal{A}}

\newcommand{\diam}[1]{\operatorname{diam}\,\left( #1 \right)}
\newcommand{\dom}[1]{\texttt{dom}\,\left( #1 \right)}
\newcommand{\Ind}[1]{\mathds{1}_{\left[ #1 \right]}}

\newcommand{\Exp}[1]{\mathbb E \left[ #1 \right]} 

\renewcommand{\Pr}{\mathbb{P}}
\newcommand{\pfail}{\delta}
\newcommand{\Qhat}[2]{\mathbf{Q}_{#1}^{#2}}
\newcommand{\Vhat}[2]{\mathbf{V}_{#1}^{#2}}
\newcommand\numberthis{\addtocounter{equation}{1}\tag{\theequation}}

\mathchardef\mhyphen="2D 

\DeclareMathOperator*{\argmax}{\arg\!\max}
\DeclarePairedDelimiter{\norm}{\lVert}{\rVert}

\let\originalleft\left
\let\originalright\right
\renewcommand{\left}{\mathopen{}\mathclose\bgroup\originalleft}
\renewcommand{\right}{\aftergroup\egroup\originalright}
\newcommand\blfootnote[1]{%
  \begingroup
  \renewcommand\thefootnote{}\footnote{#1}%
  \addtocounter{footnote}{-1}%
  \endgroup
}

\newtheorem{assumption}{Assumption}

\newtheorem{theorem}{Theorem}
\numberwithin{theorem}{section}
\newtheorem{definition}[theorem]{Definition}

\newtheorem{lemma}[theorem]{Lemma}

\newtheorem{proposition}[theorem]{Proposition}
\newtheorem{corollary}[theorem]{Corollary}

\begin{document}

\title{Adaptive Discretization for Episodic Reinforcement Learning in Metric Spaces}
\author{Sean R. Sinclair\thanks{Email: \texttt{srs429@cornell.edu}} \qquad Siddhartha Banerjee\thanks{Email: \texttt{sbanerjee@cornell.edu}} \qquad Christina Lee Yu\thanks{Email: \texttt{cleeyu@cornell.edu}} \\
	Cornell University}
\date{}
	\maketitle

	\begin{abstract}
We present an efficient algorithm for model-free episodic reinforcement learning on large (potentially continuous) state-action spaces. Our algorithm is based on a novel $Q$-learning policy with adaptive data-driven discretization. The central idea is to maintain a finer partition of the state-action space in regions which are frequently visited in historical trajectories, and have higher payoff estimates. We demonstrate how our adaptive partitions take advantage of the shape of the optimal $Q$-function and the joint space, without sacrificing the worst-case performance. In particular, we recover the regret guarantees of prior algorithms for continuous state-action spaces, which additionally require either an optimal discretization as input, and/or access to a simulation oracle. Moreover, experiments demonstrate how our algorithm automatically adapts to the underlying structure of the problem, resulting in much better performance compared both to heuristics and $Q$-learning with uniform discretization.\blfootnote{The code for the experiments is available at \url{https://github.com/seanrsinclair/AdaptiveQLearning}.}
	\end{abstract}
	\newpage
	\setcounter{tocdepth}{2}
	\tableofcontents
	\newpage

\section{Introduction}

Reinforcement learning (RL) is a natural model for systems involving real-time sequential decision making~\cite{sutton2018reinforcement}.  An agent interacts with a system having stochastic transitions and rewards, and aims to learn to control the system by exploring available actions and using real-time feedback.  This requires the agent to navigate the \textit{exploration exploitation trade-off}, between exploring unseen parts of the environment and exploiting historical high-reward actions.
In addition, many RL problems involve large state-action spaces, which makes learning and storing the entire transition kernel infeasible (for example, in memory-constrained devices). This motivates the use of \emph{model-free RL algorithms}, which eschew learning transitions and focus only on learning good state-action mappings. The most popular of these algorithms is \emph{Q-learning}~\cite{azar_2017,jin_is_2018, watkins1989learning}, which forms the focus of our work.

In even higher-dimensional state-spaces, in particular, continuous spaces, RL algorithms require embedding the setting in some metric space, and then using an appropriate discretization of the space. A major challenge here is in learning an ``optimal'' discretization, trading-off memory requirements and algorithm performance. 
{Moreover, unlike optimal quantization problems in `offline' settings (i.e., where the full problem is specified), there is an additional challenge of learning a good discretization and control policy when the process of learning itself must also be constrained to the available memory.}

This motivates our central question:
\begin{center}
    \textit{Can we modify $Q$-learning to learn a near-optimal policy while limiting the size of the discretization?}
\end{center}

Current approaches to this problem consider uniform discretization policies, which are either fixed based on problem primitives, or updated via a fixed schedule (for example, via a `doubling trick'). 
However, a more natural approach is to adapt the discretization over space and time in a data-driven manner. This allows the algorithm to learn policies which are not uniformly smooth, but adapt to the geometry of the underlying space.  Moreover, the agent would then be able to explore more efficiently by only sampling important regions. 

Adaptive discretization has been proposed and studied in the simpler multi-armed bandit settings~\cite{Kleinberg:2019:BEM:3338848.3299873, slivkins_contextual_2015}. The key idea here is to develop a non-uniform partitioning of the space, whose coarseness depends on the density of past observations. These techniques, however, do not immediately extend to RL, with the major challenge being in dealing with error propagation over periods. In more detail, in bandit settings, an algorithm's \emph{regret} (i.e., additive loss from the optimal policy) can be decomposed in straightforward ways, so as to isolate errors and control their propagation.  Since errors can propagate in complex ways over sample paths, naive discretization could result in over-partitioning suboptimal regions of the space (leading to over exploration), or not discretizing enough in the optimal region due to noisy samples (leading to loss in exploitation). Our work takes an important step towards tackling these issues. 

{Adaptive partitioning for reinforcement learning makes progress towards addressing the challenge of limited memory for real-time control problems. In particular, we are motivated by considering the use of RL for computing systems problems such as memory management, resource allocation, and load balancing \cite{Mao:2016, Comden:2019}. These are applications in which the process of learning the optimal control policy must be implemented directly on-chip due to latency constraints, leading to restrictive memory constraints. Adaptive partitioning finds a more ``efficient'' discretization of the space for the problem instance at hand, reducing the memory requirements. This could have useful applications to many control problems, even with discrete state spaces, as long as the model exhibits ``smoothness'' structure between nearby state-action pairs.}

\subsection{Our Contributions}
\label{section:contributions}

\begin{figure*}[t]
    \centering
    \includegraphics[width=\columnwidth]{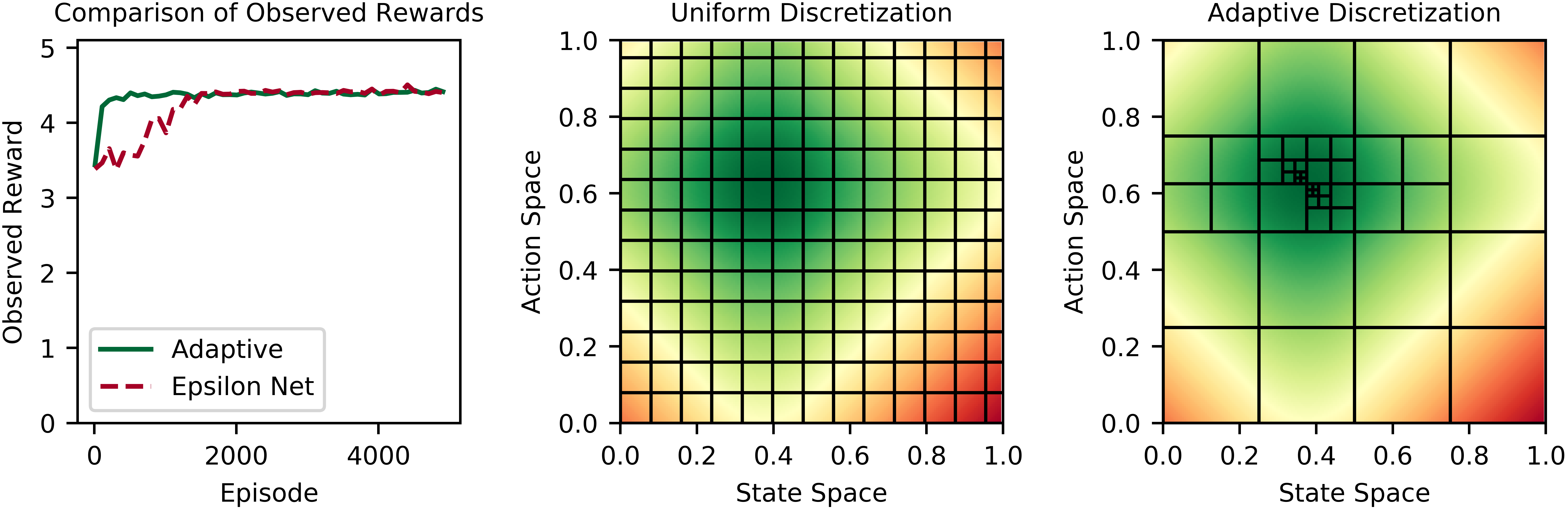}
    \caption{Comparison of the observed rewards and state-action space discretization under the uniform mesh $(\epsilon-$Net) algorithm~\cite{song_efficient_2019} and our adaptive discretization algorithm. Both algorithms were applied to the ambulance routing problem with shifting uniform arrival distributions (see Section~\ref{section:experiments_set_up}). The colors correspond to the relative $Q_h^\star$ value of the given state-action pair, where green corresponds to a higher value for the expected future rewards.  The adaptive algorithm converges faster to the optimal policy by keeping a coarser discretization of unimportant parts of the space and a fine discretization in important parts.}
    \label{fig:partition}
\end{figure*}

{As the main contribution of this paper, we design and analyze a \emph{$Q$-learning algorithm based on data-driven adaptive discretization of the state-action space}.  Our algorithm only requires that the underlying state-action space can be embedded in a compact metric space, and that the optimal $Q$-function is Lipschitz continuous with respect to the metric.  This setting is general, encompassing discrete and continuous state-action spaces, deterministic systems with natural metric structures, and stochastic settings with mild regularity assumptions on the transition kernels. Notably, our algorithm only requires access to the metric, unlike other algorithms which require access to simulation oracles.}  {In addition, our algorithm is \textit{model-free}, requiring less space and computational complexity to learn the optimal policy.}

From a theoretical perspective, we show that our adaptive discretization policy achieves near-optimal dependence of the regret on the covering dimension of the metric space.  {In particular, we prove that over $K$ episodes, our algorithm achieves a regret bound \begin{equation*}
R(K) = \tilde{O}\left(H^{5/2}K^{(d+1)/(d+2)}\right)
\end{equation*}
where $d$ is the covering dimension and $H$ is the number of steps in each episode. 
Moreover, for non-uniform metric spaces where the covering dimension is not tight, we show improved bounds which adapt to the geometry of the space.}

Our algorithm manages the trade-off between exploration and exploitation by careful use of event counters and upper-confidence bounds. 
It then creates finer partitions of regions which have high empirical rewards and/or are visited frequently, to obtain better $Q$-value estimates.  {This reduces the memory requirement of our RL algorithm, as it adapts the discretization of the space to learn the shape of the optimal policy. In addition, it requires less instance dependent tuning, as it only needs to tune the scaling of the confidence bounds. Implementing uniform mesh algorithms, in contrast, also requires tuning the mesh size.}

{We complement our theoretical guarantees with experiments, where we compare our adaptive $Q$-learning algorithm to the net based $Q$-learning algorithm from~\cite{song_efficient_2019} on two canonical problems with continuous spaces.}  Our algorithm achieves order-wise better empirical rewards compared to the uniform mesh algorithm, while maintaining a much smaller partition. Moreover, {the performance gap of our algorithm to the uniform mesh algorithm grow larger with increasing non-uniformity in the underlying $Q$-function}.  {As an example, in Figure~\ref{fig:partition} we demonstrate the performance of our algorithm and net based $Q$-learning for a canonical ambulance routing problem (cf. Section~\ref{section:experiments}).
We see that the adaptive discretization maintains different levels of coarseness across the space, resulting in a faster convergence rate to the optimal policy as compared to the uniform mesh algorithm.}

\subsection{Related Work}
\label{section:related}

Our work sits at the intersection of two lines of work -- on model-free $Q$-learning, and on adaptive zooming for multi-armed bandits. We highlight some of the closest work below; for a more extensive list of references, refer to~\cite{sutton2018reinforcement} (for RL) and~\cite{aleks2019introduction, bubeck2012regret} (for bandits).  There are two popular approaches to RL algorithms: model-free and model-based.

\textbf{Model-based.} Model-based algorithms are based on learning a model for the environment, and use this to learn an optimal policy~\cite{azar_2017, kakade_2003, Ortner2013, Ortner2012, lakshmanan15, osband2014}.  These methods converge in fewer iterations, but have much higher computation and space complexity. As an example, the UCBVI algorithm~\cite{azar_2017} requires storing an estimate of the transition kernel for the MDP, leading to a space complexity of $O(S^2 AH)$ (where $S$ is the number of states, $A$ is the number of actions, and $H$ the number of steps per episode).  Another algorithm for discrete spaces, UCRL \cite{NIPS2008_3401}, and its state-aggregation followup \cite{Ortner2012, Ortner2013}, maintain estimates of the transition kernel and use this for learning the optimal policy.  Other model-based approaches assume the optimal $Q$-function lies in a function class and hence can be found via kernel methods~\cite{yang2019sample, yang_reinforcement_2019}, or that the algorithm has access to an oracle which calculates distributional shifts~\cite{du2019provably}.

There has been some work on developing model-based algorithms for reinforcement learning on metric spaces \cite{osband2014, Ortner2013, lakshmanan15}.  The Posterior Sampling for Reinforcement Learning algorithm \cite{osband2014} uses an adaptation of Thompson sampling, showing regret scaling in terms of the Kolmogorov and eluder dimension. Other algorithms like UCCRL \cite{Ortner2013} and UCCRL-Kernel Density \cite{lakshmanan15} extend UCRL \cite{NIPS2008_3401} to continuous spaces by picking a uniform discretization of the state space and running a discrete algorithm on the discretization with a finite number of actions.  The regret bounds scale in terms of $K^{(2d+1)/(2d+2)}$. Our algorithm achieves better regret, scaling via $K^{(d+1)/(d+2)}$ and works for continuous action spaces.

\smallskip

\textbf{Model-free.}  Our work follows the model-free paradigm of learning the optimal policy {directly from the historical rewards and state trajectories without fitting the model parameters}; these typically have space complexity of $O(SAH)$, which is more amenable for high-dimensional settings or on memory constrained devices.  
The approach most relevant for us is the work on $Q$-learning first started by Watkins \cite{watkins1989learning} and later extended to the discrete model-free setting using upper confidence bounds by Jin et al.~\cite{jin_is_2018}.  They show a regret bound scaling via $\tilde{O}(H^{5/2}\sqrt{SAK})$ where $S$ is the number of states and $A$ is the number of actions.  

These works have since led to numerous extensions, including for infinite horizon time discounted MDPs~\cite{dong2019q}, continuous spaces via uniform $\epsilon$-Nets~\cite{song_efficient_2019}, {and deterministic systems on metric spaces using a function approximator~\cite{yang2019learning}.}  The work by Song et al. assumes the algorithm has access to an optimal $\epsilon-$Net as input, where $\epsilon$ is chosen as a function of the number of episodes and the dimension of the metric space \cite{song_efficient_2019}. Our work differs by adaptively partitioning the environment over the course of learning, only requiring access to a covering oracle as described in Section~\ref{section:assumptions}.  While we recover the same worst-case guarantees, we show an improved covering-type regret bound (Section~\ref{ssec:regret_detailed}).  The experimental results presented in Section~\ref{section:experiments} compare our adaptive algorithm to their net based $Q$-learning algorithm.  We also note that a similar algorithm to ours has been proposed in concurrent and independent work \cite{zhu2019stochastic}.  Our algorithm extends their results to general metric spaces, and we further provide a packing type regret guarantee and simulations.  

The work by Yang et al. for deterministic systems on metric spaces shows regret scaling via $\tilde{O}(HK^{d/(d+1)})$ where $d$ is the doubling dimension \cite{yang2019learning}.  As the doubling dimension is at most the covering dimension, they achieve better regret specialized to deterministic MDPs.  Our work achieves sub-linear regret for stochastic systems as well.

Lastly, there has been work on using $Q$-learning with nearest neighbors \cite{shah2018}.  Their setting considers continuous state spaces but finitely many actions, and analyzes the infinite horizon time-discounted case.  While the regret bounds are generally incomparable (as we consider the finite horizon non-discounted case), we believe that nearest-neighbor approaches can also be used in our setting.  Some preliminary analysis in this regards is in Section~\ref{section:discussion}.

\smallskip

\textbf{Adaptive Partitioning.} The other line of work most relevant to this paper is the literature on adaptive zooming algorithms for multi-armed bandits.  For a general overview on the line of work on regret-minimization for multi-armed bandits we refer the readers to~\cite{lattimore_2018, aleks2019introduction, bubeck2012regret}.  Most relevant to us is the work on bandits with continuous action spaces where there have been numerous algorithms which adaptively partition the space~\cite{ Kleinberg:2019:BEM:3338848.3299873, bubeck2009online}. Slivkins \cite{slivkins_contextual_2015} similarly uses data-driven discretization to adaptively discretize the space.  Our analysis supersedes theirs by generalizing it to reinforcement learning.  Recently, Wang et al.~\cite{wang2019stochbandits} gave general conditions for a partitioning algorithm to achieve regret bounds in contextual bandits.  Our partitioning can also be generalized in a similar way, and the conditions are presented in Section~\ref{section:discussion}.

\subsection{Outline of the paper}
Section~\ref{section:preliminaries} presents preliminaries for the model.  The adaptive $Q$-learning algorithm is explained in Section~\ref{section:algorithm} and the regret bound is given in Section~\ref{section:analysis}.  Section~\ref{section:sketch} presents a proof sketch of the regret bound. Section~\ref{section:experiments} presents numerical experiments of the algorithm.  Proofs are deferred to the appendix.

\section{Preliminaries}
\label{section:preliminaries}

In this paper, we consider an agent interacting with an underlying finite-horizon Markov Decision Processes (MDP) over $K$ sequential episodes, denoted $[K] = \{1, \ldots, K\}$.

The underlying MDP is given by a five-tuple $(\S, \A, H, \Pr, r)$ where $\S$ denotes the set of states, $\A$ the set of actions, and horizon $H$ is the number of steps in each episode. We allow the state-space $\S$ and action-space $\A$ to be large (potentially infinite).
Transitions are governed by a collection $\Pr = \{\Pr_h(\cdot \mid x,a) \mid h \in [H], x \in \S, a \in \A \}$ of transition kernels, where $\Pr_h(\cdot \mid x, a) \in \Delta(\S)$ gives the distribution of states if action $a$ is taken in state $x$ at step $h$, {and $\Delta(\S)$ denotes the set of probability distributions on $\S$}. Finally, the rewards are given by $r = \{r_h \mid h \in [H] \}$, where we assume each $r_h : \S \times \A \rightarrow [0, 1]$ is a deterministic reward function.~\footnote{{This assumption is made for ease of presentation, and can be relaxed by incorporating additional UCB terms for the rewards.}} 

A policy $\pi$ is a sequence of functions $\{ \pi_h \mid h \in [H] \}$ where each $\pi_h : \S \rightarrow \A$ is a mapping from a given state $x \in \S$ to an action $a \in \A$.
At the beginning of each episode $k$, the agent decides on a policy $\pi^k$, and is given an initial state $x_1^k \in \S$ (which can be arbitrary). In each step $h \in [H]$ in the episode, the agent picks the action $\pi^k_h(x_h^k)$, receives reward $r_h(x_h^k, \pi^k_h(x_k^k))$, and transitions to a random state $x_{h+1}^k$ determined by $\Pr_h\left(\cdot \mid x_h^k, \pi^k_h(x_h^k)\right)$.  
This continues until the final transition to state $x_{H+1}^k$, whereupon the agent chooses the policy $\pi^{k+1}$ for the next episode, and the process is repeated.

\subsection{Bellman Equations}

For any policy $\pi$, we use $V_h^\pi: \S \rightarrow \RR$ to denote the value function at step $h$ under policy $\pi$, i.e., $V_h^\pi(x)$ gives the expected sum of future rewards under policy $\pi$ starting from $x_h = x$ in step $h$ until the end of the episode,
$$V_h^\pi(x) := \Exp{\sum_{h'=h}^H r_{h'}(x_{h'}, \pi_{h'}(x_{h'})) ~\Big|~ x_h = x}.$$

We refer to $Q_h^\pi : \S \times \A \rightarrow \RR$ as the $Q$-value function at step $h$, where $Q_h^\pi(x,a)$ is equal to the sum of $r_h(x,a)$
and the expected future rewards received for playing policy $\pi$ in all subsequent steps of the episode after taking action $a_h = a$ at step $h$ from state $x_h = x$,
$$Q_h^\pi(x,a) := r_h(x,a) + \Exp{\sum_{h'=h+1}^H r_{h'}(x_{h'}, \pi_{h'}(x_{h'})) ~\Big|~ x_h = x, a_h = a}.$$

Under suitable conditions on $\S \times \A$ and the reward function, there exists an optimal policy $\pi^\star$ which gives the optimal value $V_h^\star(x) = \sup_\pi V_h^\pi(x)$ for all $x \in \S$ and $h \in [H]$.  For simplicity and ease of notation we denote $\Exp{V_{h+1}(\hat{x}) | x, a} := \mathbb{E}_{\hat{x} \sim \Pr_h(\cdot \mid x,a)}[V_{h+1}(\hat{x})]$ and set $Q^\star = Q^{\pi^\star}$.  We recall the Bellman equations which state that \cite{puterman_1994}:
\begin{equation}
\label{eqn:bellman_equation}
\left\{\begin{array}{l}{V_{h}^{\pi}(x)=Q_{h}^{\pi}\left(x, \pi_{h}(x)\right)} \\ {Q_{h}^{\pi}(x, a) = r_{h}(x, a) + \Exp{V_{h+1}^{\pi}(\hat{x}) \mid x, a}} \\ {V_{H+1}^{\pi}(x)=0 \quad \forall x \in \mathcal{S}.}\end{array}\right.
\end{equation}
The optimality conditions are similar, where in addition we have $V_h^\star(x) = \max_{a \in \A} Q_h^\star(x,a).$

The agent plays the game for $K$ episodes $k = 1, \ldots, K$.  For each episode $k$ the agent selects a policy $\pi^k$ and the adversary picks the starting state $x_1^k$.  The goal of the agent is to maximize her total expected reward $\sum_{k=1}^K V_1^{\pi^k}(x_1^k)$.  Similar to the benchmarks used in conventional multi-armed bandits, the agent instead attempts to minimize her \textit{regret}, the expected loss the agent experiences by exercising her policy $\pi^k$ instead of an optimal policy $\pi^\star$ in every episode.  This is defined via: 
\begin{equation}\label{equation:regret}R(K) = \sum_{k=1}^K \left( V_1^\star(x_1^k) - V_1^{\pi^k}(x_1^k) \right).\end{equation}

Our goal is to show that $R(K) \in o(K)$.  The regret bounds presented in the paper scale in terms of {$K^{(d+1)/(d+2)}$} where $d$ is a type of dimension of the metric space.  We begin the next section by outlining the relevant metric space properties.

\subsection{Packing and Covering}

Covering dimensions and other notions of dimension of a metric space will be a crucial aspect of the regret bound for the algorithm.  The ability to adaptively cover the space while minimizing the number of balls required will be a tenant in the adaptive $Q$-learning algorithm.  {Following the notation by Kleinberg, Slivkins, and Upfal \cite{Kleinberg:2019:BEM:3338848.3299873}, let $(X, \D)$ be a metric space, and $r > 0$ be arbitrary (in the regret bounds $r$ will be taken as the radius of a ball).}  We first note that the diameter of a set $B$ is $\diam{B} = \sup_{x,y \in B} \D(x,y)$ and a ball with center $x$ and radius $r$ is denoted by $B(x,r) = \{y \in X : \D(x,y) < r \}.$  We denote by $d_{max} = \diam{X}$ to be the diameter of the entire space.

\begin{definition}
An \textit{$r$-covering} of $X$ is a collection of subsets of $X$, which cover $X$, and each of which has diameter strictly less than $r$.  The minimal number of subsets in an $r$-covering is called the \textit{r-covering number} of $\P$ and is denoted by $N_r$.
\end{definition}
\begin{definition}
\label{definition:net}
{A set of points $\P$} is an \textit{$r$-packing} if the distance between any points in $\P$ is at least $r$.  An \textit{$r$-Net} of the metric space is an $r$-packing where $\bigcup_{x \in \P} B(x, r)$ covers the entire space $X$.
\end{definition}

As an aside, the Net based $Q$-learning algorithm requires an $\epsilon$-Net of the state action space $\S \times \A$ given as input to the algorithm \cite{song_efficient_2019}.

The last definition will be used for a more interpretable regret bound.  {It is also used to bound the size of the adaptive partition generated by the adaptive $Q$-learning algorithm.}  This is used as a dimension of general metric spaces.
\begin{definition}
The \textit{covering dimension} with parameter c induced by the packing numbers $N_r$ is defined as $$d_c = \inf \{ d \geq 0 \mid N_r \leq cr^{-d} \,\, \forall r \in (0, d_{max}]\}.$$
\end{definition}

For any set of finite diameter, the covering dimension is at most the doubling dimension, which is at most $d$ for any set in $(\mathbb{R}^d, \ell_p)$.  However, there are some spaces and metrics where the covering dimensions can be much smaller than the dimension of the entire space \cite{Kleinberg:2019:BEM:3338848.3299873}.

All of these notions of covering are highly related, in fact there are even more definitions of dimensions (including the doubling dimension) through which the regret bound can be formulated {(see Section 3 in \cite{Kleinberg:2019:BEM:3338848.3299873}).}
\subsection{Assumptions}
\label{section:assumptions}

In this section we state and explain the assumptions used throughout the rest of the paper.
We assume that there exists a metric $ \D: (\S \times \A)^2 \rightarrow \RR_+$ so that $\S \times \A$ is a metric space~\footnote{$\S \times \A$ can also be a product metric space, where $\S$ and $\A$ are metric spaces individually and the metric on $\S \times \A$ is a product metric.}.  To make the problem tractable we consider several assumptions which are common throughout the literature.
\begin{assumption}
\label{assumption:bounded}
	$\S \times \A$ has finite diameter with respect to the metric $\D$, namely that $$\diam{\S \times \A} \leq d_{max}.$$
\end{assumption}
This assumption allows us to maintain a partition of $\S \times \A$.  Indeed, for any point $(x, a)$ in $\S \times \A$, the ball centered at $(x, a)$ with radius $d_{max}$ covers all of $\S \times \A$.  {This is also assumed in other research on reinforcement learning in metric spaces where they set $d_{max} = 1$ by re-scaling the metric.}

\begin{assumption}
\label{assumption:Lipschitz}
	For every $h \in [H]$, $Q_h^\star$ is $L$-Lipschitz continuous with respect to $\D$, i.e. for all $(x,a), (x',a') \in \S \times \A,$ $$|Q_h^\star(x,a) - Q_h^\star(x', a')| \leq L\D((x,a), (x', a')).$$
\end{assumption}

Assumption~\ref{assumption:Lipschitz} implies that the $Q_h^\star$ value of nearby state action pairs are close.  This motivates the discretization technique as points nearby will have similar $Q_h^\star$ values and hence can be estimated together. {Requiring the $Q_h^\star$ function to be Lipschitz may seem less interpretable compared to making assumptions on the problem primitives; however, we demonstrate below that natural continuity assumptions on the MDP translate into this condition (cf. Appendix~\ref{sec:assumptions} for details).}

\begin{proposition}
	Suppose that the transition kernel is Lipschitz with respect to total variation distance and the reward function is Lipschitz continuous, i.e.
	\begin{align*}
	\norm{\Pr_h(\cdot \mid x,a) - \Pr_h(\cdot \mid x', a')}_{TV} & \leq L_1 \D((x,a), (x',a')) \text{ and } \\
	|r_h(x,a) - r_h(x', a')| & \leq L_2 \D((x,a), (x',a'))
	\end{align*}
	for all $(x,a), (x', a') \in \S \times \A$ and $h$.  Then it follows that $Q_h^\star$ is also $(2L_1H + L_2)$ Lipschitz continuous.
\end{proposition}

This gives conditions without additional assumptions on the space $\S \times \A$.  One downside, however, is for deterministic systems.  Indeed, if the transitions in the MDP were deterministic then the transition kernels $\Pr_h$ would be point masses.  Thus, their total variation distance will be either $0$ or $1$ and will not necessarily be Lipschitz.

Another setting which gives rise to a Lipschitz $Q_h^\star$ function (including deterministic transitions) is seen when $\S$ is in addition a compact separable metric space with metric $d_\S$ and the metric on $\S \times \A$ satisfies $\D((x,a),(x',a)) \leq C d_\S(x, x')$ for some constant $C$ and for all $a \in \A$ and $x, x' \in \S$.  This holds for several common metrics on product spaces.  If $\A$ is also a metric space with metric $d_\A$ then common choices for the product metric,
\begin{align*}
    \D((x,a),(x',a')) & = d_\S(x,x') + d_\A(a,a') \\
    \D((x,a),(x',a')) & = \max\{d_\S(x,x'), d_\A(a,a') \} \\
    \D((x,a),(x',a')) & = \norm{(d_\S(x,x'), d_\A(a,a'))}_p
\end{align*}
all satisfy the property with constant $C = 1$.  With this we can show the following.

\begin{proposition}
Suppose that the transition kernel is Lipschitz with respect to the Wasserstein metric and the reward function is Lipschitz continuous, i.e.
\begin{align*}
    |r_h(x,a) - r_h(x', a')| & \leq L_1 D((x,a),(x',a')) \\
    d_W(\Pr_h(\cdot \mid x,a), \Pr_h(\cdot \mid x', a')) & \leq L_2 D((x,a),(x',a'))
\end{align*}
for all $(x,a), (x', a') \in \S \times \A$ and $h$ and where $d_W$ is the Wasserstein metric.  Then $Q_h^\star$ and $V_h^\star$ are both $(\sum_{i=0}^{H-h} L_1 L_2^{i})$ Lipschitz continuous.
\end{proposition}

Because $\S$ is assumed to be a metric space as well, it follows that $V_h^\star$ is Lipschitz continuous in addition to $Q_h^\star$.  {We also note that the Wasserstein metric is always upper-bounded by the total variation distance, and so Lipschitz with respect to total variation implies Lipschitz with respect to the Wasserstein metric.}  Moreover, this allows Assumption~\ref{assumption:Lipschitz} to hold for deterministic MDPs with Lipschitz transitions.  Indeed, if $g_h(x,a) : \S \times \A \rightarrow \S$ denotes the deterministic transition from taking action $a$ in state $x$ at step $h$ {(so that $\Pr_h(x' \mid x,a) = \Ind{x' = x}$)} then using properties of the Wasserstein metric  we see the following \cite{gibbs2002choosing}. 
\begin{align*}
    d_W(\Pr_h(\cdot \mid x, a), \Pr_h(\cdot \mid x', a')) & = \sup \left\{ \left| \int f \, d\Pr_h(\cdot \mid x, a) - \int f \, d\Pr_h(\cdot \mid x', a')\right| : \norm{f}_L \leq 1 \right\} \\
    & = \sup \left\{ \left| f(g_h(x,a)) - f(g_h(x',a')) \right| : \norm{f}_L \leq 1 \right\} \\
    & \leq d_\S(g_h(x,a), g_h(x',a')) \leq L \D((x,a), (x',a'))
\end{align*}
where $\norm{f}_L$ is the smallest Lipschitz condition number of the function $f$ and we used the Lipschitz assumption of the deterministic transitions $g_h(x,a)$.

The {next} assumption is similar to that expressed in the literature on adaptive zooming algorithms for multi-armed bandits \cite{slivkins_contextual_2015, Kleinberg:2019:BEM:3338848.3299873}.  This assumes unrestricted access to the similarity metric $\D$.  The question of representing the metric, learning the metric, and picking a metric are important in practice, but beyond the scope of this paper \cite{nir2019nonparametric}. We will assume \textit{oracle access} to the similarity metric via specific queries.

\begin{assumption}
	The agent has oracle access to the similarity metric $\D$ via several queries that are used by the algorithm.
\end{assumption}

{The Adaptive $Q-$Learning algorithm presented (Algorithm~\ref{alg}) requires only a \textit{covering oracle} which takes a finite collection of balls and a set $X$ and either declares that they cover $X$ or outputs an uncovered point.  The algorithm then poses at most one oracle call in each round.  An alternative assumption is to assume the \textit{covering oracle} is able to take a set $X$ and value $r$ and output an $r$ packing of $X$.  In practice, this can be implemented in several metric spaces (e.g. Euclidean spaces).  Alternative approaches using arbitrary partitioning schemes (e.g. decision trees, etc) are presented in Section~\ref{section:discussion}.  Implementation details of the algorithm in this setting will be explained in Section~\ref{section:experiments}.}
\section{Algorithm}
\label{section:algorithm}

Our algorithm is parameterized by the number of episodes $K$ and a value $\pfail \in (0, 1)$ related to the high-probability regret bound.~\footnote{Knowledge of the number of episodes $K$ can be relaxed by allowing the algorithm to proceed in phases via the doubling trick.}  {This algorithm falls under ``Upper Confidence Bound'' algorithms popular in multi-armed bandits \cite{bubeck2012regret, lattimore_2018} as the selection rule is greedy with respect to estimates of the $Q_h^\star$ function.}  For each step $h = 1, \ldots, H$ it maintains a collection of balls $\P_h^k$ of $\S \times \A$ which is refined over the course learning for each episode $k \in [K]$.  Each element $B \in \P_h^k$ is a ball with radius $r(B)$.  Initially, when $k = 1$, there is only one ball in each partition $\P_h^1$ which has radius $d_{max}$ and contains the entire state-action space by Assumption~\ref{assumption:bounded}.  A sample of the partition resulting from our adaptive discretization of the space $\S = [0,1], \A = [0,1]$ can be seen in Figure~\ref{fig:partition} with the metric $\D((x,a),(x',a')) = \max \{|x-x'|, |a-a'|\}$.

\begin{algorithm*}[t!]
	\caption{Adaptive $Q$-Learning}
	\label{alg}
	\begin{algorithmic}[1]
		\Procedure{Adaptive $Q$-Learning}{$\S, \A, \D, H, K, \pfail$}
			\State Initiate $H$ partitions $\P_h^1$ for $h = 1, \ldots, H$ each containing a single ball with radius $d_{max}$ and $\Qhat{h}{1}$ estimate $H$
			\For{each episode $k \gets 1, \ldots K$}
				\State Receive initial state $x_1^k$
				\For{each step $h \gets 1, \ldots, H$}
					\State Select the ball $B_{sel}$ by the selection rule $B_{sel} = \argmax_{B \in \texttt{RELEVANT}_h^k(x_h^k)} \Qhat{h}{k}(B)$
					\State Select action $a_h^k = a$ for some $(x_h^k, a) \in \dom{B_{sel}}$
					\State Play action $a_h^k$, receive reward $r_h^k$ and transition to new state $x_{h+1}^k$
					\State Update Parameters: $t = n_h^{k+1}(B_{sel}) \gets n_h^{k}(B_{sel}) + 1$ \\ 
					\State $\Qhat{h}{k+1}(B_{sel}) \gets (1 - \alpha_t)\Qhat{h}{k}(B_{sel}) + \alpha_t(r_h^k + b(t) + \Vhat{h+1}{k}(x_{h+1}^k))$ where \\
					\State $\Vhat{h+1}{k}(x_{h+1}^k) = \min(H, \max_{B \in \text{RELEVANT}_{h+1}^k(x_{h+1}^k)} \Qhat{h+1}{k}(B))$ (see Section~\ref{section:algorithm})
					\If{$n_h^{k+1}(B_{sel}) \geq \left( \frac{d_{max}}{r(B_{sel})}\right)^2$}
					  \textproc{Split Ball}$(B_{sel}, h, k)$
					\EndIf
				\EndFor
			\EndFor
		\EndProcedure
		\Procedure{Split Ball}{$B$, $h$, $k$}
		    \State Set $B_1, \ldots B_n$ to be an $\frac{1}{2}r(B)$-packing of $\dom{B}$, and add each ball to the partition $\P_h^{k+1}$ (see Definition~\ref{definition:net})
		    \State Initialize parameters $\Qhat{h}{k+1}(B_i)$ and $n_h^{k+1}(B_i)$ for each new ball $B_i$ to inherent values from the parent ball $B$
		 \EndProcedure
	\end{algorithmic}
\end{algorithm*}

In comparison to the previous literature where the algorithm takes an optimal $\epsilon$-Net as input, our algorithm refines the partition $\P_h^k$ in a data-driven manner.  In each iteration, our algorithm performs three steps: select a ball via the \textit{selection rule}, \textit{update parameters}, and \textit{re-partition} the space.

For every episode $k$ and step $h$ the algorithm maintains two tables with size linear with respect to the number of balls in the partition $\P_h^k$.  For any ball $B \in \P_h^k$ we maintain an \textit{upper confidence value} $\Qhat{h}{k}(B)$ for the true $Q_h^\star$ value for points in $B$ and $n_h^k(B)$ for the number of times $B$ or its ancestors have been selected by the algorithm at step $h$ in episodes up to $k$.  This is incremented every time $B$ is played. The counter will be used to construct the bonus term in updating the $Q$ estimates and also for determining when to split a ball. We set the learning rate as follows:
\begin{equation}\label{eqn:learning_rate}
\alpha_t = \frac{H+1}{H+t}
\end{equation}
These learning rates are based on the the algorithm in~\cite{jin_is_2018}, and are chosen to satisfy certain conditions, captured via the following lemma (Lemma 4.1 from~\cite{jin_is_2018}).
\begin{lemma}
\label{lemma:lr}
Let $\alpha_t^i \triangleq a_i \prod_{j=i+1}^t (1 - \alpha_j)$. Then $\{\alpha_t\}_{i \leq t}$ satisfy:
\begin{enumerate}
	\item $\sum_{i=1}^t \alpha_t^i = 1$, $\max_{i \in [t]} \alpha_t^i \leq \frac{2H}{t}$ and $\sum_{i=1}^t (\alpha_t^i)^2 \leq \frac{2H}{t}$ for every $t \geq 1$
	\item $\frac{1}{\sqrt{t}} \leq \sum_{i=1}^t \frac{\alpha_t^i}{\sqrt{t}} \leq \frac{2}{\sqrt{t}}$ for every $t \geq 1$
	\item $\sum_{t=i}^\infty \alpha_t^i = 1 + \frac{1}{H}$ for every $i \geq 1$.
\end{enumerate}
\end{lemma}{These properties will be important in the proof and we will highlight them as they come up in the proof sketch.}

At a high level the algorithm proceeds as follows. In each episode $k$ and step $h$, a state $x_h^k$ is observed.  The algorithm selects an action according to the \textit{selection rule} by picking a relevant ball $B$ in $\P_h^k$ which has maximum upper confidence value $\Qhat{h}{k}(B)$ and taking an action in that ball.  Next, the algorithm updates the estimates for $\Qhat{h}{k}(B)$ by \textit{updating parameters} and lastly \textit{re-partitions} the state-action space.

In order to define the three steps (selection rule, updating parameters, and re-partitioning) we need to introduce some definitions and notation.  Fix an episode $k$, step $h$, and ball $B \in \P_h^k$.  Let $t = n_h^k(B)$ be the number of times $B$ or its ancestors have been selected by the algorithm at step $h$ in episodes up to the current episode $k$.  The \textit{confidence radius} or bonus of ball $B$ is
\begin{equation}\label{eqn:conf_radius}b(t) = 2 \sqrt{\frac{H^3 \log(4HK/\pfail)}{t}} + \frac{4Ld_{max}}{\sqrt{t}}.
\end{equation}
{The first term in Equation~\ref{eqn:conf_radius} corresponds to the uncertainty in the current estimate of the $Q$ value due to the stochastic nature of the transitions.  The second term corresponds to the discretization error by expanding the estimate to all points in the same ball.  If in addition the rewards were stochastic, there would be a third term to include the confidence in reward estimates.}

The \textit{domain} of a ball $B$ is a subset of $B$ which excludes all balls $B' \in \P_h^k$ of a strictly smaller radius $$\dom{B} = B \setminus \left( \cup_{B' \in \P_h^k : r(B') < r(B)} B'\right).$$  {The domain of the balls in $\P_h^k$ will cover the entire space $\S \times \A$ and be used in the algorithm as a partition of the space.}  A ball $B$ is then \textit{relevant} in episode $k$ step $h$ for a point $x \in \S$ if $(x, a) \in \dom{B}$ for some $a \in \A$.  The set of all relevant balls for a point is denoted by $\texttt{RELEVANT}_h^k(x)$.  In each round the algorithm selects one relevant ball $B$ for the current state $x_h^k$ and plays an action $a$  in that ball. After subsequently observing the reward $r_h = r(x_h^k, a)$  we increment $t = n_h^{k+1}(B) = n_h^k(B) + 1$, and perform the $Q$-learning update according to
\begin{align}
\label{eqn:update}
\Qhat{h}{k+1}(B) = (1 - \alpha_t)\Qhat{h}{k}(B) + \alpha_t(r_h^k + \Vhat{h+1}{k}(x_{new}) + b(t))
\end{align}
where $r_h^k$ is the observed reward, $x_{new}$ is the state the agent transitions to, and \begin{align}
\label{eqn:v_update}
\Vhat{h+1}{k}(x) = \min(H, \max_{B \in \texttt{RELEVANT}_{h+1}^k(x)} \Qhat{h+1}{k}(B))
\end{align} {is our estimate of the expected future reward for being in a given state.}  Let $(x_h^k, a_h^k)$ be the state action pair observed in episode $k$ step $h$ by the algorithm. Then the three rules are defined as follows
\begin{itemize}
	\item \textbf{selection rule}: Select a relevant ball $B$ for $x_h^k$ with maximal value of $\Qhat{h}{k}(B)$ (breaking ties arbitrarily).  Select any action $a$ to play such that $(x_h^k, a) \in \dom{B}$.  This is similar to the greedy ``upper confidence algorithms'' for multi-armed bandits.
	\item \textbf{update parameters}: Increment $n_h^k(B)$ by $1$, and update the $\Qhat{h}{k}(B)$ value for the selected ball given the observed reward according to Equation~\ref{eqn:update}.
	\item \textbf{re-partition the space}: Let $B$ denote the selected ball and $r(B)$ denote its radius. We split when $n_h^{k+1}(B) \geq \left(d_{max}/r(B)\right)^2$.  {We then cover $\dom{B}$ with new balls $B_1, \ldots, B_n$ which form an $\frac{1}{2}r(B)$-Net of $\dom{B}$.}  We call $B$ the \textit{parent} of these new balls and each child ball inherits all values from its parent.  {We then add the new balls $B_1, \ldots, B_n$ to $\P_h^k$ to form the partition for the next episode $\P_h^{k+1}$}.~\footnote{Instead of covering the parent ball each time it is ``split'', we can instead introduce children balls as needed in a greedy fashion until the parent ball is covered. When $B$ is first split, we create a single new ball with center $(x,a)$ and radius $\frac{1}{2}r(B)$ where $x$ and $a$ are the current state and action performed. At every subsequent time the parent ball $B$ is selected where the current state and action $(\hat{x},\hat{a}) \in \dom{B}$, then we create a new ball again with center $(\hat{x},\hat{a})$ and radius $\frac{1}{2}r(B)$, which removes this set from $\dom{B}$.}
\end{itemize}

See Algorithm~\ref{alg} for the pseudocode.  The full version of the pseudocode is in Appendix~\ref{sec:complete_algorithm}.  As a reference, see Table~\ref{table:notation} in the appendix for a list of notation used.

\section{Performance Guarantees}
\label{section:analysis}

{We provide three main forms of performance guarantees: uniform (i.e., \emph{worst-case}) regret bounds with arbitrary starting states, refined metric-specific regret bounds, 
and sample-complexity guarantees for learning a policy.}  {We close the section with a lower-bound analysis of these results, showing that our results are optimal up to logarithmic factors and a factor of $H^2$.}

\subsection{Worst-Case Regret Guarantees}
\label{ssec:regretbasic}

We provide regret guarantees for Algorithm~\ref{alg}.  First recall the definition of the covering number with parameter $c$ as \begin{equation}\label{eqn:covering_dimension}d_c = \inf \{d \geq 0 : N_r \leq cr^{-d} \,\,  \forall r \in (0, d_{max}] \}.
\end{equation}

We show the regret scales as follows:
\begin{theorem}
\label{thm:regret}
For any {any sequence of initial states $\{x_1^k \mid k\in[K]\}$}, and any
$\pfail \in (0, 1)$ with probability at least $1 - \pfail$ Adaptive $Q$-learning (Alg~\ref{alg}) achieves regret guarantee:
\begin{align*}
R(K) \leq & \, 3H^2 + 6 \sqrt{2H^3 K \log(4HK/\pfail)} \\
& + \gamma H K^{(d_c+1)/(d_c+2)}  \left( \sqrt{H^3 \log(4HK/\pfail)} + L d_{max} \right)\\ & = \tilde{O}\left(H^{5/2}K^{(d_c+1)/(d_c+2)}\right)  
\end{align*}
where $d_c$ is the covering number of $\S \times \A$ with parameter $c$ and the problem dependent constant {$$\gamma = 192c^{1/(d_c+2)}d_{max}^{-d_c/(d_c+2)}.$$}
\end{theorem}
The regret bound in Theorem~\ref{thm:regret} has three main components. The first term $3H^2$ corresponds to the regret due to actions selected in the first episode and its subsequent impact on future episodes, where we loosely initialized the upper confidence value $\Qhat{}{}$ of each state action pair as $H$.
The second term accounts for the stochastic transitions in the MDP from concentration inequalities.  The third {term} is the discretization error, and comes from the error in discretizing the state action space by the adaptive partition.  As the partition is adaptive, this term scales in terms of the covering number of the entire space. Setting $\pfail = K^{-1/(d_c+2)}$, we get the regret of Algorithm~\ref{alg} as $$\tilde{O}\left(H^{5/2}K^{(d_c+1)/(d_c+2)}\right).$$

{This matches the regret bound from prior work when the metric space $\S \times \A$ is taken to be a subset of $[0,1]^d$, such that the covering dimension is simply $d$, the dimension of the metric space \cite{song_efficient_2019}.  For the case of a discrete state-action space with the discrete metric $\D((s,a), (s',a')) = \mathbbm{1}[s=s', a=a']$, which has a covering dimension $d_c = 0$, we recover the $\tilde{O}(\sqrt{H^5 K})$ bound from discrete episodic RL \cite{jin_is_2018}.}

Our experiments in Section~\ref{section:experiments} shows that the adaptive partitioning saves on time and space complexity in comparison to the fixed $\epsilon-$Net {algorithm} \cite{song_efficient_2019}.  Heuristically, our algorithm achieves better regret while reducing the size of the partition.  We also see from experiments that the regret seems to scale in terms of the covering properties of the shape of the optimal $Q_h^\star$ function {instead of the entire space} similar to the results on {contextual bandits in metric spaces} \cite{slivkins_contextual_2015}.

Previous work on reinforcement learning in metric spaces give lower bounds on episodic reinforcement learning on metric spaces and show that any algorithm must achieve regret where $H$ scales as $H^{3/2}$ and $K$ in terms of $K^{(d_c + 1)/(d_c+2)}$.  Because our algorithm achieves worst case regret $\tilde{O}(K^{(d_c+1)/(d_c+2)} H^{5/2})$ this matches the lower bound up to polylogarithmic factors and a factor of $H$ \cite{song_efficient_2019, slivkins_contextual_2015}.  {More information on the lower bounds is in Section~\ref{ssec:lb}.}

\subsection{Metric-Specific Regret Guarantees}
\label{ssec:regret_detailed}
The regret bound formulated in Theorem~\ref{thm:regret} is a covering guarantee similar to prior work on bandits in metric spaces \cite{Kleinberg:2019:BEM:3338848.3299873, slivkins_contextual_2015}.  This bound suffices for metric spaces where the inequality in the definition of the covering dimensions $N_r \leq cr^{-d}$ is tight; {a canonical example is when $\S \times \A = [0,1]^d$ under the Euclidean metric, where the covering dimension scales as $N_r = \frac{1}{r^d}$}.

{More generally, the guarantee in Theorem~\ref{thm:regret} arises from a more refined bound, wherein we replace the $\gamma K^{(d_c+1)/(d_c+2)}$ factor in the third term of the regret with} 
\begin{equation*}
\inf_{r_0 \in (0, d_{max}]} \left(\frac{Kr_0}{d_{max}} + \sum_{\substack{r = d_{max}2^{-i} \\ r \geq r_0}} N_r \frac{d_{max}}{r} \right).    
\end{equation*}
The bound in Theorem~\ref{thm:regret} is obtained by taking $r_0 = \Theta\left( K^{\frac{-1}{d_c+2}}\right)$ inside of the infimum.

This regret bound gives a packing $N_r$ type guarantee.  Discussion on the scaling is deferred to Section~\ref{ssec:lb}.

\subsection{Policy-Identification Guarantees}
\label{ssec:PAC}

{We can also adapt our algorithm to give sample-complexity guarantees on learning a policy of desired quality. For such a guarantee, assuming that the starting states are adversarially chosen is somewhat pessimistic. A more natural framework here is that of probably approximately correct (PAC) guarantees for learning RL policies~\cite{watkins1989learning}.
Here, we assume that in each episode $k\in[K]$, we have a random initial state $X_1^k \in \S$ drawn from some fixed distribution $F_1$, and try to find the minimum number of episodes needed to find an $\epsilon$-optimal policy $\pi$ with probability at least $1 - \pfail$. 
}

Following similar arguments as~\cite{jin_is_2018}, we can show that
\begin{theorem}
For $K = \tilde{O}\left( \left(H^{5/2}/\pfail\epsilon \right)^{d_c+2} \right)$ (where $d_c$ is the covering dimension with parameter $c$), consider a policy $\pi$ chosen uniformly at random from $\pi_1, \ldots, \pi_K$. Then, for initial state $X\sim F_1$, with probability at least $1-\pfail$, the policy $\pi$ obeys
\begin{equation*}
{V_1^\star(X) - V_1^\pi(X)} \leq \epsilon.
\end{equation*}
\end{theorem}
\noindent Note that in the above guarantee, both $X$ and $\pi$ are random.  The proof is deferred to Appendix~\ref{app:pac_proof}

\subsection{Lower Bounds}
\label{ssec:lb}

Existing lower bounds for this problem have been established previously in the discrete tabular setting \cite{jin_is_2018} and in the contextual bandit literature \cite{slivkins_contextual_2015}.

Jin et al. \cite{jin_is_2018} show the following for discrete spaces.
\begin{theorem}[Paraphrase of Theorem 3 from \cite{jin_is_2018}]
For any algorithm, there exists an $H$-episodic discrete MDP with $S$ states and $A$ actions such that for any $K$, the algorithm's regret is $\Omega(H^{3/2}\sqrt{SAK})$.
\end{theorem}

This shows that our scaling in terms of $H$ is off by a linear factor.  As analyzed in \cite{jin_is_2018}, we believe that using Bernstein's inequality instead of Hoeffding's inequality to better manage the variance of the $\Qhat{h}{k}(B)$ estimates of a ball will allow us to recover $H^{2}$ instead of $H^{5/2}$.

Existing lower bounds for learning in continuous spaces have been established in the contextual bandit literature.  A contextual bandit instance is characterized by a context space $\S$, action space $\A$, and reward function $r: \S \times \A \rightarrow [0,1]$. The agent interacts in rounds, where in each round the agent observes an arbitrary initial context $x$, either drawn from a distribution or specified by an adversary,
and the agent subsequently picks an action $a \in \A$, and receives reward $r(x,a)$.  This is clearly a simplification of an episodic MDP where the number of steps $H=1$ and the transition kernel $\Pr_h(\cdot \mid x, a)$ is independent of $x$ and $a$.  Lower bounds presented in \cite{Kleinberg:2019:BEM:3338848.3299873, slivkins_contextual_2015} show that the scaling in terms of $N_r$ in this general regret bound is optimal.

\begin{theorem}[Paraphrase of Theorem 5.1 from \cite{slivkins_contextual_2015}]
Let $(\S \times \A, \D)$ be an arbitrary metric space satisfying the assumptions in Section~\ref{section:assumptions} with $d_{max} = 1$.  Fix an arbitrary number of episodes $K$ and a positive number $R$ such that $$R \leq C_0 \inf_{r \in (0, 1)} \left( r_0K + \sum_{\substack{r = 2^{-i} \\ r \geq r_0}} \frac{N_r}{r} \log(K)\right).$$  Then there exists a distribution $\mathcal{I}$ over problem instances such that for any algorithm, it holds that $\mathbb{E}_{\mathcal{I}}(R(K)) \geq \Omega(R / \log(K))$.
\end{theorem}

This shows that the scaling in terms of $K$ in Section~\ref{ssec:regret_detailed} is optimal up to logarithmic factors.  Plugging in $r_0 = \Theta(K^{-1/(d_c+2)})$ and exhibiting the dependence on $c$ from the definition of the covering dimension (Equation~\ref{eqn:covering_dimension}) gives that the regret of any algorithm over the distribution of problem instances is at least $\Omega(K^{(d_c + 1)/(d_c+2)} c^{1 / (d_c+2)})$.  This matches the dependence on $K$ and $c$ from Theorem~\ref{thm:regret}.  We can port this lower bound construction over to reinforcement learning by constructing a problem instance with $H$ bandit problems in sequence.  An interesting direction for future work is determining which reinforcement learning problem instances have more suitable structure where we can develop tighter regret bounds.

\section{Proof Sketch}
\label{section:sketch}

{In this section we give a proof sketch for Theorem~\ref{thm:regret}; details are deferred to Appendix~\ref{sec:detailedproof}.  We start with a map of the proof before giving some details.}

{Recall that the algorithm proceeds over $K$ \textit{episodes}, with each episode comprising of $H$ \textit{steps}.  We start by showing that our algorithm is \textit{optimistic} \cite{simchowitz2019nonasymptotic}, which means that with high probability, the estimates maintained by the algorithm are an upper bound on their true values.  This allows us to write the regret in terms of the error in approximating the value function to the true value function for the policy employed on that episode.  Next, using induction we relate the error from a given step in terms of the error from the next step.  Unraveling the relationship and using the fact that the value function for the last step is always zero, we write the regret as the sum of the confidence bound terms from Equation~\ref{eqn:conf_radius}.  We finish by bounding these quantities using properties of the splitting rule from Algorithm~\ref{alg}.  Together, this shows the regret bound established in Theorem~\ref{thm:regret}.}

{Before giving some details of the proof, we start with some notation.  Let $B_h^k$ and $(x_h^k, a_h^k)$ denote the ball and state-action pair selected by the algorithm in episode $k$, step $h$.  We also denote $n_h^k = n_h^k(B_h^k)$ as the number of times the ball $B_h^k$ or its ancestors has been previously played by the algorithm.  The overall regret is then given by (Equation~\ref{equation:regret}):
$$R(K) = \sum_{k=1}^K \left( V_1^\star(x_1^k) - V_1^{\pi^k}(x_1^k)\right).$$ 
To simplify presentation, we denote ($V_h^\star(x_h^k) - V_h^{\pi^k}(x_h^k)$) by $(V_h^\star - V_h^{\pi^k})(x_h^k)$. }

{We start by relating the error in the estimates from step $h$ in terms of the $(h+1)$ step estimates.  The following Lemma establishes this relationship, which shows that with high probability $\Qhat{h}{k}$ is both an upper bound on $Q_h^\star$, and exceeds $Q_h^\star$ by an amount which is bounded as a function of the step $(h+1)$ estimates.  This also shows that our algorithm is \textit{optimistic}.}

\begin{lemma}
\label{lemma:upper_lower_bounds_main}
For any ball $B$, step $h$ and episode $k$, let $t = n_h^k(B)$, and $k_1 < \ldots < k_t$ to be the episodes where $B$ or its ancestors were encountered previously by the algorithm in step $h$. 
Then, for any $\pfail \in (0,1)$, with probability at least $1 - \pfail/2$ the following holds simultaneously for all $(x,a,h,k) \in \S \times \A \times [H] \times [K]$ and ball $B$ such that $(x,a) \in \dom{B}$:
\begin{align*}
&1.\, \Qhat{h}{k}(B) \geq Q_h^\star(x,a) ~\text{ and }~ \Vhat{h}{k}(x) \geq V_h^\star(x)\\
&2.\, \Qhat{h}{k}(B) - Q_h^\star(x,a) \leq \Ind{t = 0}H + \beta_t + \sum_{i=1}^t \alpha_t^i\left(\Vhat{h+1}{k_i} - V_{h+1}^\star\right)(x_{h+1}^{k_i})
\end{align*}
where, for any $t\in[K]$ and $i\leq t$, we define $\alpha_i^t = \alpha_i \prod_{j=i+1}^t (1 - \alpha_j)$ and $\beta_t = 2\sum_{i=1}^t \alpha_i^t b(i)$ (where $\alpha_j$ are the learning rates, and $b(\cdot)$ the confidence radius).
\end{lemma}
In Appendix~\ref{sec:detailedproof}, we provide an expanded version of Lemma \ref{lemma:upper_lower_bounds_main} with a detailed proof (Lemma~\ref{lemma:upper_lower_bounds}).  Below we provide a proof sketch.
The main claim in the Lemma follows from first expanding the recursive update for $\Qhat{h}{k}(B)$ from \eqref{eqn:update}, the property that $\sum_{i=1}^t \alpha_t^i = 1$ from Lemma~\ref{lemma:lr}, and applying definitions from the Bellman Equation (Equation~\ref{eqn:bellman_equation}) to show that 
\begin{align*}
&\Qhat{h}{k}(B) - Q_h^\star(x,a) = \Ind{t = 0}(H - Q_h^\star(x,a)) +  \sum_{i=1}^t \alpha_t^i \left((\Vhat{h+1}{k_i} - V_{h+1}^\star)(x_{h+1}^{k_i})  + b(i) \right)\\
& + \sum_{i=1}^t \alpha_t^i \left(V_{h+1}^\star(x_{h+1}^{k_i}) - \Exp{V_{h+1}^\star(\hat{x}) \mid x_h^{k_i}, a_h^{k_i}}\right) +\sum_{i=1}^t \alpha_t^i\left( Q_h^\star(x_h^{k_i}, a_h^{k_i}) -  Q_h^\star(x,a) \right).
\end{align*}
The last term $\left|\sum_{i=1}^t \alpha_t^i (Q_h^\star(x_h^{k_i}, a_h^{k_i}) -  Q_h^\star(x,a))\right|$ is the error due to the bias induced by discretization, which we bound by $4Ld_{max}/\sqrt{t}$ using Lipschitzness of the $Q$-function. {The second term is the error due to approximating the future value of a state}.  {By using} Azuma-Hoeffding's inequality, we show that with probability at least $1 - \pfail/2$, the error due to approximating the future value by the next state as opposed to computing the expectation is bounded above by
\begin{align*}
\left|\sum_{i=1}^t \alpha_t^i \left( V_{h+1}^\star(x_{h+1}^{k_i}) - \Exp{V_{h+1}^\star(\hat{x}) \mid x_h^{k_i}, a_h^{k_i}}\right) \right|  \leq 2 \sqrt{\frac{H^3 \log(4HK/\pfail)}{t}}.
\end{align*}
The final inequalities in Lemma \ref{lemma:upper_lower_bounds_main} follow by the definition of $b(i), \beta_t$, and substituting the above inequalities to the expression for $\Qhat{h}{k}(B) - Q_h^\star(x,a)$.

By Lemma \ref{lemma:upper_lower_bounds_main}, with high probability, $\Vhat{h}{k}(x) \geq V_h^\star(x)$, such that the terms within the summation of the regret $(V_1^\star(x_1^k) - V_1^{\pi^k}(x_1^k))$ can be upper bounded by $(\Vhat{1}{k} - V_1^{\pi^k})(x_1^k)$ to show that $$R(K) \leq \sum_{k=1}^K \left( (\Vhat{1}{k} - V_1^{\pi^k})(x_1^k)\right).$$ 

The main inductive relationship is stated in Lemma \ref{lemma:induction_main} which writes the errors of these step $h$ estimates in terms of the $(h+1)$ step estimates.
\begin{lemma} \label{lemma:induction_main}
For any $\pfail \in (0,1)$ if $\beta_t = 2\sum_{i=1}^t \alpha_t^i b(i)$ then with probability at least $1 - \pfail/2$, for all $h \in [H]$,
\begin{align*}
    \sum_{k=1}^K (\Vhat{h}{k} - V_h^{\pi^k})(x_h^k) &\leq \sum_{k=1}^K \left(H \Ind{n_h^k = 0} + \beta_{n_h^k} + \xi_{h+1}^k\right) + \left(1 + \frac{1}{H}\right) \sum_{k=1}^K (\Vhat{h+1}{k} - V_{h+1}^{\pi^k})(x_{h+1}^k) \\
& \text{where }~\xi_{h+1}^k = \Exp{V_{h+1}^\star(\hat{x}) - V_{h+1}^{\pi^k}(\hat{x}) ~\big|~ x_h^k, a_h^k} - (V_{h+1}^\star- V_{h+1}^{\pi^k})(x_{h+1}^k).
\end{align*}
\end{lemma}
Expanding this inductive relationship, and using the base case that the value functions at step $H+1$, $\Vhat{H+1}{k}$ and $V_{H+1}^{\pi^k}$, are always zero, it follows that
\begin{align*}
    R(K) 
    & \leq \sum_{h=1}^H\left(1 + \frac{1}{H}\right)^{h-1} \sum_{k=1}^K \left(H \Ind{n_h^k = 0} + \beta_{n_h^k} + \xi_{h+1}^k\right) \\
    & \leq 3 \sum_{h=1}^H \sum_{k=1}^K \left(H \Ind{n_h^k = 0} + \beta_{n_h^k} + \xi_{h+1}^k\right).
\end{align*}

Clearly, $\sum_{k=1}^K H \Ind{n_h^k = 0} = H$ as $n_h^k = 0$ only for the first episode $k = 1$.  For the terms with $\xi_{h+1}^k$ we use a standard martingale analysis with Azuma-Hoeffdings inequality to show in Lemma  \ref{lemma:martingale_difference_sequence} of the Appendix that
$$\sum_{h=1}^H \sum_{k=1}^K \xi_{h+1}^k \leq 6 \sqrt{2H^3 K \log(4HK/\pfail)}.$$
The dominating term in the regret is $\sum_{h=1}^H \sum_{k=1}^K \beta_{n_h^k}$, which captures critical terms in the approximation error of $\Qhat{}{}$. By construction of $b(t)$, it follows that $\beta_t = \Theta(\frac{1}{\sqrt{t}})$ from the second condition in Lemma~\ref{lemma:lr}. Using the following two lemmas we are able to bound this sum.  The first lemma states several invariants maintained by the partitioning scheme.

\begin{lemma}
\label{lemma:partition_main}
For every $(h, k) \in [H] \times [K]$ the following invariants are maintained:
\begin{itemize}
	\item (Covering) The domains of each ball in $\P_h^k$ cover $\S \times \A$.
	\item (Separation) For any two balls of radius $r$, their centers are at distance at least $r$.
\end{itemize}
\end{lemma}
The next lemma states that the number of samples in a given ball grows in terms of its radius.  This is needed to show a notion of ``progress'', where as we get more samples in a ball we further partition it in order to get a more refined estimate for the $Q_h^\star$ value for points in that ball.

\begin{lemma}
	\label{lemma:bound_ball_main}
	For any $h \in [H]$ and child ball $B \in \P_h^K$ (the partition at the end of the last episode $K$) the number of episodes $k \leq K$ such that $B$ is selected by the algorithm is less than $\frac{3}{4}\left(d_{max} / r \right)^2$ where $r = r(B)$.  I.e. denoting by $B_h^k$ the ball selected by the algorithm in step $h$ episode $k$,
	$$| \{ k : B_h^k = B \}| \leq \frac{3}{4} \left(\frac{d_{max}}{r}\right)^2.$$
	
	Moreover, the number of times that ball $B$ and it's ancestors have been played is at least $\frac{1}{4} \left(\frac{d_{max}}{r}\right)^2$.
\end{lemma}

Using these two lemmas we show that $| \P_h^k | \simeq K^{\frac{d_c}{d_c+2}}$.  Hence we get using Jensen's and Cauchy's inequality that:
\begin{align*}
    \sum_{k=1}^K \beta_{n_h^k} & \simeq \sum_{k=1}^K \frac{1}{\sqrt{n_h^k}} \simeq \sum_{B \in \P_h^k} \sum_{k : B_h^k = B} \frac{1}{\sqrt{n_h^k}} \\
    & \simeq \sum_{B \in \P_h^k} \sqrt{|k : B_h^k = B|} \simeq \sqrt{| \P_h^k | K} \\
    & \simeq K^{(d_c+1)/(d_c+2)}.
\end{align*}
Combining these terms gives the regret bound in Theorem~\ref{thm:regret}.

Finally we would like to give some intuition for how the induction relationship presented in Lemma \ref{lemma:induction_main} follows from the key Lemma \ref{lemma:upper_lower_bounds_main}.
Recall that the selection rule in the algorithm enforces that the selected ball $B_h^k$ is the one that maximizes the upper confidence $\Qhat{}{}$-values amongst relevant balls. By definition, it follows that for any $h$ and $k$,
\begin{align*}
    (\Vhat{h}{k} - & V_h^{\pi^k})(x_h^k)
     \leq \max_{B \in \texttt{RELEVANT}_h^k(x_h^k)} \Qhat{h}{k}(B) - Q_h^{\pi^k}(x_h^k, a_h^k) \\
    & = \Qhat{h}{k}(B_h^k) - Q_h^{\pi^k}(x_h^k, a_h^k) \\
    & = \Qhat{h}{k}(B_h^k) - Q_h^\star(x_h^k, a_h^k) + Q_h^\star(x_h^k, a_h^k) - Q_h^{\pi^k}(x_h^k, a_h^k).
\end{align*}
By the definition of the $Q$-function in Equation~\ref{eqn:bellman_equation},
$$Q_h^\star(x_h^k, a_h^k) - Q_h^{\pi^k}(x_h^k, a_h^k) = \Exp{V_{h+1}^\star(\hat{x}) - V_{h+1}^{\pi^k}(\hat{x}) ~\big|~ x_h^k, a_h^k}.$$  We bound $\Qhat{h}{k}(B_h^k) - Q_h^\star(x_h^k, a_h^k)$ by Lemma~\ref{lemma:upper_lower_bounds_main}, as $(x_h^k, a_h^k) \in \dom{B_h^k}$. Putting these bounds together, for $t = n_h^k(B_h^k)$ and for $k_1 < \ldots < k_t$ denoting episodes where $B_h^k$ or its ancestors were previously encountered, it follows that 
\begin{align*}
    (\Vhat{h}{k} - V_h^{\pi^k})(x_h^k)
    & \leq \Ind{t = 0}H + \beta_t + \sum_{i=1}^t \alpha_t^i(\Vhat{h+1}{k_i} - V_{h+1}^\star)(x_{h+1}^{k_i}) + (V_{h+1}^\star- V_{h+1}^{\pi^k})(x_{h+1}^k) + \xi_{h+1}^k,
\end{align*}
where $\xi_{h+1}^k$ is defined in Lemma \ref{lemma:induction_main}. Let us denote $n_h^k = n_h^k(B_h^k)$, and let the respective episodes $k_i(B_h^k)$ denote the time $B_h^k$ or its ancestors were selected for the $i$-th time. By comparing the above inequality to the final inductive relationship, the last inequality we need show is that upon summing over all episodes $k$,
\begin{align*}
&\sum_{k=1}^{K} \sum_{i=1}^{n_h^k} \alpha_{n_h^k}^i(\Vhat{h+1}{k_i(B_h^k)} - V_{h+1}^\star)(x_{h+1}^{k_i(B_h^k)}) + \sum_{k=1}^{K} (V_{h+1}^\star- V_{h+1}^{\pi^k})(x_{h+1}^k) \\
&\leq \left(1 + \frac{1}{H}\right) \sum_{k=1}^K (\Vhat{h+1}{k} - V_{h+1}^{\pi^k})(x_{h+1}^k).
\end{align*}
For every $k' \in [K]$ the term $(\Vhat{h+1}{k'} - V_{h+1}^\star)(x_{h+1}^{k'})$ appears in the summand when $k = n_h^{k'}$. The next time it appears when $k = n_h^{k'} + 1$ and so on. By rearranging the order of the summation,
\begin{align*}
\sum_{k=1}^K &\sum_{i=1}^{n_h^k} \alpha_{n_h^k}^i (\Vhat{h+1}{k_i(B_h^k)} - V_{h+1}^\star)(x_{h+1}^{k_i(B_h^k)}) \\
&\leq \sum_{k=1}^K (\Vhat{h+1}{k} - V_{h+1}^\star)(x_{h+1}^{k}) \sum_{t = n_{h}^{k}}^\infty \alpha_t^{n_h^{k}}.
\end{align*}
The final recursive relationship results from the property that by construction $\sum_{t=i}^{\infty} \alpha_t^i = 1 + \frac{1}{H}$ for all $i$ from Lemma~\ref{lemma:lr}, and the inequality $V_{h+1}^{\pi^k}(x_{h+1}^k) \leq V_{h+1}^\star(x_{h+1}^k) \leq \Vhat{h+1}{k}(x_{h+1}^k)$.

\section{Discussion and Extensions}
\label{section:discussion}

{The partitioning method used in Algorithm~\ref{alg} was chosen due to its implementability.  However, our analysis can be extended to provide a framework for adaptive $Q$-learning algorithms that flexibly learns partitions of the space in a data-driven manner.  This allows practitioners to use their favourite partitioning algorithm (e.g. decision trees, kernel methods, etc) to adaptively partition the space.} {This begs the question: {\em Amongst different partitioning algorithms, which ones still guarantee the same regret scaling as Theorem~\ref{thm:regret}?}}

{In particular, we consider black box partitioning schemes that incrementally build a nested partition, determining {\em when} and {\em how} to split regions as a function of the observed data. This black box partitioning algorithm is then plugged into the ``repartioning'' step of Algorithm 1.}
{Let $\P_h^k$ denote the partition kept by the algorithm for step $h$ in episode $k$.  The algorithm stores estimates $\Qhat{h}{k}(P)$ for the optimal $Q_h^\star$ value for points in a given region $P \in \P_h^k$ and $n_h^k(P)$ for the number of times $P$ or its ancestors has been selected in step $h$ up to the current episode $k$.  In every episode $k$, step $h$ of the algorithm, the procedure proceeds identically to that in Algorithm~\ref{alg} by selecting a region $P_h^k$ which is relevant for the current state with maximal $\Qhat{h}{k}(P)$ value.  The update rules are the same as those defined in \eqref{eqn:update} where instead of balls we consider regions in the partition. When a region is subpartitioned, all of its children must inherent the values of $\Qhat{h}{k}(P)$ and $n_h^k(P)$ from its parent. The black box partitioning algorithm only decides \textit{when} and \textit{how} to subpartition the selected region $P_h^k$.}

In Theorem~\ref{thm:regret_general}, we extend the regret guarantees to modifications of Algorithm~\ref{alg} that maintain the conditions below.
\begin{theorem} \label{thm:regret_general}
For any modification of Algorithm~\ref{alg} with a black box partitioning scheme that satisfies
$\forall\,h \in [H]$, $\P_h^k$:
    \begin{enumerate}
        \item $\{ \P_h^k \}_{k \geq 1}$ is a sequence of nested partitions.
        \item There exists constants $c_1$ and $c_2$ such that for every $h, k$ and $P \in \P_h^k$ $$\frac{c_1^2}{\diam{P}^2} \leq n_h^k(P) \leq \frac{c_2^2}{\diam{P}^2}.$$
        \item $|\P_h^K| \leq K^{d_c/(d_c+2)}$.
    \end{enumerate}
    the achieved regret is bounded by  $\tilde{O}(H^{5/2}K^{(d_c+1)/(d_c+2)})$.
\end{theorem}

{Our algorithm clearly satisfies the first condition.  The second condition is verified in Lemma~\ref{lemma:partition}, and the third condition follows from a tighter analysis in Lemma~\ref{lemma:bound_on_beta_summation}.  The only dependence on the partitioning method used when proving Theorem~\ref{thm:regret} is through these sets of assumptions, and so the proof follows from a straightforward generalization of the results.}

{Similar generalizations were provided in the contextual bandit setting by \cite{wang2019stochbandits}, although they only show sublinear regret.}
{The first condition requires that the partition evolves in a hierarchical manner, as could be represented by a tree, where each region of the partition has an associated parent. Due to the fact that a child inherits its $\Qhat{h}{k}$ estimates from its parent, the recursive update expansion in Lemma~\ref{lemma:upper_lower_bounds_main} still holds. The second condition establishes that the number of times a given region is selected grows in terms of the square of its diameter.  This can be thought of as a bias variance trade-off, as when the number of samples in a given region is large, the variance term dominates the bias term and it is advantageous to split the partition to obtain more refined estimates.  This assumption is necessary for showing Lemma~\ref{lemma:bound_ball_main}.  It also enforces that the coarseness of the partition depends on the density of the observations.  The third condition controls the size of the partition and is used to compute the sum $\sum_{k=1}^K \beta_{n_h^k}$ from the proof sketch in Section~\ref{section:sketch}.}

This theorem provides a practical framework for developing adaptive $Q$-learning algorithms.  The first condition is easily satisfied by picking the partition in an online fashion and splitting the selected region into sub-regions.  {The second condition can be checked at every iteration and determines \textit{when} to sub-partition a region.  The third condition limits the practitioner from creating unusually shaped regions.  Using these conditions, a practitioner can use any decision tree or partitioning algorithm that satisfies these properties to create an adaptive $Q$-learning algorithm.  An interesting future direction is to understand whether different partitioning methods leads to provable instance-specific gains.}
    
\section{Experiments}
\label{section:experiments}
\subsection{Experimental Set up}
\label{section:experiments_set_up}

{We compare the performance of $Q$-learning with Adaptive Discretization (Algorithm~\ref{alg}) to Net Based $Q$-Learning~\cite{song_efficient_2019} on two canonical problems to illustrate the advantage of adaptive discretization compared to uniform mesh.  On both problems the state and action space are taken to be $\S = \A = [0,1]$ and the metric the product metric, i.e. $\D((x,a), (x',a')) = \max\{|x - x'|, |a-a'|\}$.  This choice of metric results in a rectangular partition of $\S \times \A = [0,1]^2$ allowing for simpler implementation and easy comparison of the partitions used by the algorithms.  We focus on problems which have a two-dimensional state-action space to provide a proof of concept of the adaptive discretizations constructed by the algorithm.  These examples have more structure than the worst-case bound, but provide intuition on the partition ``zooming'' into important parts of the space.  More simulation results are available with the code on github at \url{https://github.com/seanrsinclair/AdaptiveQLearning}}.

\medskip
\textbf{Oil Discovery.} This problem, adapted from~\cite{mason2012collaborative}, is a continuous variant of the popular `Grid World' game. It comprises of an agent (or agents) surveying a 1D map in search of hidden ``oil deposits''. The world is endowed with an unknown survey function $f(x)$ which encodes the oil deposit at each location $x$. For agents to move to a new location they must pay a cost proportional to the distance moved; moreover, surveying the land produces noisy estimates of the true value of that location. 

We consider an MDP where $\S = \A = [0,1]$ represent the current and future locations of the agent.  The time-invariant transition kernel (i.e., homogeneous for all $h\in[H]$) is defined via $\Pr_h(x' \mid x, a) = \Ind{x'=a}$, signifying that the agent moves to their chosen location. Rewards are given by $r_h(x,a) = \max \{0, f(a) + \epsilon - |x - a|\}$, where $\epsilon$ is independent sub-Gaussian noise and $f(a) \in [0,1]$ is the survey value of the location $a$ (the max ensures rewards are in $[0,1]$). 

We choose the survey function $f(x)$ to be either $f(x) = e^{-\lambda |x - c|}$ or $f(x) = 1 - \lambda (x-c)^2$ where $c \in [0,1]$ is the location of the oil well and $\lambda$ is a smoothing parameter, which can be tuned to adjust the Lipschitz constant.

\medskip
\textbf{Ambulance Relocation.} This is a widely-studied stochastic variant of the above problem~\cite{brotcorne2003ambulance}. Consider an ambulance navigating an environment and trying to position itself in a region with high predicted service requests. The agent interacts with the environment by first choosing a location to station the ambulance, paying a cost to travel to that location.  Next, an ambulance request is realized, drawn from a fixed distribution, after which the ambulance must travel to meet the demand at the random location.

Formally, we consider an MDP with $\S = \A = [0,1]$ encoding the current and future locations of the ambulance. {The transition kernels are defined via $\Pr_h(x' | x, a) \sim \mathcal{F}_h$, where $\mathcal{F}_h$ denotes the request distribution for time step $h$.}  The reward is $r_h(x'|x,a) = 1 - [ c |x - a| + (1 - c)|x' - a| ]$; here, $c\in[0,1]$ models the trade-offs between the cost of relocation (often is less expensive) and cost of traveling to meet the demand.

For the arrival distribution $\mathcal{F}_h$ we consider Beta$(5,2)$ and Uniform$(0,1)$ to illustrate dispersed and concentrated request distributions.  {We also analyzed the effect of changing the arrival distribution over time (e.g. Figure~\ref{fig:partition}).}  We compare the RL methods to two heuristics: ``No Movement'', where the ambulance pays the cost of traveling to the request, but does not relocate after that, and ``Median'', where the ambulance always relocates to the median of the observed requests.  Each heuristic is near-optimal respectively at the two extreme values of $c$. 

\subsection{Adaptive Tree Implementation}

{We used a tree data structure to implement the partition $\P_h^k$ of Algorithm~\ref{alg}.  We maintain a tree for every step $h \in [H]$ to signify our partition of $\S \times \A = [0,1]^2$.  Each node in the tree corresponds to a rectangle of the partition, and contains algorithmic information such as the estimate of the $\Qhat{h}{k}$ value at the node.  Each node has an associated center $(x,a)$ and radius $r$.  We store a list of (possibly four) children for covering the region which arises when \textit{splitting} a ball.}

To implement the \textit{selection rule} we used a recursive algorithm which traverses through all the nodes in the tree, checks each node if the given state $x_h^k$ is contained in the node, and if so recursively checks the children to obtain the maximum $\Qhat{h}{k}$ value.  {This speeds up computation by following the tree structure instead of linear traversal.  For additional savings we implemented this using a max-heap.}

For Net Based $Q$-Learning, based on the recommendation in~\cite{song_efficient_2019}, we used a fixed $\epsilon$-Net of the state-action space, with $\epsilon = (KH)^{-1/4}$ (since $d = 2$).  An example of the discretization can be seen in Figure~\ref{fig:partition} where each point in the discretization is the center of a rectangle.

\begin{figure*}[t!]
\centering
\includegraphics[width=\columnwidth]{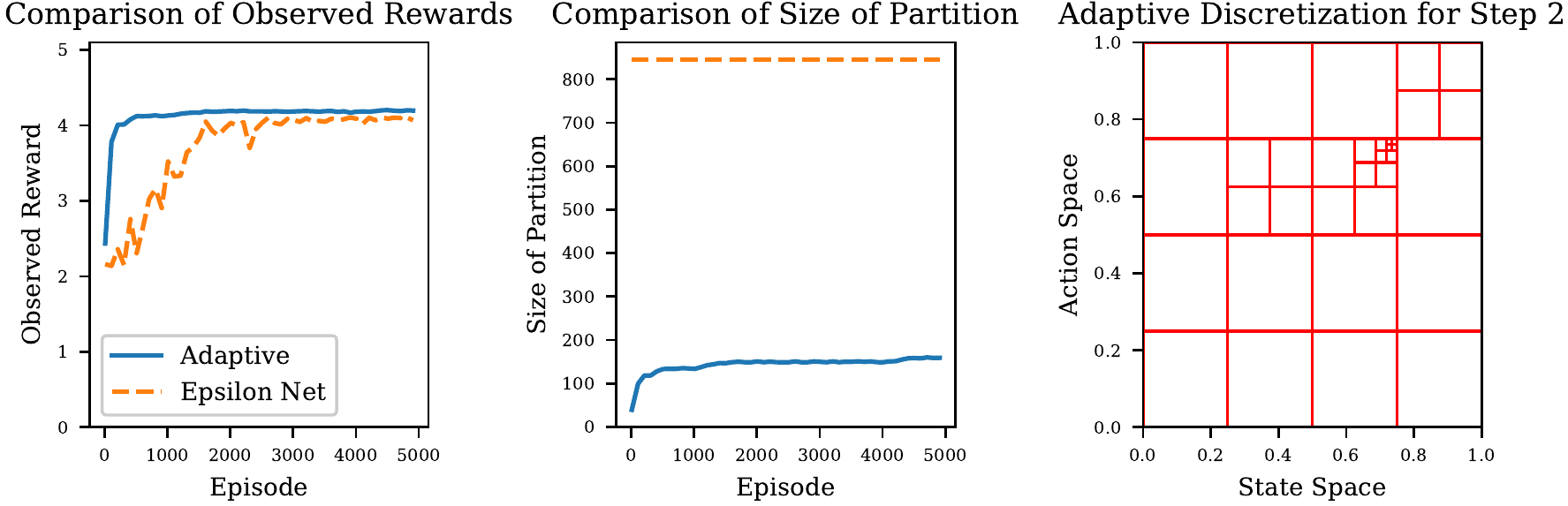}
\caption{Comparison of the observed rewards, size of partition, and discretization for the uniform mesh~\cite{song_efficient_2019} and our adaptive discretization algorithms, on the oil discovery problem with survey function $f(x) = e^{- |x - 0.75|}$.  The transition kernel is $\Pr_h(x' \mid x,a) = \Ind{x' = a}$ and reward function is $r(x,a) = (1 - e^{-|a - 0.75|} - |x - a|)_+$ (see Section~\ref{section:experiments_set_up}).  The adaptive algorithm quickly learns the location of the optimal point $0.75$ and creates a fine partition of the space around the optimal.
}
\label{fig:oil_laplace}
\end{figure*}

\begin{figure*}[t!]
\centering
\includegraphics[width=\columnwidth]{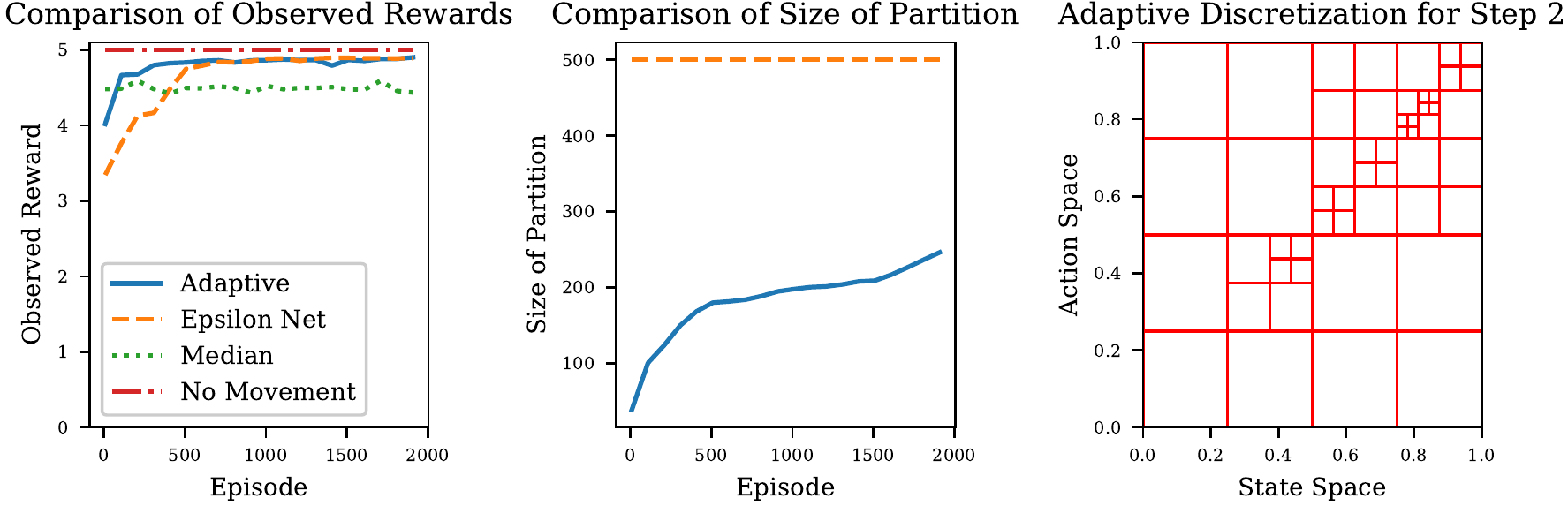}
\caption{
Comparison of the algorithms on the ambulance problem with Beta$(5,2)$ arrival distribution and reward function $r(x, a) = 1 - |x - a|$ (see Section~\ref{section:experiments_set_up}).  Clearly, the no movement heuristic is the optimal policy but the adaptive algorithm learns a fine partition across the diagonal of the space where the optimal policy lies.
}
\label{fig:ambulance_beta}
\end{figure*}

\subsection{Experimental Results}

\medskip
\textbf{Oil Discovery.}
First we consider the oil problem with survey function $f(a) = e^{-\lambda |a-c|}$ where $\lambda = 1$.  This results in a sharp reward function.  Heuristically, the optimal policy can be seen to take the first step to travel to the maximum reward at $c$ and then for each subsequent step stay at the current location.  For the experiments we took the number of steps $H$ to be five and tuned the scaling of the confidence bounds for the two algorithms separately.  In Figure~\ref{fig:oil_laplace} we see that the adaptive $Q$-learning algorithm is able to ascertain the optimal policy in fewer iterations than the epsilon net algorithm.  Moreover, due to adaptively partitioning the space instead of fixing a discretization before running the algorithm, the size of the partition in the adaptive algorithm is drastically reduced in comparison to the epsilon net.  Running the experiment on the survey function $f(a) = 1 - \lambda(x-c)^2$ gave similar results and the graph can be seen in Appendix~\ref{section:full_experiments}.  We note that for both experiments as we increase $\lambda$, causing the survey distribution to become peaky and have a higher Lipschitz constant, the Net based $Q$-learning algorithm suffered from a larger discretization error.  This is due to that fact that the tuning of the $\epsilon$ parameter ignores the Lipschitz constant.  However, the adaptive $Q$-learning algorithm is able to take advantage and narrow in on the point with large reward.

\medskip
\textbf{Ambulance Routing.}
We consider the ambulance routing problem with arrival distribution $\mathcal{F} = \text{Beta}(5,2)$.  The time horizon $H$ is again taken to be five.  We implemented two additional heuristics for the problem to serve as benchmarks.  The first is the ``No Movement'' heuristic.  This algorithm takes the action to never move, and pays the entire cost of traveling to service the arrival.  The second, ``Median'' heuristic takes the action to travel to the estimated median of the distribution based on all past arrivals.  For the case when $c = 1$ no movement is the optimal algorithm.  For $c = 0$ the optimal will be to travel to the median.

After running the algorithms with $c = 1$ the adaptive $Q$-learning algorithm is better able to learn the stationary policy of not moving from the current location, as there is only a cost of moving to the action.  The rewards observed by each algorithm is seen in Figure~\ref{fig:ambulance_beta}.  In this case, the discretization (for step $h = 2$) shows that the adaptive $Q$-learning algorithm maintains a finer partition across the diagonal where the optimal policy lives.  Running the algorithm for $c < 1$ shows that the adaptive $Q$-learning algorithm has a finer discretization around $(x,a) = (0.7, 0.7)$, where $0.7$ is the approximate median of a Beta$(5,2)$ distribution. The algorithm keeps a fine discretization both around where the algorithm frequently visits, but also places of high reward.  For $c < 1$ and the arrival distribution $\mathcal{F} = \text{Uniform}[0,1]$ the results were similar and the graphs are in Appendix~\ref{section:full_experiments}.

{The last experiment was to analyze the algorithms when the arrival distribution changes over time (e.g. over steps $h$).  In Figure~\ref{fig:partition} we took $c = 0$ and shifting arrival distributions $\mathcal{F}_h$ where $\mathcal{F}_1 = \text{Uniform}(0, 1/4), \mathcal{F}_2 = \text{Uniform}(1/4, 1/2), \mathcal{F}_3 = \text{Uniform}(1/2, 3/4), \mathcal{F}_4 = \text{Uniform}(3/4, 1), \text{and }\mathcal{F}_5 = \text{Uniform}(1/2-0.05, 1/2+0.05)$.  In the figure, the color corresponds to the $Q_h^\star$ value of that specific state-action pair, where green corresponds to a larger value for the expected future rewards.  The adaptive algorithm was able to converge faster to the optimal policy than the uniform mesh algorithm.  Moreover, the discretization observed from the Adaptive $Q$-Learning algorithm follows the contours of the $Q$-function over the space.  This shows the intuition behind the algorithm of storing a fine partition across near-optimal parts of the space, and a coarse partition across sub-optimal parts.}
\section{Conclusion}
\label{section:conclusion}
We presented an algorithm for model-free episodic reinforcement learning on continuous state action spaces that uses data-driven discretization to adapt to the shape of the optimal policy.  Under the assumption that the optimal $Q^\star$ function is Lipschitz, the algorithm achieves regret bounds scaling in terms of the covering dimension of the metric space.

Future directions include relaxing the requirements on the metric, such as considering weaker versions of the Lipschitz condition. In settings where the metric may not be known a priori, it would be meaningful to be able to estimate the distance metric from data over the course of executinng of the algorithm. Lastly, we hope to characterize problems where adaptive discretization outperforms uniform mesh.
    \section*{Acknowledgments}
    	We gratefully acknowledge funding from the NSF under grants ECCS-1847393 and DMS-1839346, and the ARL under grant W911NF-17-1-0094.
    \bibliographystyle{plain}
    {\bibliography{references}}

\begin{thebibliography}{10}

\bibitem{NIPS2008_3401}
Peter Auer, Thomas Jaksch, and Ronald Ortner.
\newblock Near-optimal regret bounds for reinforcement learning.
\newblock In D.~Koller, D.~Schuurmans, Y.~Bengio, and L.~Bottou, editors, {\em
  Advances in Neural Information Processing Systems 21}, pages 89--96. Curran
  Associates, Inc., 2009.

\bibitem{azar_2017}
Mohammad~Gheshlaghi Azar, Ian Osband, and R{\'e}mi Munos.
\newblock Minimax regret bounds for reinforcement learning.
\newblock In {\em Proceedings of the 34th International Conference on Machine
  Learning - Volume 70}, ICML'17, pages 263--272. JMLR.org, 2017.

\bibitem{brotcorne2003ambulance}
Luce Brotcorne, Gilbert Laporte, and Frederic Semet.
\newblock Ambulance location and relocation models.
\newblock {\em European journal of operational research}, 147(3):451--463,
  2003.

\bibitem{bubeck2012regret}
S{\'e}bastien Bubeck, Nicolo Cesa-Bianchi, et~al.
\newblock Regret analysis of stochastic and nonstochastic multi-armed bandit
  problems.
\newblock {\em Foundations and Trends{\textregistered} in Machine Learning},
  5(1):1--122, 2012.

\bibitem{bubeck2009online}
S{\'e}bastien Bubeck, Gilles Stoltz, Csaba Szepesv{\'a}ri, and R{\'e}mi Munos.
\newblock Online optimization in x-armed bandits.
\newblock In {\em Advances in Neural Information Processing Systems}, pages
  201--208, 2009.

\bibitem{Comden:2019}
Joshua Comden, Sijie Yao, Niangjun Chen, Haipeng Xing, and Zhenhua Liu.
\newblock Online optimization in cloud resource provisioning: Predictions,
  regrets, and algorithms.
\newblock {\em Proc. ACM Meas. Anal. Comput. Syst.}, 3(1):16:1--16:30, March
  2019.

\bibitem{dong2019q}
Kefan Dong, Yuanhao Wang, Xiaoyu Chen, and Liwei Wang.
\newblock Q-learning with ucb exploration is sample efficient for
  infinite-horizon mdp.
\newblock {\em arXiv preprint arXiv:1901.09311}, 2019.

\bibitem{du2019provably}
Simon~S Du, Yuping Luo, Ruosong Wang, and Hanrui Zhang.
\newblock Provably efficient $ q $-learning with function approximation via
  distribution shift error checking oracle.
\newblock {\em arXiv preprint arXiv:1906.06321}, 2019.

\bibitem{zhu2019stochastic}
Xu~Zhu~David Dunson.
\newblock Stochastic lipschitz q-learning, 2019.

\bibitem{gibbs2002choosing}
Alison~L Gibbs and Francis~Edward Su.
\newblock On choosing and bounding probability metrics.
\newblock {\em International statistical review}, 70(3):419--435, 2002.

\bibitem{jin_is_2018}
{Jin C}, {Jordan M.I}, {Allen-Zhu Z}, {Bubeck S}, and NeurIPS~2018 32nd
  Conference~on Neural Information Processing~Systems.
\newblock Is {Q}-learning provably efficient?
\newblock {\em Adv. neural inf. proces. syst. Advances in Neural Information
  Processing Systems}, 2018-December:4863--4873, 2018.
\newblock OCLC: 8096900528.

\bibitem{kakade_2003}
Sham Kakade, Michael Kearns, and John Langford.
\newblock Exploration in metric state spaces.
\newblock In {\em Proceedings of the Twentieth International Conference on
  International Conference on Machine Learning}, ICML'03, pages 306--312. AAAI
  Press, 2003.

\bibitem{Kleinberg:2019:BEM:3338848.3299873}
Robert Kleinberg, Aleksandrs Slivkins, and Eli Upfal.
\newblock Bandits and experts in metric spaces.
\newblock {\em J. ACM}, 66(4):30:1--30:77, May 2019.

\bibitem{lakshmanan15}
K.~Lakshmanan, Ronald Ortner, and Daniil Ryabko.
\newblock Improved regret bounds for undiscounted continuous reinforcement
  learning.
\newblock In Francis Bach and David Blei, editors, {\em Proceedings of the 32nd
  International Conference on Machine Learning}, volume~37 of {\em Proceedings
  of Machine Learning Research}, pages 524--532, Lille, France, 07--09 Jul
  2015. PMLR.

\bibitem{lattimore_2018}
Tor Lattimore and Csaba Szepesv{\'a}ri.
\newblock Bandit algorithms.
\newblock {\em preprint}, 2018.

\bibitem{Mao:2016}
Hongzi Mao, Mohammad Alizadeh, Ishai Menache, and Srikanth Kandula.
\newblock Resource management with deep reinforcement learning.
\newblock In {\em Proceedings of the 15th ACM Workshop on Hot Topics in
  Networks}, HotNets '16, pages 50--56, New York, NY, USA, 2016. ACM.

\bibitem{mason2012collaborative}
Winter Mason and Duncan~J Watts.
\newblock Collaborative learning in networks.
\newblock {\em Proceedings of the National Academy of Sciences},
  109(3):764--769, 2012.

\bibitem{Ortner2013}
Ronald Ortner.
\newblock Adaptive aggregation for reinforcement learning in average reward
  markov decision processes.
\newblock {\em Annals of Operations Research}, 208(1):321--336, Sep 2013.

\bibitem{Ortner2012}
Ronald Ortner and Daniil Ryabko.
\newblock Online regret bounds for undiscounted continuous reinforcement
  learning.
\newblock In F.~Pereira, C.~J.~C. Burges, L.~Bottou, and K.~Q. Weinberger,
  editors, {\em Advances in Neural Information Processing Systems 25}, pages
  1763--1771. Curran Associates, Inc., 2012.

\bibitem{osband2014}
Ian Osband and Benjamin Van~Roy.
\newblock Model-based reinforcement learning and the eluder dimension.
\newblock In Z.~Ghahramani, M.~Welling, C.~Cortes, N.~D. Lawrence, and K.~Q.
  Weinberger, editors, {\em Advances in Neural Information Processing Systems
  27}, pages 1466--1474. Curran Associates, Inc., 2014.

\bibitem{puterman_1994}
Martin~L. Puterman.
\newblock {\em Markov Decision Processes: Discrete Stochastic Dynamic
  Programming}.
\newblock John Wiley \& Sons, Inc., New York, NY, USA, 1st edition, 1994.

\bibitem{shah2018}
Devavrat Shah and Qiaomin Xie.
\newblock Q-learning with nearest neighbors.
\newblock In S.~Bengio, H.~Wallach, H.~Larochelle, K.~Grauman, N.~Cesa-Bianchi,
  and R.~Garnett, editors, {\em Advances in Neural Information Processing
  Systems 31}, pages 3111--3121. Curran Associates, Inc., 2018.

\bibitem{simchowitz2019nonasymptotic}
Max Simchowitz and Kevin Jamieson.
\newblock Non-asymptotic gap-dependent regret bounds for tabular mdps, 2019.

\bibitem{slivkins_contextual_2015}
Aleksandrs Slivkins.
\newblock Contextual {Bandits} with {Similarity} {Information}.
\newblock {\em Journal of machine learning research : JMLR.}, 15(2):2533--2568,
  2015.
\newblock OCLC: 5973068319.

\bibitem{aleks2019introduction}
Aleksandrs Slivkins.
\newblock Introduction to multi-armed bandits, 2019.

\bibitem{song_efficient_2019}
Zhao Song and Wen Sun.
\newblock Efficient {Model}-free {Reinforcement} {Learning} in {Metric}
  {Spaces}.
\newblock {\em arXiv:1905.00475 [cs, stat]}, May 2019.
\newblock arXiv: 1905.00475.

\bibitem{sutton2018reinforcement}
Richard~S Sutton and Andrew~G Barto.
\newblock {\em Reinforcement learning: An introduction}.
\newblock MIT press, 2018.

\bibitem{wang2019stochbandits}
Tianyu {Wang}, Weicheng {Ye}, Dawei {Geng}, and Cynthia {Rudin}.
\newblock {Towards Practical Lipschitz Stochastic Bandits}.
\newblock {\em arXiv e-prints}, page arXiv:1901.09277, Jan 2019.

\bibitem{nir2019nonparametric}
Nirandika Wanigasekara and Christina~Lee Yu.
\newblock Nonparametric contextual bandits in an unknown metric space, 2019.

\bibitem{watkins1989learning}
Christopher John Cornish~Hellaby Watkins.
\newblock Learning from delayed rewards.
\newblock 1989.

\bibitem{yang2019sample}
Lin Yang and Mengdi Wang.
\newblock Sample-optimal parametric q-learning using linearly additive
  features.
\newblock In {\em International Conference on Machine Learning}, pages
  6995--7004, 2019.

\bibitem{yang2019learning}
Lin~F Yang, Chengzhuo Ni, and Mengdi Wang.
\newblock Learning to control in metric space with optimal regret.
\newblock {\em arXiv preprint arXiv:1905.01576}, 2019.

\bibitem{yang_reinforcement_2019}
Lin~F. Yang and Mengdi Wang.
\newblock Reinforcement {Leaning} in {Feature} {Space}: {Matrix} {Bandit},
  {Kernels}, and {Regret} {Bound}.
\newblock {\em arXiv:1905.10389 [cs, stat]}, May 2019.
\newblock arXiv: 1905.10389.

\end{thebibliography}
    
    \newpage
    \onecolumn
    \appendix
\section{Notation Table}
\renewcommand{\arraystretch}{1.2}
\begin{table*}[h!]
\begin{tabular}{l|l}
\textbf{Symbol} & \textbf{Definition} \\ \hline
$\S,\A,H$  & MDP specifications (state space, action space, steps per episode)\\
$r_h(x,a)\,,\,\Pr_h(\cdot \mid x,a)$ & Reward/transition kernel for taking action $a$ in state $x$ at step $h$\\
$K$ & Number of episodes\\
$\pi^\star_h$ & Optimal policy in step $h$\\
$(x_h^k, a_h^k)$ & State and action executed by the algorithm at step h in episode k \\
$\Delta(\S)$ & Set of probability measures on $\S$ \\
$V_h^\pi(\cdot),Q_h^\pi(\cdot,\cdot)$ & Value/Q-function at step $h$ under policy $\pi$ \\
$V_h^\star(\cdot),Q_h^\star(\cdot,\cdot)$ & Value/Q-function for step $h$ under the optimal policy \\
$L$ & Lipschitz constant for $Q^\star$ \\
$\D$ & Metric on $\S \times \A$ \\
$d_{max}$ & Bound on $\S \times \A$ using the metric $\D$ \\
$N_r$ & $r$ covering number of $\S \times \A$ \\
$d_c$ & The covering dimension of $\S \times \A$ with parameter $c$\\
$\P_h^k$ & Partition of $\S \times \A$ for step $h$ in episode $k$\\
$B_h^k$ & Ball selected by algorithm in step $h$ episode $k$ \\
$\Qhat{h}{k}(B)$ & Estimate of the $Q$ value for points in $B$ in step $h$ episode $k$\\
$\Vhat{h}{k}(x)$ & Estimate of the value function for $x$ in step $h$ episode $k$ \\
$n_h^k(B)$ & Number of times $B$ or its ancestors has been chosen 
\\
& \quad before episode $k$ in step $h$ \\
$\alpha_t$ & The adaptive learning rate $\frac{H+1}{H+t}$ (cf. Alg~\ref{alg}) \\
$\alpha_t^i$ & $\alpha_i \prod_{j=i+1}^t (1 - \alpha_j)$ \\
$\diam{B}$ & The diameter of a set $B$ \\
$r(B)$ & The radius of a ball $B$ \\
$b(t)$ & The upper confidence term, $2\sqrt{\frac{H^3 \log(4HK/\pfail)}{t}} + \frac{4 L D_{max}}{\sqrt{t}}$ \\
$R(K)$ & The regret up to episode $K$ \\
$\dom{B}$ & The domain of a ball $B$, excludes points in a ball with a smaller radius \\
$\text{RELEVANT}_h^k(x)$ & The set of all balls which are relevant for $x$ in step $h$ episode $k$\\
$\Exp{V_{h+1}(\hat{x}) \mid x,a}$ & $\mathbb{E}_{\hat{x} \sim \Pr_h(\cdot \mid x,a)} [V_{h+1}(\hat{x})]$\\
\hline
\end{tabular}
\caption{List of common notation}
\label{table:notation}
\end{table*}

\newpage

\section{Complete Pseudocode for Adaptive $Q$-Learning Algorithm}
\label{sec:complete_algorithm}

\begin{algorithm}[h!]
	\caption{Adaptive $Q$-Learning}
	\label{full_alg}
	\begin{algorithmic}[1]
		\Procedure{Adaptive $Q$-learning}{$\S, \A, \D, H, K, \pfail$}
			\For{$h \gets 1, \ldots, H$}
				\State $\P_h^1$ contains a single ball $B$ with radius $d_{max}$
				\State $\Qhat{h}{1}(B) \gets H$
			\EndFor
			\For{each episode $k \gets 1, \ldots K$}
				\State Receive initial state $x_1^k$
				\For{each step $h \gets 1, \ldots, H$}
	                \State Selection Rule:
					\State Set $\texttt{RELEVANT}_h^k (x_h^k) = \{B \in \P_h^k \mid \exists a \in \A \text{ with } (x_h^k, a) \in \dom{B}\}$ 
					\State Select the ball $B_{sel}$ by the selection rule $B_{sel} = \argmax_{B \in \texttt{RELEVANT}_h^k(x_h^k)} \Qhat{h}{k}(B)$
					\State Select action $a_h^k = a$ for some $(x_h^k,a) \in \dom{B_{sel}}$
					\State Play action $a_h^k$, receive reward $r_h^k$, and new state $x_{h+1}^k$
					\State Update Parameters:				\State $n_h^{k+1}(B_{sel}) \gets n_h^{k}(B_{sel}) + 1$
					\State $t \gets n_h^{k+1}(B_{sel})$
					\State $b(t) \gets 4 \sqrt{\frac{H^3 \log(4HK / \pfail)}{t}} + \frac{2Ld_{max}}{\sqrt{t}}$
					\State $\Qhat{h}{k+1}(B_{sel}) \gets (1 - \alpha_t)\Qhat{h}{k}(B_{sel}) + \alpha_t(r_h^k + \Vhat{h+1}{k}(x_{h+1}^k) + b(t))$ where
					\State $\Vhat{h}{k}(x_h^k) = \min(H, \max_{B \in \texttt{RELEVANT}_{h}^k (x_{h}^k)} \Qhat{h}{k}(B))$
					\If{$t \geq \left( \frac{d_{max}}{r(B_{sel})}\right)^2$}
					    \State \textproc{Split Ball}$(B_{sel}, h, k)$
					\EndIf
					\State $\P_h^{k+1} \gets \P_h^k$
				\EndFor
			\EndFor
		\EndProcedure
		\Procedure{Split Ball}{$B, h, k$}
		    \State Set $B_1, \ldots B_n$ to be a $\frac{1}{2}r(B)-$Net of $\dom{B}$
		    \For{each ball $B_i$}
		        \State $\Qhat{h}{k+1}(B_i) \gets \Qhat{h}{k+1}(B)$ \\
		        \State $n_h^{k+1}(B_i) \gets n_h^{k+1}(B)$ \\
		        \State Add $B_i$ to $\P_h^{k+1}$
		    \EndFor
		 \EndProcedure
	\end{algorithmic}
\end{algorithm}
\section{Proofs for Lipschitz $Q^\star$ Function}
\label{sec:assumptions}
\begin{proposition}
	Suppose that the transition kernel and reward function are Lipschitz continuous with respect to $\D$, i.e.
	\begin{align*}
	\norm{\Pr_h(\cdot \mid x,a) - \Pr_h(\cdot \mid x', a')}_{TV} & \leq L_1 \D((x,a), (x',a')) \text{ and } \\
	|r_h(x,a) - r_h(x', a')| & \leq L_2 \D((x,a), (x',a'))
	\end{align*}
	for all $(x,a), (x', a') \in \S \times \A$ and $h$.  Then $Q_h^\star$ is also $(2L_1H + L_2)$ Lipschitz continuous.
\end{proposition}
\begin{proof}
	We use the definition of $Q_h^\star$.  Indeed, for any $h \in [H]$,
	\begin{align*}
	|Q_h^\star(x,a) - Q_h^\star(x', a')| & = |r_h(x,a) - r_h(x', a') + \Exp{V_{h+1}^\star(\hat{x}) \mid x,a} - \Exp{V_{h+1}^\star(\hat{x}) \mid x', a'}|.
	\end{align*}
	However, noting that both $\Pr_h(\cdot \mid x,a)$ and $\Pr_h(\cdot \mid x', a')$ are absolutely continuous with respect to the base measure $\lambda = \frac{1}{2}\left( \Pr_h(\cdot \mid x, a) + \Pr_h(\cdot \mid x', a' )\right)$ we are able to deconstruct the difference in expectations with respect to the Radon-Nikodym derivatives as the measures are $\sigma$-finite.  Hence,
	\begin{align*}
	& \leq |r_h(x,a) - r_h(x', a')| + \left|\int V_{h+1}^\star(\hat{x}) \, d\Pr_h(\hat{x} \mid x,a) - \int V_{h+1}^\star(\hat{x}) \, d\Pr_h(\hat{x} \mid x', a')\right| \\
	& \leq L_2 \D((x,a), (x',a')) + \left|\int V_{h+1}^\star(\hat{x}) \frac{d\Pr_h(\hat{x} \mid x, a)}{d\lambda(\hat{x})}\, d\lambda(\hat{x}) - \int V_{h+1}^\star(\hat{x}) \frac{d\Pr_h(\hat{x} \mid x', a')}{d\lambda(\hat{x})} \, d\lambda(\hat{x})\right| \\
	& \leq L_2 \D((x,a), (x', a')) + \int |V_{h+1}^\star(\hat{x})| \left|\frac{d\Pr_h(\hat{x} \mid x, a)}{d\lambda(\hat{x})} - \frac{d\Pr_h(\hat{x} \mid x', a')}{d\lambda(\hat{x})}\right| \, d\lambda(\hat{x}) \\
	& \leq L_2 \D((x,a), (x', a')) + \norm{V_{h+1}^\star}_{\infty} \int \left|\frac{d\Pr_h(\hat{x} \mid x, a)}{d\lambda(\hat{x})} - \frac{d\Pr_h(\hat{x} \mid x', a')}{d\lambda(\hat{x})}\right| \, d\lambda(\hat{x}) \\
	& \leq L_2 \D((x,a), (x', a')) + H\left\Vert\frac{d\Pr_h(\cdot \mid x, a)}{d\lambda} - \frac{d\Pr_h(\cdot \mid x', a')}{d\lambda}\right\Vert_1 \\
	& = L_2 \D((x,a), (x', a')) + 2H\norm{\Pr_h(\cdot \mid x, a) - \Pr_h(\cdot \mid x', a')}_{TV} \\
	& \leq L_2 \D((x, a), (x', a')) + 2HL_1\D((x, a),(x', a'))
	\end{align*}
	where we have used the fact that the total variation distance is twice the $\mathcal{L}_1$ distance of the Radon-Nikodym derivatives.
\end{proof}

Now we further assume that $\S$ is a separable and compact metric space with metric $d_\S$.  We also assume that $\D((x,a),(x',a)) \leq Cd_\S(x,x')$ where we assume $C = 1$ for simplicity.

As a preliminary we begin with the following lemma.
\begin{lemma}
\label{lemma:Lipschitz}
Suppose that $f : \S \times \A \rightarrow \mathbb{R}$ is $L$ Lipschitz and uniformly bounded.  Then $g(x) = \sup_{a \in \A} f(x,a)$ is also $L$ Lipschitz.
\end{lemma}
\begin{proof}
Fix any $x_1$ and $x_2 \in \S$.  First notice that $|f(x_1, a) - f(x_2, a)| \leq L \D((x_1, a), (x_2,a)) \leq Ld_\S(x_1, x_2)$ by choice of product metric.

Thus, for any $a \in \A$ we have that $$f(x_1, a) \leq f(x_2, a) + Ld_\S(x_1, x_2) \leq g(x_2) + Ld_\S(x_1, x_2).$$  However, as this is true for any $a \in \A$ we see that $g(x_1) \leq g(x_2) + Ld_\S(x_1, x_2)$.  Swapping the role of $x_1$ and $x_2$ in the inequality shows $|g(x_1) - g(x_2)| \leq Ld_\S(x_1, x_2)$ as needed.
\end{proof}
We will use this and properties of the Wasserstein metric from \cite{gibbs2002choosing} to give conditions on the $Q^\star$ function to be Lipschitz.  First notice the definition of the Wasserstein metric on a separable metric space via:
$$d_W(\mu, \nu) = \sup \left\{ \left| \int f \, d\mu - \int f \, d\nu \right| : \norm{f}_L \leq 1\right\}$$ where $\norm{f}_L$ is the smallest Lipschitz constant of the function $f$ and $\mu$ and $\nu$ are measures on the space $\S$.
\begin{proposition}
Suppose that the transition kernel is Lipschitz with respect to the Wasserstein metric and the reward function is Lipschitz continuous, i.e. 
\begin{align*}
    |r_h(x,a) - r_h(x', a')| & \leq L_1 D((x,a),(x',a')) \\
    d_W(\Pr_h(\cdot \mid x,a), \Pr_h(\cdot \mid x', a')) & \leq L_2 D((x,a),(x',a'))
\end{align*}
for all $(x,a), (x', a') \in \S \times \A$ and $h$ where $d_W$ is the Wasserstein metric.  Then $Q_h^\star$ and $V_h^\star$ are both $(\sum_{i=0}^{H-h} L_1 L_2^{i})$ Lipschitz continuous.
\end{proposition}
\begin{proof}
We show this by induction on $h$.  For the base case when $h = H+1$ then $V_{H+1}^\star = Q_{H+1}^\star = 0$ and so the result trivially follows.

Similarly when $h = H$ then by Equation~\ref{eqn:bellman_equation} we have that $Q_H^\star(x,a) = r_H(x,a)$.  Thus, $|Q_H^\star(x,a) - Q_H^\star(x',a')| = |r_H(x,a) - r_H(x',a')| \leq L_1 \D((x,a),(x',a'))$ by assumption.  Moreover, $V_H^\star(x) = \max_{a \in \A} Q_H^\star(x,a)$ is $L_1$ Lipschitz by Equation~\ref{eqn:bellman_equation} and Lemma~\ref{lemma:Lipschitz}.

For the step case we assume that $Q_{h+1}^\star$ and $V_{h+1}^\star$ are both $\sum_{i=0}^{H-h-1} L_1 L_2^{i}$ Lipschitz and show the result for $Q_h^\star$ and $V_h^\star$.  Indeed,
\begin{align*}
    & |Q_h^\star(x,a) - Q_h^\star(x', a')| = |r_h(x,a) - r_h(x', a') + \Exp{V_{h+1}^\star(\hat{x}) \mid x,a} - \Exp{V_{h+1}^\star(\hat{x}) \mid x' a'}| \\
	& \leq |r_h(x,a) - r_h(x', a')| + |\Exp{V_{h+1}^\star(\hat{x}) \mid x,a} - \Exp{V_{h+1}^\star(\hat{x}) \mid x' a'}| \\
	& \leq L_1 \D((x,a),(x',a')) + \left|\int V_{h+1}^\star(\hat{x}) \, d\Pr_h(\hat{x} \mid x,a) - \int V_{h+1}^\star(\hat{x}) \, d\Pr_h(\hat{x} \mid x', a')\right|
\end{align*}
Now denoting by $K = \sum_{i=0}^{H-h-1} L_1 L_2^{i}$ by the induction hypothesis and the properties of Wasserstein metric (Equation 3 from \cite{gibbs2002choosing}) we have that
\begin{align*}
|Q_h^\star(x,a) - Q_h^\star(x', a')| & \leq L_1 \D((x,a),(x',a')) \\
& \qquad + K \left|\int \frac{1}{K} V_{h+1}^\star(\hat{x}) \, d\Pr_h(\hat{x} \mid x,a) - \int \frac{1}{K} V_{h+1}^\star(\hat{x}) \, d\Pr_h(\hat{x} \mid x', a')\right| \\
	& \leq L_1 \D((x,a),(x',a')) + K d_W(\Pr_h(\cdot \mid x, a), \Pr_h(\cdot \mid x', a')) \\
	& \leq L_1 \D((x,a),(x',a')) + K L_2 \D((x,a), (x',a')) \\
	& = (L_1 + K L_2)\D((x,a),(x',a')).
\end{align*}
Noting that by definition of $K$ we have $L_1 + K L_2 = \sum_{i=0}^{H-h-1} L_1 L_2^{i}$.  To show that $V_h^\star$ is also Lipschitz we simply use Lemma~\ref{lemma:Lipschitz}. 
\end{proof}
\section{Policy-Identification Guarantees Proof}
\label{app:pac_proof}

\begin{theorem}
{For $K = \tilde{O}\left( \left(H^{5/2}/\pfail\epsilon \right)^{d_c+2} \right)$ (where $d_c$ is the covering dimension with parameter $c$), consider a policy $\pi$ chosen uniformly at random from $\pi_1, \ldots, \pi_K$. Then, for an initial state $X\sim F_1$, with probability at least $1-\pfail$, the policy $\pi$ obeys
\begin{equation*}
{V_1^\star(X) - V_1^\pi(X)} \leq \epsilon.
\end{equation*}
}
\end{theorem}
\noindent Note that in the above guarantee, both $X$ and $\pi$ are random.
\begin{proof}
Let $X\sim F_1$ be a random starting state sampled from $F_1$. Then from Markov's inequality and the law of total expectation, we have that:
\begin{align*}
\Pr(V_1^\star(X) - V_1^\pi(X) > c) 
		& \leq \frac{1}{K c}\sum_{k=1}^K \Exp{V_1^\star(X) - V_1^{\pi^k}(X)} \\
		& \leq \frac{1}{K c} \tilde{O}\left(H^{5/2}K^{(d_c+1)/(d_c+2)})\right)
		 = \tilde{O}\left(\frac{H^{5/2}K^{-\frac{1}{d_c+2}}}{c}\right),
\end{align*}
where the last inequality follows from the guarantee in Theorem~\ref{thm:regret}.
Setting $c = \epsilon$ and $K = \tilde{O}\left( \left(H^{5/2}/\pfail \epsilon \right)^{d_c+2} \right)$ in the above expression gives that $V_1^\star(x) - V_1^\pi(x) \leq \epsilon$ with probability at least $1 - \pfail$.
\end{proof}
\section{Proof Details}
\label{sec:detailedproof}

\subsection{Required Lemmas}
We start with a collection of lemmas which are required to show Theorem~\ref{thm:regret}.

The first two lemmas consider invariants established by the partitioning of the algorithm.  These state that the algorithm maintains a partition of the state-action space at every iteration of the algorithm, and that the balls of similar radius are sufficiently apart.

\begin{lemma}[Lemma~\ref{lemma:partition_main} from the main paper]
\label{lemma:partition}
For every $(h, k) \in [H] \times [K]$ the following invariants are maintained:
\begin{itemize}
	\item (Covering) The domains of each ball in $\P_h^k$ cover $\S \times \A$.
	\item (Separation) For any two balls of radius $r$, their centers are at distance at least r.
\end{itemize}
\end{lemma}
\begin{proof}
Let $(h, k) \in [H] \times [K]$ be arbitrary.

For the covering invariant notice that $\cup_{B \in \P_h^k} \dom{B} = \cup_{B \in \P_h^k} B$ where we are implicitly taking the domain with respect to the partition $\P_h^k$.  The covering invariant follows then since $\P_h^k$ contains a ball which covers the entire space $\S \times \A$ from the initialization in the algorithm.

{To show the separation invariant, suppose that $B_1$ and $B_2$ are two balls of radius $r$.  If $B_1$ and $B_2$ share a parent, we note by the splitting rule the algorithm maintains the invariant that the centers are at a distance at least $r$ from each other.  Otherwise, suppose without loss of generality that $B_1$ has parent $B^{par}$ and that $B_2$ was activated before $B_1$.  The center of $B_1$ is some point $(x,a) \in \dom{B^{par}}$.  It follows then that $(x,a) \not\in \dom{B_2}$ by definition of the domain as $r(B^{par}) > r(B_2)$.  Thus, their centers are at a distance at least $r$ from each other.}
\end{proof}

The second property is useful as it maintains that the centers of the balls of radius $r$ form an $r-$packing of $\S \times \A$ and hence there are at most $N_r^{pack} \leq N_r$ balls activated of radius $r$.

The next theorem gives an analysis on the number of times that a ball of a given radius will be selected by the algorithm.  {This Lemma also shows the second condition required for a partitioning algorithm to achieve the regret bound as described in Section~\ref{section:discussion}.}
\begin{lemma}[Lemma~\ref{lemma:bound_ball_main} from the main paper]
	\label{lemma:bound_ball}
	For any $h \in [H]$ and child ball $B \in \P_h^K$ (the partition at the end of the last episode $K$) the number of episodes $k \leq K$ such that $B$ is selected by the algorithm is less than $\frac{3}{4}\left(d_{max} / r \right)^2$ where $r = r(B)$.  I.e. denoting by $B_h^k$ the ball selected by the algorithm in step $h$ episode $k$,
	$$| \{ k : B_h^k = B \}| \leq \frac{3}{4} \left(\frac{d_{max}}{r}\right)^2.$$
	
	Moreover, the number of times that ball $B$ and its ancestors have been played is at least $\frac{1}{4} \left(\frac{d_{max}}{r}\right)^2$.
	
	For the case when $B$ is the initial ball which covers the entire space then the number of episodes that $B$ is selected is only one.
\end{lemma}
\begin{proof}
	Consider an arbitrary $h \in [H]$ and child ball $B \in \P_h^K$ such that $r(B) = r$.  Furthermore, let $k$ be the episode for which ball $B$ was activated.  Then $B_h^k$, the ball selected by the algorithm at step $h$ in episode $k$ is the parent of $B$.  Moreover, if $t = n_h^{k+1}(B_h^k)$ is the number of times that $B_h^k$ or it's ancestors have been played then $t = \left(\frac{d_{max}}{r(B_h^k)}\right)^2$ by the activation rule.  Also, $r(B_h^k) = 2r(B)$ by the re-partitioning scheme. Hence, the number of times that $B$ and its ancestors have been played is at least $$t = \left(\frac{d_{max}}{r(B_h^k)}\right)^2 = \left(\frac{d_{max}}{2r(B)}\right)^2 = \frac{1}{4}\left(\frac{d_{max}}{r(B)}\right)^2.$$
	
	The number of episodes that $B$ can be selected (i.e. $| \{ k : B_h^k = B \}|$) by the algorithm is bounded above by $\left(\frac{d_{max}}{r(B)}\right)^2 - \left(\frac{d_{max}}{r(B_h^k)}\right)^2$ as this is the number of samples of $B$ required to split the ball by the partitioning scheme.  However, plugging in $r(B_h^k) = 2r(B)$ gives
	\begin{align*}
	\left(\frac{d_{max}}{r(B)}\right)^2 - \left(\frac{d_{max}}{2r(B)}\right)^2 = \frac{3}{4} \left(\frac{d_{max}}{r(B)}\right)^2
	\end{align*}
	as claimed.
	
	Lastly if $B$ is the initial ball which covers the entire space then $r(B) = d_{max}$ initially and so the ball is split after it is selected only once.
\end{proof}

Next we provide a recursive relationship for the update in $Q$ estimates.

\begin{lemma}
\label{lemma:recursive_relationship}
	For any $h, k \in [H] \times [K]$ and ball $B \in \P_h^k$ let $t = n_h^k(B)$ be the number of times that $B$ or its ancestors were encountered during the algorithm before episode $k$.  Further suppose that $B$ and its ancestors were encountered at step $h$ of episodes $k_1 < k_2 < \ldots < k_t < k$.  By the update rule of $Q$ we have that:
	$$ Q_{h}^k(B) = \Ind{t = 0} H + \sum_{i=1}^t \alpha_t^i \left(r_h(x_h^{k_i}, a_h^{k_i}) + \Vhat{h+1}{k_i}(x_{h+1}^{k_i}) + b(i)\right).$$
\end{lemma}

\begin{proof}
We show the claim by induction on $t = n_h^k(B)$.  

First suppose that $t = 0$, i.e. that the ball $B$ has not been encountered before by the algorithm.  Then initially $\Qhat{h}{1}(B) = H = \Ind{t = 0} H$.

Now for the step case we notice that $\Qhat{h}{k}(B)$ was last updated at episode $k_t$.  $k_t$ is either the most recent episode when ball $B$ was encountered, or the most recent episode when it's parent was encountered if $B$ was activated and not yet played.  In either case by the update rule (Equation~\ref{eqn:update}) we have
\begin{align*}
\Qhat{h}{k}(B) & = (1 - \alpha_t)\Qhat{h}{k_t}(B) + \alpha_t\left(r_h(x_h^{k_t}, a_h^{k_t}) + \Vhat{h+1}{k_t}(x_{h+1}^{k_t}) + b(t)\right) \\
& = (1 - \alpha_t)\alpha_{t-1}^0 H + (1 - \alpha_t) \sum_{i=1}^{t-1} \alpha_t^i \left(r_h(x_h^{k_i}, a_h^{k_i}) + \Vhat{h+1}{k_i}(x_{h+1}^{k_i}) + b(i)\right) \\ & + \alpha_t\left(r_h(x_h^{k_t}, a_h^{k_t}) + V_{h+1}^{k_t}(x_{h+1}^{k_t}) + b(t)\right) \text{ by the induction hypothesis}\\
& = \Ind{t = 0} H + \sum_{i=1}^t \alpha_t^i \left(r_h(x_h^{k_i}, a_h^{k_i}) + \Vhat{h+1}{k_i}(x_{h+1}^{k_i}) + b(i)\right)
\end{align*}
by definition of $\alpha_t^i$.
\end{proof}

The next lemma extends the relationship between the optimal $Q$ value $Q_h^\star(x,a)$ for any $(x, a) \in \S \times \A$ to the estimate of the $Q$ value for any ball $B$ containing $(x,a)$.  For ease of notation we denote $(\Vhat{h}{k}(x) - V_h^\star(x)) = (\Vhat{h}{k} - V_h^\star)(x)$.

\begin{lemma}
\label{lemma:recursive_difference}
	For any $(x, a, h, k) \in \S \times \A \times [H] \times [K]$ and ball $B \in \P_h^k$ such that $(x, a) \in \dom{B}$ then if $t = n_h^k(B)$ and $B$ and its ancestors were previously encountered at step $h$ in episodes $k_1 < k_2 < \ldots < k_t < k$ then 
	\begin{align*}
	\Qhat{h}{k}(B) - Q_h^\star(x,a) & = \Ind{t = 0}(H - Q_h^\star(x,a)) +  \sum_{i=1}^t \alpha_t^i \Big((\Vhat{h+1}{k_i} - V_{h+1}^\star)(x_{h+1}^{k_i}) + V_{h+1}^\star(x_{h+1}^{k_i})\\
	& \hspace{1cm} - \Exp{V_{h+1}^\star(\hat{x}) \mid x_h^{k_i}, a_h^{k_i}} + b(i) + Q_h^\star(x_h^{k_i}, a_h^{k_i}) - Q_h^\star(x,a) \Big)
	\end{align*}
\end{lemma}
\begin{proof}
Consider any $(x, a) \in \S \times \A$ and $h \in [H]$ arbitrary and an episode $k \in [K]$.  Furthermore, let $B$ be any ball such that $(x, a) \in B$.  First notice by Lemma~\ref{lemma:lr} that $\Ind{t = 0} + \sum_{i=1}^t \alpha_t^i = 1$ for any $t \geq 0$.
	
We show the claim by induction on $t = n_h^k(B)$.
	
For the base case that $t = 0$ then $\Qhat{h}{k}(B) = H$ as it has not been encountered before by the algorithm.  Then $$\Qhat{h}{k}(B) - Q_h^\star(x,a) = H - Q_h^\star(x,a) = \Ind{t = 0}(H - Q_h^\star(x,a)).$$
	
For the step case we use the fact that $\sum_{i=1}^t \alpha_t^i + \Ind{t = 0} = 1$.  Thus, $Q_h^\star(x,a) = \sum_{i=1}^t \alpha_t^i Q_h^\star(x,a) + \Ind{t = 0} Q_h^\star(x,a)$.  Subtracting this from $\Qhat{h}{k}(B)$ and using Lemma~\ref{lemma:recursive_relationship} yields
\begin{align*}
    \Qhat{h}{k}(B) - Q_h^\star(x,a) & = \Ind{t = 0} H + \sum_{i=1}^t \alpha_t^i \left(r_h(x_h^{k_i}, a_h^{k_i}) + \Vhat{h+1}{k_i}(x_{h+1}^{k_i}) + b(i)\right) \\
    & \qquad - \Ind{t = 0} Q_h^\star(x,a) - \sum_{i=1}^t \alpha_t^i Q_h^\star(x,a) \\
	& = \Ind{t = 0}(H - Q_h^\star(x,a)) + \sum_{i=1}^t \alpha_t^i \Big( r_h(x_h^{k_i}, a_h^{k_i}) + \Vhat{h+1}{k_i}(x_{h+1}^{k_i}) + b(i) \\
	& \qquad + Q_h^\star(x_h^{k_i}, a_h^{k_i}) - Q_h^\star(x_h^{k_i}, a_h^{k_i}) - Q_h^\star(x,a) \Big).
\end{align*}
However, using the Bellman Equations (\ref{eqn:bellman_equation}) we know that $$Q_h^\star(x_h^{k_i}, a_h^{k_i}) = r_h(x_h^{k_i}, a_h^{k_i}) + \Exp{V_{h+1}^\star(\hat{x}) \mid x_h^{k_i}, a_h^{k_i}}.$$  Using this above we get that
\begin{align*}
		\Qhat{h}{k}(B) - Q_h^\star(x,a) & = \Ind{t = 0}(H - Q_h^\star(x,a)) + \sum_{i=1}^t \alpha_t^i \Big(\Vhat{h+1}{k_i}(x_{h+1}^{k_i}) - \Exp{V_{h+1}^\star(\hat{x})\mid x_h^{k_i}, a_h^{k_i}} \\
		&  + b(i) + Q_h^\star(x_h^{k_i}, a_h^{k_i}) -  Q_h^\star(x,a)\Big) \\
		& = \Ind{t = 0}(H - Q_h^\star(x,a)) +  \sum_{i=1}^t \alpha_t^i \Big((\Vhat{h+1}{k_i} - V_{h+1}^\star)(x_{h+1}^{k_i}) \\ & + V_{h+1}^\star(x_{h+1}^{k_i}) - \Exp{V_{h+1}^\star(\hat{x}) \mid x_h^{k_i}, a_h^{k_i}} + b(i) + Q_h^\star(x_h^{k_i}, a_h^{k_i}) -  Q_h^\star(x,a) \Big).
\end{align*}
\end{proof}

Consider now $V_{h+1}^\star(x_{h+1}^{k_i}) - \Exp{V_{h+1}^\star(\hat{x}) \mid x_h^{k_i}, a_h^{k_i}}$.  We notice that as the next state $x_{h+1}^{k_i}$ is drawn from $\Pr_h(\cdot \mid x_h^{k_i}, a_h^{k_i})$ then $\Exp{V_{h+1}^\star(x_{h+1}^{k_i})} = \Exp{V_{h+1}^\star(\hat{x}) \mid x_h^{k_i}, a_h^{k_i}}$.  Thus, this sequence forms a martingale difference sequence and so by Azuma-Hoeffding's inequality we are able to show the following.

\begin{lemma}
\label{lemma:concentration_bound}
For all $(x,a, h, k) \in \S \times \A \times [H] \times [K]$ and for all $\pfail \in (0, 1)$ we have with probability at least $1 - \pfail/2$ if $(x,a) \in B$ for some $B \in \P_h^k$ then for $t = n_h^k(B)$ and episodes $k_1 < k_2 < \ldots < k_t < k$ where $B$ and its ancestors were encountered at step $h$ before episode $k$ then $$\left| \sum_{i=1}^t \alpha_t^i \left( V_{h+1}^\star(x_{h+1}^{k_i}) - \Exp{V_{h+1}^\star(\hat{x}) \mid x_h^{k_i}, a_h^{k_i}} \right) \right| \leq H\sqrt{2 \sum_{i=1}^t (\alpha_t^i)^2 \log(4HK/\pfail)}.$$
\end{lemma}
\begin{proof}
Consider the sequence $$k_i = \min(k, \min\{ \hat{k} : n^{\hat{k}}_h(B^a) = i \text{ and } B^a \text{ is an ancestor of } B \text{ or } B \text{ itself}\}).$$  Clearly $k_i$ is the episode for which $B$ or its ancestors were encountered at step $h$ for the $i$-th time (as once a ball is split it is never chosen by the algorithm again).  Setting $$Z_i = \Ind{k_i \leq k}\left( V_{h+1}^\star(x_{h+1}^{k_i}) - \Exp{V_{h+1}^\star(\hat{x}) \mid x_h^{k_i}, a_h^{k_i}} \right)$$ then $Z_i$ is a martingale difference sequence with respect to the filtration $\mathcal{F}_i$ which we denote as the information available to the agent up to an including the step $k_i$.  Moreover, as the sum of a martingale difference sequence is a martingale then for any $\tau \leq K$, $\sum_{i=1}^\tau Z_i$ is a martingale.  Noticing that the difference between subsequent terms is bounded by $H \alpha_\tau^{i}$ and Azuma-Hoeffding's inequality we see that for a fixed $\tau \leq K$

\begin{align*}
&\Pr \left( \left| \sum_{i=1}^\tau \alpha_\tau^i Z_i \right| \leq H\sqrt{2 \sum_{i=1}^\tau (\alpha_\tau^i)^2 \log\left(\frac{4 HK }{\pfail}\right)} \right) \\
&\qquad \geq 1 - 2\exp \left( - \frac{2H^2 \sum_{i=1}^\tau (\alpha_\tau^i)^2 \log(\frac{4HK}{\pfail})}{2 H^2 \sum_{i=1}^\tau (\alpha_\tau^i)^2} \right) \\
&\qquad = 1 - \frac{\pfail}{2HK}.
\end{align*}
Since the inequality holds for fixed $\tau \leq K$ the the inequality must also hold for the random stopping time $\tau = t = n_h^k(B) \leq K$.  Moreover, for this stopping time each of the indicator functions will be $1$.

{Taking a union bound over the number of episodes and over all $H$ the result follows.  We only need to union bound over the number of episodes instead of the number of balls as the inequality is satisfied for all balls not selected in a given round as it inherits its concentration from its parent ball because the value for $t$ does not change.}

We also notice that for $t > 0$ that $\sum_{i=1}^t (\alpha_t^i)^2 \leq \frac{2H}{t}$ by Lemma~\ref{lemma:lr} and so $$H\sqrt{2 \sum_{i=1}^t (\alpha_t^i)^2 \log(4HK/\pfail)} \leq H \sqrt{2 \frac{2H}{t} \log(4HK/\pfail)} = 2 \sqrt{\frac{H^3 \log(4HK/\pfail)}{t}}.$$
\end{proof}

The next lemma uses the activation rule and Assumption~\ref{assumption:Lipschitz} to bound the difference between the optimal $Q$ functions at a point $(x,a)$ in $B$ from the points sampled in ancestors of $B$ by the algorithm.  This corresponds to the accumulation of discretization errors in the algorithm by accumulating estimates over a ball.

\begin{lemma}
\label{lemma:radius_ball_bonus}
For any $(x,a,h,k) \in \S \times \A \times [H] \times [K]$ and ball $B \in \P_h^k$ with $(x,a) \in \dom{B}$ if $B$ and its ancestors were encountered at step $h$ in episodes $k_1 < k_2 < \ldots < k_t < k$ where $t = n_h^k(B)$ then

$$\sum_{i=1}^t \alpha_t^i \left|Q_h^\star(x_h^{k_i}, a_h^{k_i}) - Q_h^\star(x,a)\right| \leq \frac{4Ld_{max}}{\sqrt{t}}.$$
\end{lemma}
\begin{proof}
Denote by $B_h^{k_i}$ as the ball selected in step $h$ of episode $k_i$ and $n_h^{k_i}$ as $n_h^{k_i}(B_h^{k_i})$.  Then each $B_h^{k_i}$ is an ancestor of $B$ and s both $(x,a)$ and $(x_h^{k_i}, a_h^{k_i})$.  Hence by the Lipschitz assumption (Assumption~\ref{assumption:Lipschitz}) we have that $|Q_h^\star(x_h^{k_i}, a_h^{k_i}) - Q_h^\star(x,a)| \leq L \diam{B_h^{k_i}} \leq 2L r(B_h^{k_i}).$

However, $r(B_h^{k_i}) \leq \frac{d_{max}}{\sqrt{i}}$.  Indeed, by the re-partition rule we split when $n_h^{k_i} = \left( \frac{d_{max}}{r(B_h^{k_i})} \right)^2$ (and afterwards $B_h^{k_i}$ is not chosen again) and so $n_h^{k_i} \leq \left( \frac{d_{max}}{r(B_h^{k_i})} \right)^2$.  Square rooting this gives that $\sqrt{n_h^{k_i}} \leq  \frac{d_{max}}{r(B_h^{k_i})}$.  However, as $n_h^{k_i} = i$ we get that $r(B_h^{k_i}) \leq \frac{d_{max}}{\sqrt{i}}.$  Using this we have that
\begin{align*}
\sum_{i=1}^t \alpha_t^i \left|Q_h^\star(x_h^{k_i}, a_h^{k_i}) - Q_h^\star(x,a)\right| & \leq \sum_{i=1}^t \alpha_t^i 2L r(B_h^{k_i}) \\
& \leq 2\sum_{i=1}^t \alpha_t^i L d_{max} \frac{1}{\sqrt{i}} \\
& \leq 4 L d_{max} \frac{1}{\sqrt{t}} \text{ by Lemma}~\ref{lemma:lr}.
\end{align*}
\end{proof}

The next lemma provides an upper and lower bound on the difference between $\Qhat{h}{k}(B)$ and $Q_h^\star(x,a)$ for any $(x,a) \in \dom{B}$.  It also shows that our algorithm is \textit{optimistic}, in that the estimates are upper bounds for the true quality and value function for the optimal policy \cite{simchowitz2019nonasymptotic}.

\begin{lemma}[Expanded version of Lemma~\ref{lemma:upper_lower_bounds_main} from the main paper]
\label{lemma:upper_lower_bounds}
For any $\pfail \in (0,1)$ if $\beta_t = 2\sum_{i=1}^t \alpha_t^i b(i)$ then $$\beta_t \leq 8\sqrt{\frac{H^3 \log(4HK/\pfail)}{t}} + 16\frac{L d_{max}}{\sqrt{t}}$$ and $$\beta_t \geq 4\sqrt{\frac{H^3 \log(4HK/\pfail)}{t}} +  8\frac{L d_{max}}{\sqrt{t}}.$$  Moreover, with probability at least $1 - \pfail/2$ the following holds simultaneously for all $(x,a,h,k) \in \S \times \A \times [H] \times [K]$ and ball $B$ such that $(x,a) \in \dom{B}$ where $t = n_h^k(B)$ and $k_1 < \ldots < k_t$ are the episodes where $B$ or its ancestors were encountered previously by the algorithm
\begin{align*}
	0 \leq \Qhat{h}{k}(B) - Q_h^\star(x,a) \leq \Ind{t = 0}H + \beta_t + \sum_{i=1}^t \alpha_t^i(\Vhat{h+1}{k_i} - V_{h+1}^\star)(x_{h+1}^{k_i})
\end{align*}
\end{lemma}
\begin{proof}
First consider $\beta_t = 2 \sum_{i=1}^t \alpha_t^i b(i)$.  By definition of the bonus term as $$b(t) = 2 \sqrt{\frac{H^3 \log(4HK/\pfail)}{t}} + 2 \frac{L d_{max}}{\sqrt{t}}$$ we have the following (where we use Lemma~\ref{lemma:lr}):
\begin{align*}
	\beta_t & = 2 \sum_{i=1}^t \alpha_t^i b(i) \\
	& = 2 \sum_{i=1}^t \alpha_t^i (2 \sqrt{\frac{H^3 \log(4HK/\pfail)}{i}} + 4 \frac{L d_{max}}{\sqrt{i}}) \\
	& \leq 4 (2 \sqrt{\frac{H^3 \log(4HK/\pfail)}{t}} + 4 \frac{L d_{max}}{\sqrt{t}}) \\
	& = 8 \sqrt{\frac{H^3 \log(4HK/\pfail)}{t}} + 16\frac{L d_{max}}{\sqrt{t}}
\end{align*}
We can similarly lower-bound the expression by $4 \sqrt{\frac{H^3 \log(4HK/\pfail)}{t}} + 8 \frac{Ld_{max}}{\sqrt{t}}$ using the lower bound from Lemma~\ref{lemma:lr}.
	
We start with the upper bound on $\Qhat{h}{k}(B) - Q_h^\star(x,a)$.  Indeed, by Lemma~\ref{lemma:recursive_difference} we have
\begin{align*}
	\Qhat{h}{k}(B) - Q_h^\star(x,a) & = \Ind{t = 0}(H - Q_h^\star(x,a)) +  \sum_{i=1}^t \alpha_t^i \Big((\Vhat{h+1}{k_i} - V_{h+1}^\star)(x_{h+1}^{k_i}) \\
	& + V_{h+1}^\star(x_{h+1}^{k_i}) - \Exp{V_{h+1}^\star(\hat{x}) \mid x_h^{k_i}, a_h^{k_i}} + b(i) + Q_h^\star(x_h^{k_i}, a_h^{k_i}) - Q_h^\star(x,a) \Big).
\end{align*}

However, with probability at least $1 - \pfail/2$ by Lemma~\ref{lemma:concentration_bound} and Lemma~\ref{lemma:radius_ball_bonus} we get
\begin{align*}
	& \leq \Ind{t = 0} H + \sum_{i=1}^t \alpha_t^i b(i) + 2\sqrt{\frac{H^3 \log(4HK/\pfail)}{t}} + 4\frac{L d_{max}}{\sqrt{t}} + \sum_{i=1}^t \alpha_t^i (\Vhat{h+1}{k_i} - V_{h+1}^\star)(x_{h+1}^{k_i}) \\
	& \leq \Ind{t = 0} H + \sum_{i=1}^t \alpha_t^i b(i) + \frac{\beta_t}{2} + \sum_{i=1}^t \alpha_t^i (\Vhat{h+1}{k_i} - V_{h+1}^\star)(x_{h+1}^{k_i})
\end{align*}

Using that $\sum_{i=1}^t \alpha_t^i b(i) = \beta_t / 2$ establishes the upper bound.

For the lower bound we show the claim by induction on $h = H+1, H, \ldots, 1$.

Indeed, for the base case when $h = H+1$ then $Q_{H+1}^\star(x,a) = 0 = \Qhat{H+1}{k}(B)$ for every $k$, ball $B$, and $(x,a) \in \S \times \A$ trivially as at the end of the episode the expected future reward is always zero.

Now, assuming the claim for $h+1$ we show the claim for $h$. First consider $\Vhat{h+1}{k_i}(x_{h+1}^{k_i}) - V_{h+1}^\star(x_{h+1}^{k_i})$.  We show that $\Vhat{h+1}{k_i}(x_{h+1}^{k_i}) - V_{h+1}^\star(x_{h+1}^{k_i}) \geq 0$. By the Bellman Equations (\ref{eqn:bellman_equation}) we know that $$V_{h+1}^\star(x_{h+1}^{k_i}) = \sup_{a \in \A} Q_{h+1}^\star(x_{h+1}^{k_i}, a) = Q_{h+1}^\star(x_{h+1}^{k_i}, \pi_{h+1}^\star(x_{h+1}^{k_i})).$$

If $\Vhat{h+1}{k_i}(x_{h+1}^{k_i}) = H$ then the inequality trivially follows as $V^\star(x_{h+1}^{k_i}) \leq H$.  Thus by the update in the algorithm we can assume that $\Vhat{h+1}{k_i}(x_{h+1}^{k_i}) =  \max_{B \in \texttt{RELEVANT}_{h+1}^{k_i} (x_{h+1}^{k_i})} \Qhat{h}{k_i}(B)$.  Now let $B^\star$ be the ball with smallest radius in $\P_{h+1}^{k_i}$ such that $(x_{h+1}^{k_i}, \pi_{h+1}^\star(x_{h+1}^{k_i})) \in B^\star$.  Such a ball exists as $\P_{h+1}^{k_i}$ covers $\S \times \A$ by Lemma~\ref{lemma:partition}.  Moreover $(x_{h+1}^{k_i}, \pi_{h+1}^\star(x_{h+1}^{k_i})) \in \dom{B^\star}$ as $B^\star$ was taken to have the smallest radius.  Thus, $\Qhat{h+1}{k_i}(B^\star) \geq Q_{h+1}^\star(x_{h+1}^{k_i}, \pi_{h+1}^\star(x_{h+1}^{k_i}))$ by the induction hypothesis.  Hence we have that $$\Vhat{h+1}{k_i}(x_{h+1}^{k_i}) \geq \Qhat{h+1}{k_i}(B^\star) \geq Q_{h+1}^\star(x_{h+1}^{k_i}, \pi_{h+1}^\star(x_{h+1}^{k_i})) = V_{h+1}^\star(x_{h+1}^{k_i}).$$
Putting all of this together with Lemma~\ref{lemma:recursive_difference} then with probability $1 - \pfail/2$ we have (using Lemma~\ref{lemma:concentration_bound} and definition of $\beta_t/2$) that
\begin{align*}
	\Qhat{h}{k}(B) - Q_h^\star(x,a) & = \Ind{t = 0}(H - Q_h^\star(x,a)) +  \sum_{i=1}^t \alpha_t^i \Big((\Vhat{h+1}{k_i} - V_{h+1}^\star)(x_{h+1}^{k_i}) \\
& + V_{h+1}^\star(x_{h+1}^{k_i}) - \Exp{V_{h+1}^\star(\hat{x}) \mid x_h^{k_i}, a_h^{k_i}} + b(i) + Q_h^\star(x_h^{k_i}, a_h^{k_i}) - Q_h^\star(x,a) \Big) \\
& \geq \sum_{i=1}^t \alpha_t^i(\Vhat{h+1}{k_i} - V_{h+1}^\star)(x_{h+1}^{k_i}) + \frac{\beta_t}{2} - 2\sqrt{H^3 \log(4HK/\pfail) / t} - 4 \frac{L d_{max}}{\sqrt{t}} \\
& \geq \frac{\beta_t}{2} - \frac{\beta_t}{2} \geq 0
\end{align*}
where in the last line we used that $\Vhat{h+1}{k_i}(x_{h+1}^{k_i}) \geq V_{h+1}^\star(x_{h+1}^{k_i})$ from before.
\end{proof}

The next Lemma extends the results from Lemma~\ref{lemma:upper_lower_bounds} to lower bounds of $\Vhat{h}{k} - V_h^\star$ over the entire state space $\S$.

\begin{corollary}
\label{lemma:upper_lower_bounds_value_fn}
	For any $\pfail \in (0,1)$ with probability at least $1 - \pfail/2$ the following holds simultaneously for all $(x,h,k) \in \S \times [H] \times [K]$ 
	\begin{align*}
	\Vhat{h}{k}(x) - V_h^\star(x) \geq 0
	\end{align*}
\end{corollary}
\begin{proof}
The proof of the lower bound follows from the argument used in the proof of Lemma~\ref{lemma:upper_lower_bounds}.
\end{proof}

The next lemma uses Azuma-Hoeffding's inequality to bound a martingale difference sequence which arises in the proof of the regret bound.

\begin{lemma}
\label{lemma:martingale_difference_sequence}
For any $\pfail \in (0, 1)$ with probability at least $1 - \pfail/2$ we have that $$\sum_{h=1}^H \sum_{k=1}^K \left|\Exp{V_{h+1}^\star(\hat{x}) - V_{h+1}^{\pi^k}(\hat{x}) \mid x_h^k, a_h^k} - (V_{h+1}^\star(x_{h+1}^k) - V_{h+1}^{\pi^k}(x_h^k, a_h^k))\right| \leq 2\sqrt{2H^3 K \log(4HK/\pfail)}.$$
\end{lemma}

\begin{proof}

First consider $Z_h^k = \Exp{V_{h+1}^\star(\hat{x}) - V_{h+1}^{\pi^k}(\hat{x}) \mid x_h^k, a_h^k} - (V_{h+1}^\star(x_{h+1}^k) - V_{h+1}^{\pi^k}(x_h^k))$.  Similar to the proof of Lemma~\ref{lemma:concentration_bound} we notice that $Z_h^k$ is a martingale difference sequence due to the fact that the next state is drawn from the distribution $\Pr_h(\cdot \mid x_h^k, a_h^k)$.  Using that $|Z_h^k| \leq 2H$ we have that

\begin{align*}
& \Pr\left(\sum_{h=1}^H \sum_{k=1}^K \left|\Exp{V_{h+1}^\star(\hat{x}) - V_{h+1}^{\pi^k}(\hat{x}) \mid x_h^k, a_h^k} - (V_{h+1}^\star(x_{h+1}^k) - V_{h+1}^{\pi^k}(x_h^k, a_h^k))\right| > \sqrt{8 H^3 K \log(4HK/\pfail)}\right) \\
& \leq 2\exp\left( -\frac{8 H^3 K \log(4HK/\pfail)}{2HK(2H)^2}\right) \\
& = 2 \exp  \left( -\frac{8 H^3 K \log(4HK/\pfail)}{8H^3 K} \right) \\
& = 2 \frac{\pfail}{4KH} \leq \pfail/2.
\end{align*}

Thus with probability at least $1 - \pfail/2$ we have that $$\sum_{h=1}^H \sum_{k=1}^K \left|\Exp{V_{h+1}^\star(\hat{x}) - V_{h+1}^{\pi^k}(\hat{x}) \mid x_h^k, a_h^k} - (V_{h+1}^\star(x_{h+1}^k) - V_{h+1}^{\pi^k}(x_h^k, a_h^k))\right| \leq \sqrt{8H^3 K \log(4HK/\pfail)}$$
as claimed.
\end{proof}
This next lemma provides a bound on the sum of errors accumulated from the confidence bounds $\beta_t$ throughout the algorithm.  This term gives rise to the $N_r$ covering terms from Theorem~\ref{thm:regret}.  {This Lemma can also show the third condition required for a partitioning algorithm to achieve the regret bound as described in Section~\ref{section:discussion}.}
\begin{lemma}\label{lemma:bound_on_beta_summation}
For every $h \in [H]$, if $B_h^k$ is the ball selected by the algorithm in step $h$ episode $k$ and $n_h^k = n_h^k(B_h^k)$ then
$$\sum_{k=1}^K \beta_{n_h^k} \leq 32\left( \sqrt{H^3 \log(4HK/\pfail)} + L d_{max} \right) \inf_{r_0 \in (0, d_{max}]} \left(\sum_{\substack{r = d_{max} 2^{-i} \\ r \geq r_0}} N_r \frac{d_{max}}{r} + \frac{K r_0}{d_{max}}\right).$$
\end{lemma}

\begin{proof}
Using Lemma~\ref{lemma:upper_lower_bounds},
\begin{align*}
	\sum_{k=1}^K \beta_{n_h^k} & \leq \sum_{k=1}^K 8 \sqrt{\frac{H^3 \log(4HK/\pfail)}{n_h^k}} + 16\frac{L d_{max}}{\sqrt{n_h^k}} \\
	& \leq 16\left( \sqrt{H^3 \log(4HK/\pfail)} + L d_{max} \right) \sum_{k=1}^K \frac{1}{\sqrt{n_h^k}}.
\end{align*}
We thus turn our attention to $\sum_{k=1}^K 1/\sqrt{n_h^k}$.  Rewriting the sum in terms of the radius of all balls activated by the algorithm we get for an arbitrary $r_0 \in (0, d_{max}]$,
\begin{align*}
	\sum_{k=1}^K \frac{1}{\sqrt{n_h^k}} & = \sum_{r = d_{max} 2^{-i}} \sum_{\substack{B \in \P_h^K\\ r(B) = r}} \sum_{k : B_h^k = B} \frac{1}{\sqrt{n_h^k(B)}} \\
	& = \sum_{\substack{r = d_{max} 2^{-i} \\ r \geq r_0}} \sum_{\substack{B \in \P_h^K\\ r(B) = r}} \sum_{k : B_h^k = B} \frac{1}{\sqrt{n_h^k(B)}} + \sum_{\substack{r = d_{max} 2^{-i} \\ r < r_0}} \sum_{\substack{B \in \P_h^K\\ r(B) = r}} \sum_{k : B_h^k = B} \frac{1}{\sqrt{n_h^k(B)}},
\end{align*}
taking the two cases when $r \geq r_0$ and $r < r_0$ separately. We first start with $r < r_0$.  Then,
\begin{align*}
& \sum_{\substack{r = d_{max} 2^{-i} \\ r < r_0}} \sum_{\substack{B \in \P_h^K\\ r(B) = r}} \sum_{k : B_h^k = B} \frac{1}{\sqrt{n_h^k(B)}} \\
&\qquad \leq \sum_{\substack{r = d_{max} 2^{-i} \\ r < r_0}} \sum_{\substack{B \in \P_h^K\\ r(B) = r}} \sum_{k : B_h^k = B} \frac{1}{\sqrt{\frac{1}{4}d_{max}^2 / r^2}} \text{ by Lemma}~\ref{lemma:bound_ball} \\
&\qquad \leq \sum_{\substack{r = d_{max} 2^{-i} \\ r < r_0}} \sum_{\substack{B \in \P_h^K\\ r(B) = r}} \sum_{k : B_h^k = B} \frac{2 r_0}{d_{max}} \leq \frac{2 K r_0}{d_{max}}
\end{align*}
by bounding the number of terms in the sum by the number of episodes $K$.

For the case when $r \geq r_0$ we get
\begin{align*}
	\sum_{\substack{r = d_{max} 2^{-i} \\ r \geq r_0}} \sum_{\substack{B \in \P_h^K\\ r(B) = r}} \sum_{k : B_h^k = B} \frac{1}{\sqrt{n_h^k(B)}} & \leq \sum_{\substack{r = d_{max} 2^{-i} \\ r \geq r_0}} \sum_{\substack{B \in \P_h^K\\ r(B) = r}} \sum_{i=1}^{\frac{3}{4}\left(\frac{d_{max}}{r}\right)^2} \frac{1}{\sqrt{i+\frac{1}{4}\left(\frac{d_{max}}{r}\right)^2}} \text{ by Lemma~\ref{lemma:bound_ball}} \\
	& \leq \sum_{\substack{r = d_{max} 2^{-i} \\ r \geq r_0}} \sum_{\substack{B \in \P_h^K\\ r(B) = r}} \int_{1}^{\frac{3}{4}\left(\frac{d_{max}}{r}\right)^2} \frac{1}{\sqrt{x+\frac{1}{4}\left(\frac{d_{max}}{r}\right)^2}}dx \\
	& = \sum_{\substack{r = d_{max} 2^{-i} \\ r \geq r_0}} \sum_{\substack{B \in \P_h^K\\ r(B) = r}} 2\frac{d_{max}}{r} \\
	& \leq \sum_{\substack{r = d_{max} 2^{-i} \\ r \geq r_0}} 2 N_r \frac{d_{max}}{r}
\end{align*}
as by Lemma~\ref{lemma:partition} the centers of balls of radius $r$ are at a distance at least $r$ from each other and thus form an $r$-packing of $\S \times \A$.  Hence, the total number of balls of radius $r$ is at most $N_r$.  Thus we get that (taking the inf over $r_0$ arbitrary):
\begin{align*}
\sum_{k=1}^K \frac{1}{\sqrt{n_h^k}} & \leq \inf_{r_0 \in (0, d_{max}]} \left(\sum_{\substack{r = d_{max} 2^{-i} \\ r \geq r_0}} 2 N_r \frac{d_{max}}{r} + \frac{2 K r_0}{d_{max}} \right).
\end{align*}
Plugging this back into the first inequality gives that
\begin{align*}
    \sum_{k=1}^K \beta_{n_h^k} & \leq 16\left( \sqrt{H^3 \log(4HK/\pfail)} + L d_{max} \right) \sum_{k=1}^K \frac{1}{\sqrt{n_h^k}} \\
    & \leq 16\left( \sqrt{H^3 \log(4HK/\pfail)} + L d_{max} \right) \inf_{r_0 \in (0, d_{max}]} \left( \sum_{\substack{r = d_{max} 2^{-i} \\ r \geq r_0}} 2 N_r \frac{d_{max}}{r} + \frac{2 K r_0}{d_{max}}\right) \\
    & = 32\left( \sqrt{H^3 \log(4HK/\pfail)} + L d_{max} \right) \inf_{r_0 \in (0, d_{max}]} \left( \sum_{\substack{r = d_{max} 2^{-i} \\ r \geq r_0}} N_r \frac{d_{max}}{r} + \frac{K r_0}{d_{max}} \right).
\end{align*}
\end{proof}

\subsection{Regret Analysis}

{We start with a lemma which was highlighted in the proof sketch from Section~\ref{section:sketch}.}

\begin{lemma}[Lemma~\ref{lemma:induction_main} from the main paper] \label{lemma:induction}
For any $\pfail \in (0,1)$ if $\beta_t = 2\sum_{i=1}^t \alpha_t^i b(i)$ then with probability at least $1 - \pfail/2$, for all $h \in [H]$,
\begin{align*}
    \sum_{k=1}^K (\Vhat{h}{k} - V_h^{\pi^k})(x_h^k) \leq \sum_{k=1}^K (H \Ind{n_h^k = 0} + \beta_{n_h^k} + \xi_{h+1}^k) + \left(1 + \frac{1}{H}\right) \sum_{k=1}^K (\Vhat{h+1}{k} - V_{h+1}^{\pi^k})(x_{h+1}^k).
\end{align*}
\end{lemma}

\begin{proof}[Proof of Lemma \ref{lemma:induction}]
By equations~\ref{eqn:v_update}~and~\ref{eqn:bellman_equation}, and the definition of the selection rule by the algorithm it follows that for any $h$ and $k$ that
\begin{align*}
    \Vhat{h}{k}(x_h^k) - V_h^{\pi^k}(x_h^k) & \leq \max_{B \in \texttt{RELEVANT}_h^k(x_h^k)} \Qhat{h}{k}(B) - Q_h^{\pi^k}(x_h^k, a_h^k) \\
    & = \Qhat{h}{k}(B_h^k) - Q_h^{\pi^k}(x_h^k, a_h^k) \\
    & = \Qhat{h}{k}(B_h^k) - Q_h^\star(x_h^k, a_h^k) + Q_h^\star(x_h^k, a_h^k) - Q_h^{\pi^k}(x_h^k, a_h^k).
\end{align*}
First, by Equation~\ref{eqn:bellman_equation} we know that $Q_h^\star(x_h^k, a_h^k) - Q_h^{\pi^k}(x_h^k, a_h^k) = \Exp{V_{h+1}^\star(\hat{x}) - V_{h+1}^{\pi^k}(\hat{x}) \mid x_h^k, a_h^k}$.  Moreover, as $(x_h^k, a_h^k) \in \dom{B_h^k}$ we can use Lemma~\ref{lemma:upper_lower_bounds_main} and get for $t = n_h^k(B_h^k)$ and episodes $k_1 < \ldots < k_t$ where $B_h^k$ or its ancestors were previously encountered 
\begin{align*}
    & \Qhat{h}{k}(B_h^k) - Q_h^\star(x_h^k, a_h^k) + Q_h^\star(x_h^k, a_h^k) - Q_h^{\pi^k}(x_h^k, a_h^k) \\
    & \leq \Ind{t = 0}H + \beta_t + \sum_{i=1}^t \alpha_t^i(\Vhat{h+1}{k_i} - V_{h+1}^\star)(x_{h+1}^{k_i}) + \Exp{V_{h+1}^\star(\hat{x}) - V_{h+1}^{\pi^k}(\hat{x}) \mid x_h^k, a_h^k} \\
    & \leq \Ind{t = 0}H + \beta_t + \sum_{i=1}^t \alpha_t^i(\Vhat{h+1}{k_i} - V_{h+1}^\star)(x_{h+1}^{k_i}) + (V_{h+1}^\star- V_{h+1}^{\pi^k})(x_{h+1}^k) + \xi_{h+1}^k
\end{align*}
where $\xi_{h+1}^k = \Exp{V_{h+1}^\star(\hat{x}) - V_{h+1}^{\pi^k}(\hat{x}) \mid x_h^k, a_h^k} - (V_{h+1}^\star- V_{h+1}^{\pi^k})(x_{h+1}^k)$.  Taking the sum over all episodes $k$ and letting $n_h^k = n_h^k(B_h^k)$ and the respective episodes $k_i(B_h^k)$ as the time $B_h^k$ or its ancestors were selected for the i'th time,
\begin{align*}
  \sum_{k=1}^K \Vhat{h}{k}(x_h^k) - V_h^{\pi^k}(x_h^k) &\leq \sum_{k=1}^K \left(\Ind{n_h^k = 0}H + \beta_{n_h^k}\right) + \sum_{k=1}^K \sum_{i=1}^{n_h^k} \alpha_{n_h^k}^i(\Vhat{h+1}{k_i(B_h^k)} - V_{h+1}^\star)(x_{h+1}^{k_i(B_h^k)}) \nonumber \\
	&\qquad + \sum_{k=1}^K \left((V_{h+1}^\star- V_{h+1}^{\pi^k})(x_{h+1}^k) + \xi_{h+1}^k\right). \numberthis \label{eqn:recursive_proof_sketch}
\end{align*}
For the second term we rearrange the summation using the observation used in the proof from \cite{jin_is_2018,song_efficient_2019}. For every $k' \in [K]$ the term $(\Vhat{h+1}{k'} - V_{h+1}^\star)(x_{h+1}^{k'})$ appears in the summand when $k = n_h^{k'}$. The next time it appears when $k = n_h^{k'} + 1$ and so on.  Hence by rearranging and using Lemma~\ref{lemma:lr}, it follows that
\begin{align*}
\sum_{k=1}^K \sum_{i=1}^{n_h^k} \alpha_{n_h^k}^i (\Vhat{h+1}{k_i(B_h^k)} - V_{h+1}^\star)(x_{h+1}^{k_i(B_h^k)})
&\leq \sum_{k=1}^K (\Vhat{h+1}{k} - V_{h+1}^\star)(x_{h+1}^{k}) \sum_{t = n_{h}^{k}}^\infty \alpha_t^{n_h^{k}} \\
&\leq \left(1 + \frac{1}{H}\right)\sum_{k=1}^K (\Vhat{h+1}{k} - V_{h+1}^\star)(x_{h+1}^{k}).
\end{align*}
Using these inequalities in Equation~\ref{eqn:recursive_proof_sketch} we have that
\begin{align*}
    \sum_{k=1}^K \Vhat{h}{k}(x_h^k) - V_h^{\pi^k}(x_h^k) &\leq \sum_{k=1}^K (H \Ind{n_h^k = 0} + \beta_{n_h^k} + \xi_{h+1}^k) + \left(1 + \frac{1}{H}\right) \sum_{k=1}^K (\Vhat{h+1}{k} - V_{h+1}^\star)(x_{h+1}^k) \nonumber \\
		&\qquad + \sum_{k=1}^K (V_{h+1}^\star - V_{h+1}^{\pi^k})(x_{h+1}^k). \numberthis \label{eqn:recursive_proof_next}
\end{align*}
However, noticing that $V_{h+1}^{\pi^k}(x_{h+1}^k) \leq V_{h+1}^\star(x_{h+1}^k)$ we have
\begin{align*}
    & \left(1 + \frac{1}{H}\right) \sum_{k=1}^K (\Vhat{h+1}{k} - V_{h+1}^\star)(x_{h+1}^k)+ \sum_{k=1}^K (V_{h+1}^\star - V_{h+1}^{\pi^k})(x_{h+1}^k) \\
    & = \left(1 + \frac{1}{H}\right) \sum_{k=1}^K (\Vhat{h+1}{k} - V_{h+1}^\star)(x_{h+1}^k) + \sum_{k=1}^K(\Vhat{h+1}{k} - V_{h+1}^{\pi^k})(x_{h+1}^k) - (\Vhat{h+1}{k} - V_{h+1}^\star)(x_{h+1}^k) \\
    & = \frac{1}{H} \sum_{k=1}^K (\Vhat{h+1}{k} - V_{h+1}^\star)(x_{h+1}^k) + \sum_{k=1}^K (\Vhat{h+1}{k} - V_{h+1}^{\pi^k})(x_{h+1}^k) \\
    & \leq \left(1 + \frac{1}{H}\right) \sum_{k=1}^K (\Vhat{h+1}{k} - V_{h+1}^{\pi^k})(x_{h+1}^k).
\end{align*}
Substituting this back into Equation~\ref{eqn:recursive_proof_next} we get that
\begin{align*}
    \sum_{k=1}^K (\Vhat{h}{k} - V_h^{\pi^k})(x_h^k) \leq \sum_{k=1}^K (H \Ind{n_h^k = 0} + \beta_{n_h^k} + \xi_{h+1}^k) + \left(1 + \frac{1}{H}\right) \sum_{k=1}^K (\Vhat{h+1}{k} - V_{h+1}^{\pi^k})(x_{h+1}^k).
\end{align*}
\end{proof}

With the machinery in place we are now ready to show the general version of Theorem~\ref{thm:regret}.  We restate it here for convenience.

\begin{theorem}
	For any any sequence of initial states $\{x_1^k \mid k\in[K]\}$, and any $\pfail \in (0, 1)$ with probability at least $1 - \pfail$ Adaptive $Q$-learning (Alg~\ref{alg}) achieves regret guarantee: 
\begin{align*}
R(K) &\leq 3H^2 + 6 \sqrt{2H^3 K \log(4HK/\pfail)} \\
&\qquad + 96 H \left( \sqrt{H^3 \log(4HK/\pfail)} + L d_{max} \right) \inf_{r_0 \in (0, d_{max}]} \left( \sum_{\substack{r = d_{max} 2^{-i} \\ r \geq r_0}} N_r \frac{d_{max}}{r} + \frac{K r_0}{d_{max}} \right)
\end{align*}
	where $N_r$ is the $r$-covering number of $\S \times \A$ with respect to the metric $\D$.
\end{theorem}

\begin{proof}
By definition of the regret we have that $R(K) = \sum_{k=1}^K (V_1^\star(x_1^k) - V_1^{\pi^k}(x_1^k))$.  By Lemma~\ref{lemma:upper_lower_bounds_value_fn} we know with probability at least $1 - \pfail/2$ that for any $(x,a,h,k)$ that $\Vhat{h}{k}(x) - V_h^\star(x) \geq 0$.  Hence we have that $R(K) \leq \sum_{k=1}^K (\Vhat{1}{k}(x_1^k) - V_1^{\pi^k}(x_1^k))$.

The main idea of the rest of the proof is to upper bound $\sum_{k=1}^K (\Vhat{h}{k}(x_h^k) - V_h^{\pi^k}(x_h^k))$ by the next step $\sum_{k=1}^K (\Vhat{h+1}{k}(x_{h+1}^k) - V_{h+1}^{\pi^k}(x_{h+1}^k))$.  For any fixed $(h,k) \in [H] \times [K]$ let $B_h^k$ be the ball selected at episode $k$ step $h$ and $t = n_h^k(B)$.  Using Lemma~\ref{lemma:induction} we know for any $h \in [H]$ that
\begin{align*}
        \sum_{k=1}^K (\Vhat{h}{k} - V_h^{\pi^k})(x_h^k) \leq \sum_{k=1}^K (H \Ind{n_h^k = 0} + \beta_{n_h^k} + \xi_{h+1}^k) + \left(1 + \frac{1}{H}\right) \sum_{k=1}^K (\Vhat{h+1}{k} - V_{h+1}^{\pi^k})(x_{h+1}^k).
\end{align*}

For the first term we notice that 
\begin{align*}
	\sum_{k=1}^K \Ind{n_h^k = 0} H \leq H
\end{align*}
as the indicator is $1$ only when $k = 1$ for the first iteration when there is only a single ball covering the entire space.

Substituting this in and expanding this relationship out and using the fact that $\Vhat{H+1}{k}(x_{H+1}^k) - V_{H+1}^{\pi^k}(x_{h+1}^k) = 0$ as $V^k_{H+1} = V_{H+1}^{\pi^k} = 0$ gives
$$R(K) \leq \sum_{k=1}^K (V_1^k(x_1^k) - V_1^{\pi^k}(x_1^k)) \leq H \sum_{h=1}^H \left(1 + \frac{1}{H}\right)^{h-1} + \sum_{h=1}^H \left(1 + \frac{1}{H}\right)^{h-1}\sum_{k=1}^K (\beta_{n_h^k} + \xi_{h+1}^k).$$

Indeed, we can show this by induction on $H$.  For the case when $H = 1$ then we get (using that $\Vhat{2}{k}(x) = V_2^{\pi^k}(x) = 0$)
\begin{align*}
	\sum_{k=1}^K (V_1^k(x_1^k) - V_1^{\pi^k}(x_1^k)) & \leq H + \left(1 + \frac{1}{H}\right)\sum_{k=1}^K \left(\Vhat{2}{k}(x_2^k) - V_2^{\pi^k}(x_2^k)\right) + \sum_{k=1}^K \left(\beta_{n_1^k} + \xi_{2}^k\right) \\
	& = H \sum_{h=1}^1 \left(1 + \frac{1}{H}\right)^{h-1} + \sum_{h=1}^1 \left(1 + \frac{1}{H}\right)^{h-1}\sum_{k=1}^K \left(\beta_{n_1^k} + \xi_{2}^k\right).
\end{align*}
For the step case then
\begin{align*}
	&\sum_{k=1}^K V_1^k(x_1^k) - V_1^{\pi^k}(x_1^k) \\
	& \leq H + \left(1 + \frac{1}{H}\right)\sum_{k=1}^K \left(\Vhat{2}{k}(x_2^k) - V_2^{\pi^k}(x_2^k)\right) + \sum_{k=1}^K \left(\beta_{n_1^k} + \xi_{2}^k\right) \\
	& \leq H + \left(1 + \frac{1}{H}\right) \left( H \sum_{h=1}^{H-1} \left(1 + \frac{1}{H}\right)^{h-1} + \sum_{h=1}^{H-1} \left(1 + \frac{1}{H}\right)^{h-1}\sum_{k=1}^K \left( \beta_{n_{h+1}^k} + \xi_{h+2}^k\right) \right) + \sum_{k=1}^K \left(\beta_{n_1^k} + \xi_{2}^k\right) \\
	& = H \sum_{h=1}^H \left(1 + \frac{1}{H}\right)^{h-1} + \sum_{h=1}^H \left(1 + \frac{1}{H}\right)^{h-1}\sum_{k=1}^K \left( \beta_{n_h^k} + \xi_{h+1}^k \right).
\end{align*}
Moreover, noticing that $\sum_{h=1}^H \left(1 + \frac{1}{H}\right)^{h-1} \leq 3H$ and that $\left(1 + \frac{1}{H}\right)^{h-1} \leq \left(1 + \frac{1}{H}\right)^{H} \leq 3$ we get
$$\sum_{k=1}^K \left(V_1^k(x_1^k) - V_1^{\pi^k}(x_1^k)\right) \leq 3H^2 + 3\sum_{h=1}^H \sum_{k=1}^K \left( \beta_{n_h^k} + \xi_{h+1}^k \right).$$

However, $\sum_{k=1}^K \sum_{h=1}^H \xi_{h+1}^k \leq 2\sqrt{2H^3 K \log(4HK/\pfail)}$ with probability at least $1 - \pfail/2$ by Lemma~\ref{lemma:martingale_difference_sequence}.  Hence by with probability $1 - \pfail$ by combining the two high probability guarantees we have
$$\sum_{k=1}^K \left(V_1^k(x_1^k) - V_1^{\pi^k}(x_1^k)\right) \leq 3H^2 + 6 \sqrt{2H^3 K \log(4HK/\pfail)} + 3\sum_{h=1}^H \sum_{k=1}^K \beta_{n_h^k}.$$
We use Lemma \ref{lemma:bound_on_beta_summation} to bound $\sum_{k=1}^K \beta_{n_h^k}$.

Combing all of the pieces in the final regret bound we get with probability at least $1 - \pfail$ that
\begin{align*}
 R(K) &\leq \sum_{k=1}^K (V_1^k(x_1^k) - V_1^{\pi^k}(x_1^k)) \\
&\leq 3H^2 + 6 \sqrt{2H^3 K \log\big(4HK/\pfail\big)} + 3\sum_{h=1}^H \sum_{k=1}^K \beta_{n_h^k}\\
	& \leq 3H^2 + 6\sqrt{2H^3 K \log(4HK/\pfail)} \\
	&\qquad + 3 \sum_{h=1}^H  16\left( \sqrt{H^3 \log(4HK/\pfail)} + L d_{max} \right) \inf_{r_0 \in (0, d_{max}]} \left(\sum_{\substack{r = d_{max} 2^{-i} \\ r \geq r_0}} 2 N_r \frac{d_{max}}{r} + \frac{2 K r_0}{ d_{max}}\right)  \\
	& \leq 3H^2 + 6 \sqrt{2H^3 K \log(4HK/\pfail)} \\
	&\qquad + 96 H \left( \sqrt{H^3 \log(4HK/\pfail)} + L d_{max} \right) \inf_{r_0 \in (0, d_{max}]} \left( \sum_{\substack{r = d_{max} 2^{-i} \\ r \geq r_0}} N_r \frac{d_{max}}{r} + \frac{K r_0}{d_{max}} \right)
\end{align*}
To recover the term established in Theorem~\ref{thm:regret} we simply take $r_0 = O\left( K^{\frac{-1}{d_c+2}}\right)$.
\end{proof}

\clearpage
\section{Experimental Results and Figures}
\label{section:full_experiments}

\subsection{Oil Problem with Quadratic Survey Function}

\begin{figure*}[h!]
\centering
\includegraphics[scale=0.65]{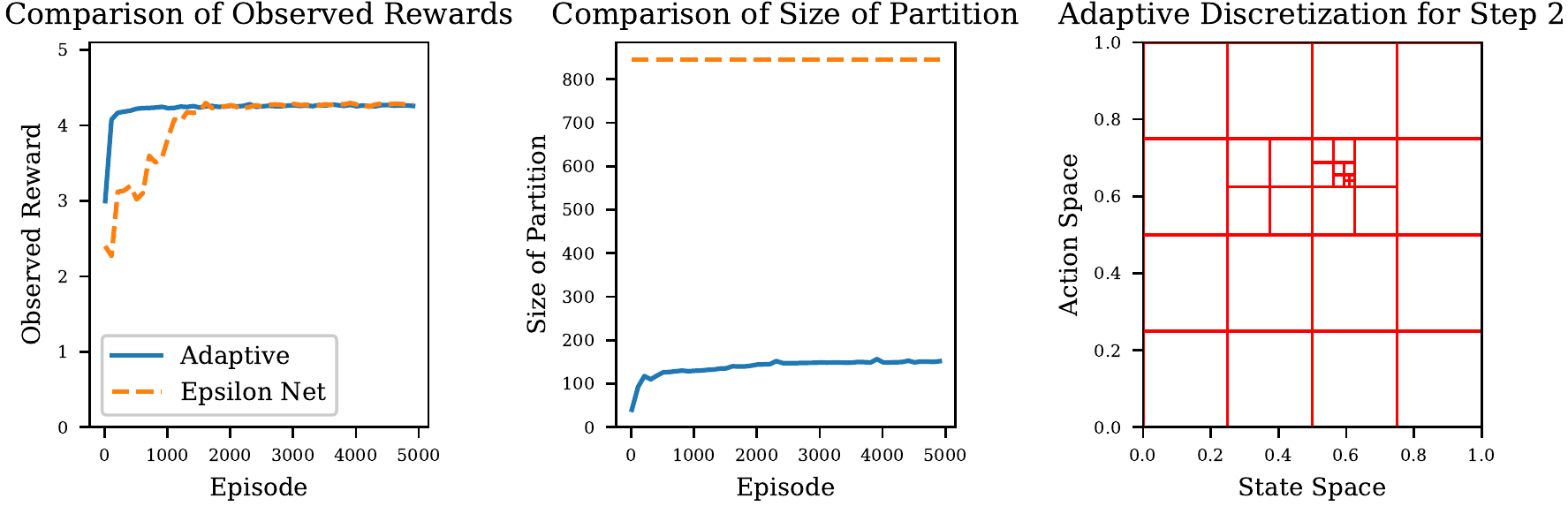}
\caption{
Comparison of the algorithms on the oil problem with quadratic survey function.  The transition kernel is $\Pr_h(x' \mid x,a) = \Ind{x' = a}$ and reward function is $r(x,a) = (1 - (a-0.75)^2 - |x - a|)_+$ (see Section~\ref{section:experiments_set_up}).  The adaptive algorithm creates a fine partition of the space around the optimal point of $0.75$.
}
\end{figure*}

\begin{figure*}[h!]
\centering
\includegraphics[scale=0.65]{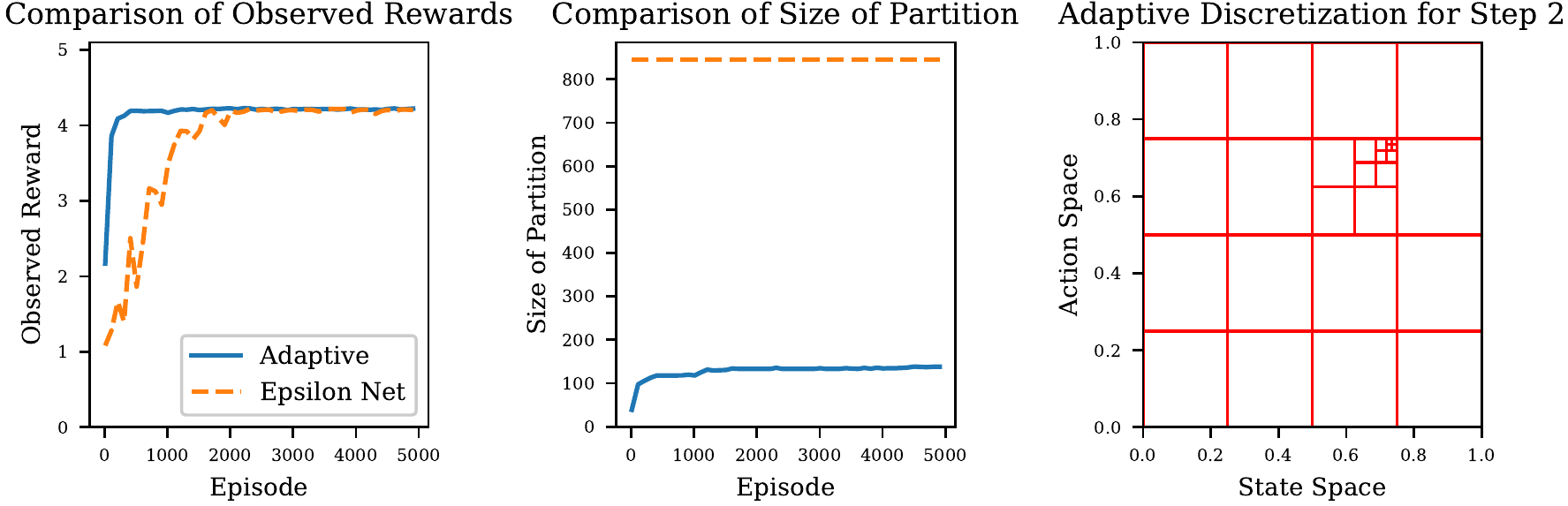}
\caption{
Comparison of the algorithms on the oil problem with quadratic survey function.  The transition kernel is $\Pr_h(x' \mid x,a) = \Ind{x' = a}$ and reward function is $r(x,a) = (1 - 10(a-0.75)^2 - |x - a|)_+$ (see Section~\ref{section:experiments_set_up}).  The epsilon net algorithm suffers by exploring more parts of the space.
}
\end{figure*}

\begin{figure*}[h!]
\centering
\includegraphics[scale=0.65]{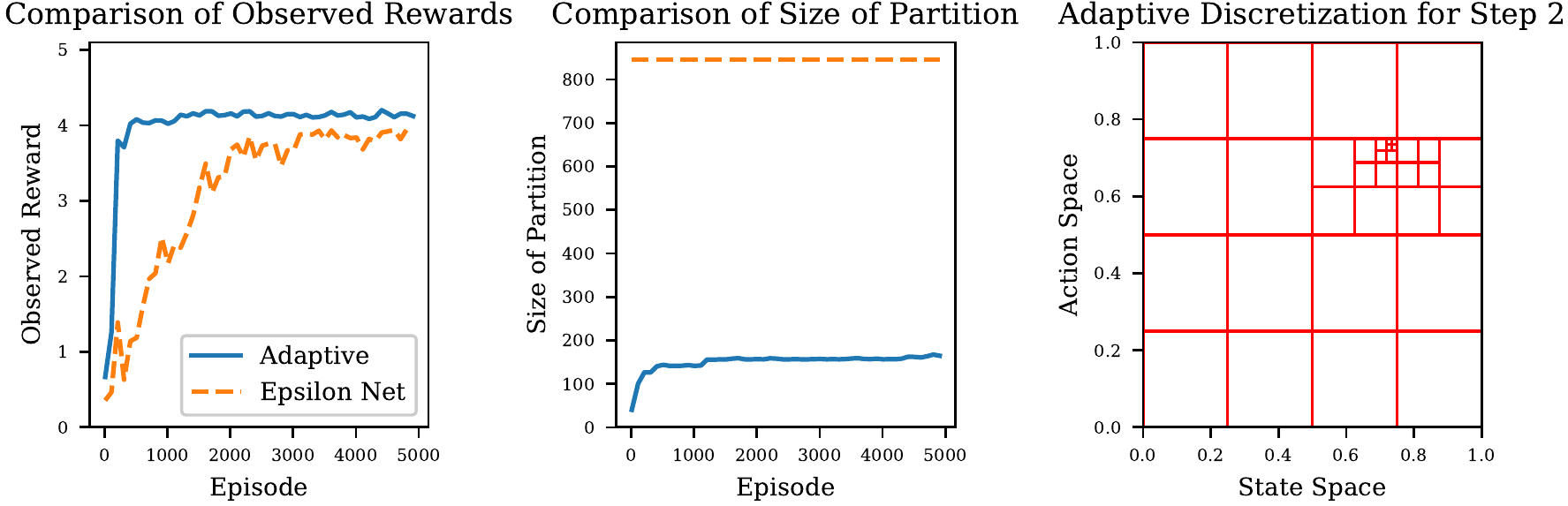}
\caption{
Comparison of the algorithms on the oil problem with quadratic survey function.  The transition kernel is $\Pr_h(x' \mid x,a) = \Ind{x' = a}$ and reward function is $r(x,a) = (1 - 50(a-0.75)^2 - |x - a|)_+$ (see Section~\ref{section:experiments_set_up}).  The epsilon net algorithm suffers a fixed discretization error as the mesh is not fine enough to capture the peak reward.
}
\end{figure*}

\FloatBarrier
\clearpage
\subsection{Oil Problem with Laplace Survey Function}

\begin{figure*}[th!]
\centering
\includegraphics[scale=0.65]{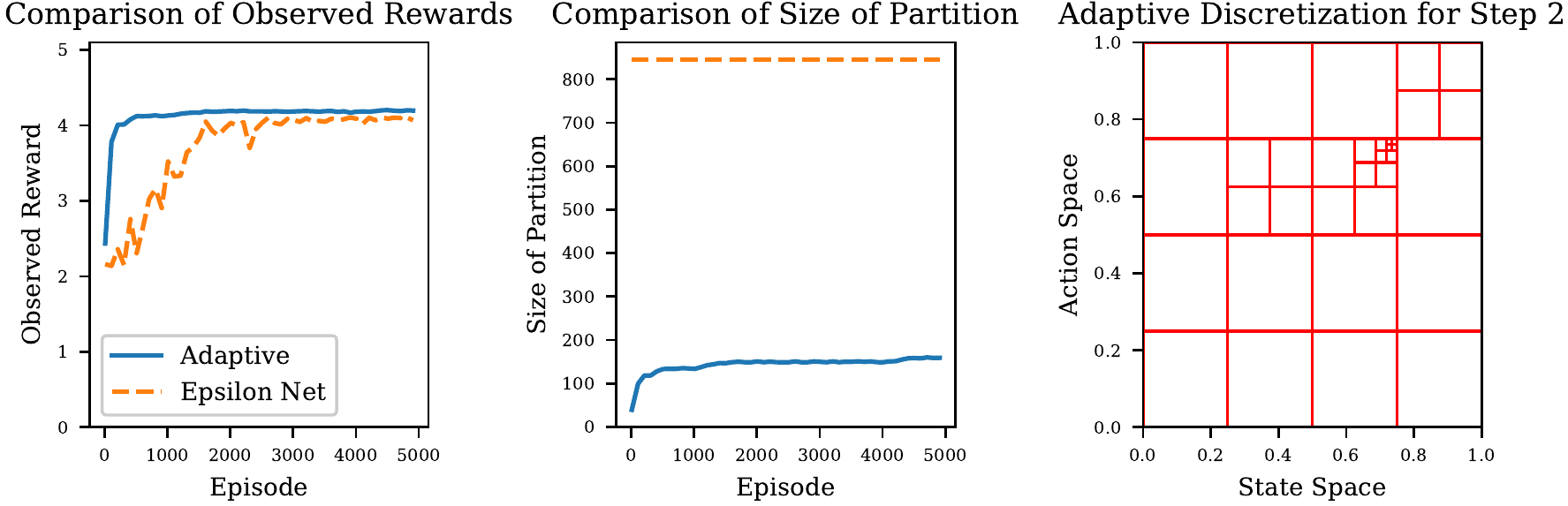}
\caption{\textit{(Duplicate of Figure~\ref{fig:oil_laplace})}  Comparison of the algorithms on the oil discovery problem with survey function $f(x) = e^{- |x - 0.75|}$.  The transition kernel is $\Pr_h(x' \mid x,a) = \Ind{x' = a}$ and reward function is $r(x,a) = (1 - e^{-|a - 0.75|} - |x - a|)_+$ (see Section~\ref{section:experiments_set_up}).  The adaptive algorithm quickly learns the location of the optimal point $0.75$ and creates a fine partition of the space around the optimal.
}
\label{fig:oil_laplace_repeat}
\end{figure*}

\begin{figure*}[ht!]
\centering

\includegraphics[scale=0.65]{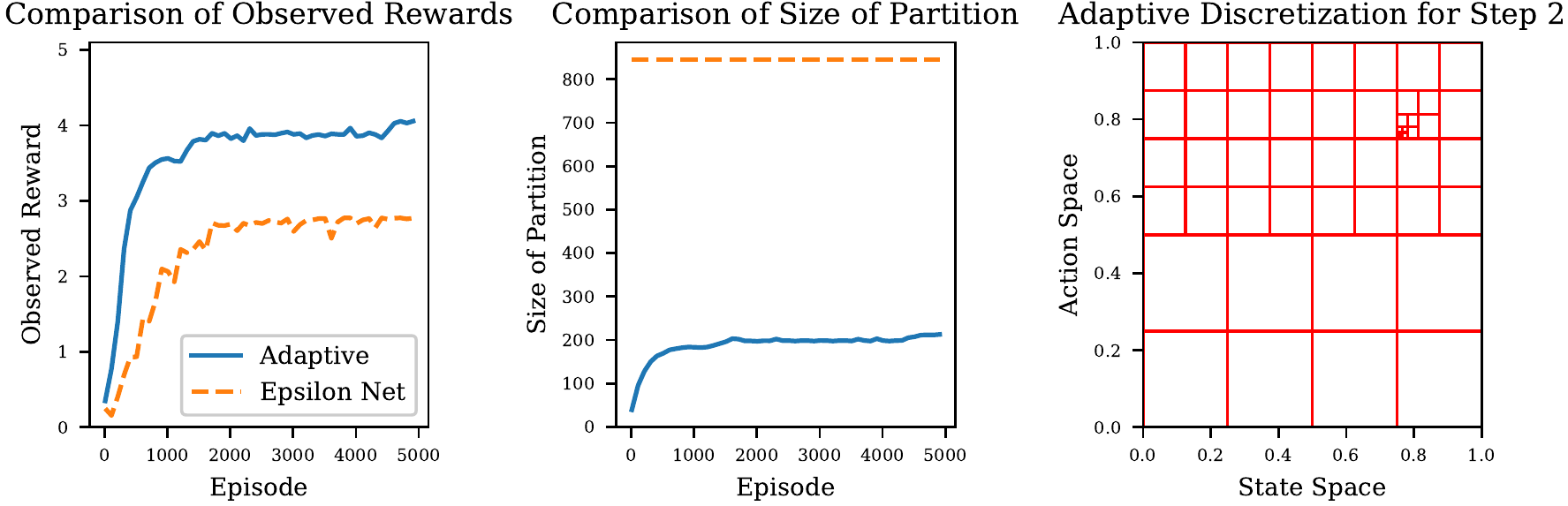}

\caption{Comparison of the algorithms on the oil problem with Laplace survey function.  The transition kernel is $\Pr_h(x' \mid x,a) = \Ind{x' = a}$ and reward function is $r(x,a) = (1 - 10e^{-|a - 0.75|} - |x - a|)_+$ (see Section~\ref{section:experiments_set_up}).  The adaptive algorithm learns a fine partition around the optimal point $0.75$ while the $\epsilon$-Net algorithm suffers from a large discretization error.
}
\end{figure*}

\begin{figure*}[ht!]
\centering

\includegraphics[scale=0.65]{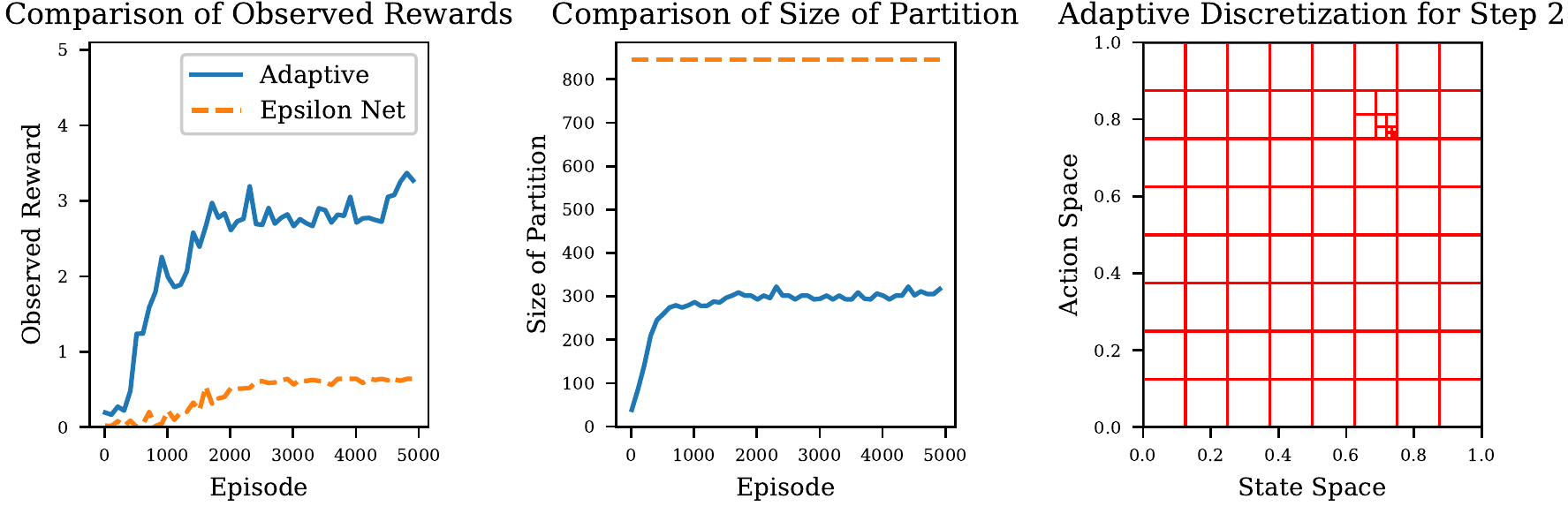}

\caption{Comparison of the algorithms on the oil problem with Laplace survey function.  The transition kernel is $\Pr_h(x' \mid x,a) = \Ind{x' = a}$ and reward function is $r(x,a) = (1 - 50e^{-|a - 0.75|} - |x - a|)_+$ (see Section~\ref{section:experiments_set_up}).  The epsilon net algorithm suffers from a large discretization error as the mesh is not fine enough to capture the high reward region.
}
\end{figure*}

\FloatBarrier
\clearpage
\subsection{Ambulance Problem with Uniform Arrivals}

\begin{figure*}[th!]
\centering
\includegraphics[scale=0.65]{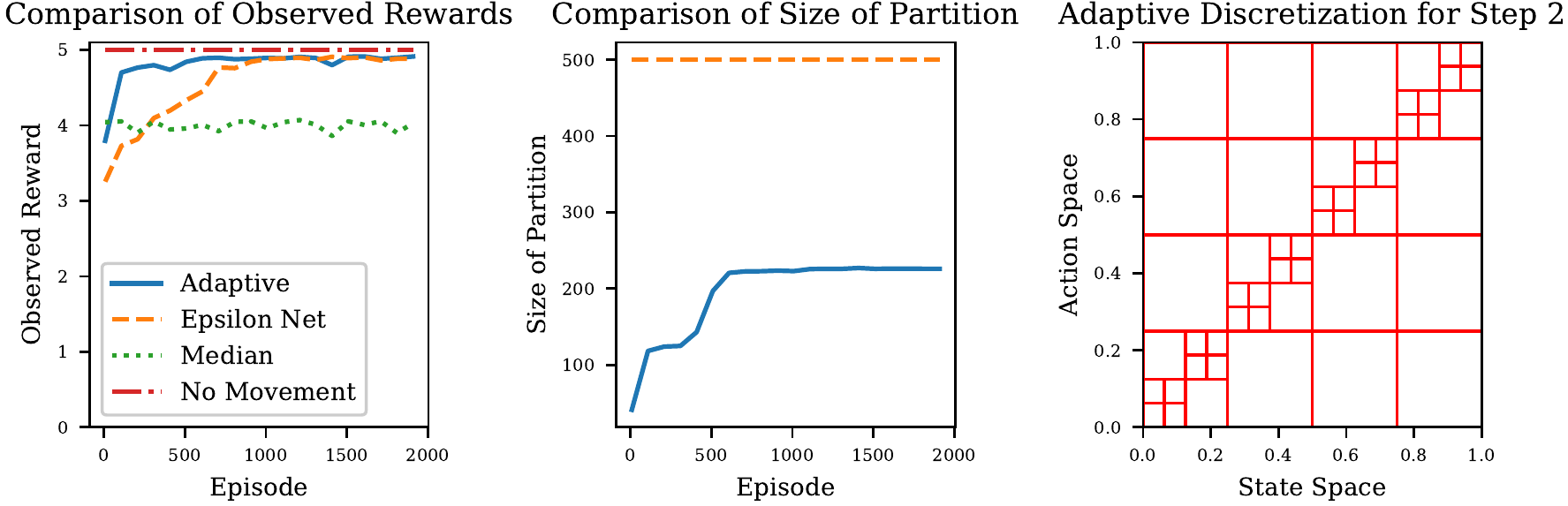}
\caption{
Comparison of the algorithms on the ambulance problem with Uniform$(0,1)$ arrivals and reward function $r(x, a) = 1 - |x - a|$ (see Section~\ref{section:experiments_set_up}).  Clearly, the no movement heuristic is the optimal policy but the adaptive algorithm learns a fine partition across the diagonal of the space where the optimal policy lies.
}
\end{figure*}

\begin{figure*}[th!]
\centering
\includegraphics[scale=0.65]{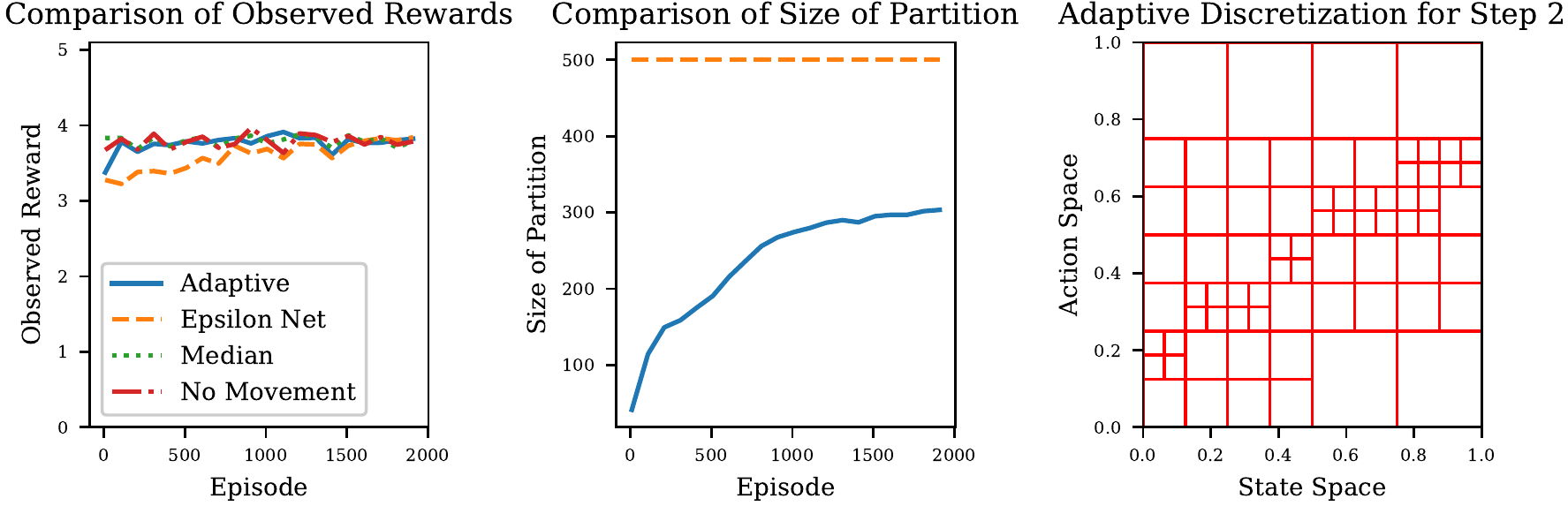}
\caption{
Comparison of the algorithms on the ambulance problem with Uniform$(0,1)$ arrival distribution and reward function $r(x',x, a) = 1 - .25|x - a|-.75|x'-a|$ (see Section~\ref{section:experiments_set_up}).  All algorithms perform sub optimally but the adaptive algorithm is able to learn a mixed policy between no movement and traveling to the median.
}
\end{figure*}

\begin{figure*}[th!]
\centering
\includegraphics[scale=0.65]{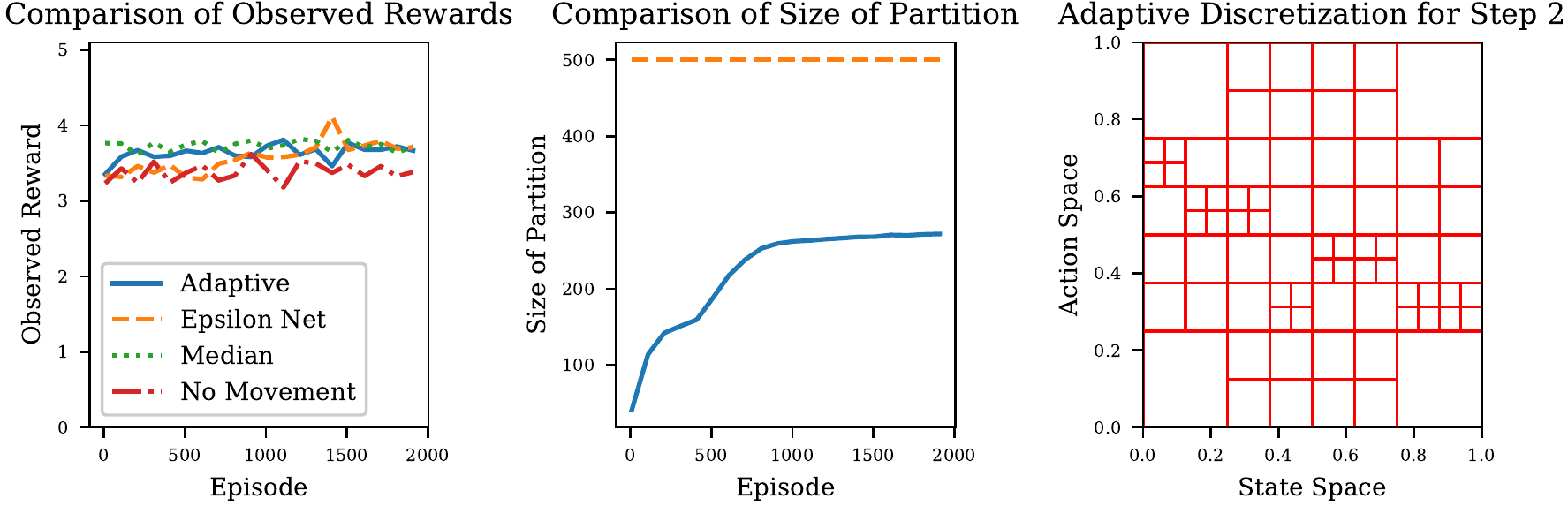}
\caption{
Comparison of the algorithms on the ambulance problem with Uniform$(0,1)$ arrival distribution and reward function $r(x',x, a) = 1 - |x'-a|$ (see Section~\ref{section:experiments_set_up}).   The median policy performs best, and the adaptive algorithm is beginning to learn a finer partition around the median of the arrival distribution ($0.5$).
}
\end{figure*}

\FloatBarrier
\clearpage

\subsection{Ambulance Problem with Beta Arrivals}

\begin{figure*}[th!]
\centering
\includegraphics[scale=0.65]{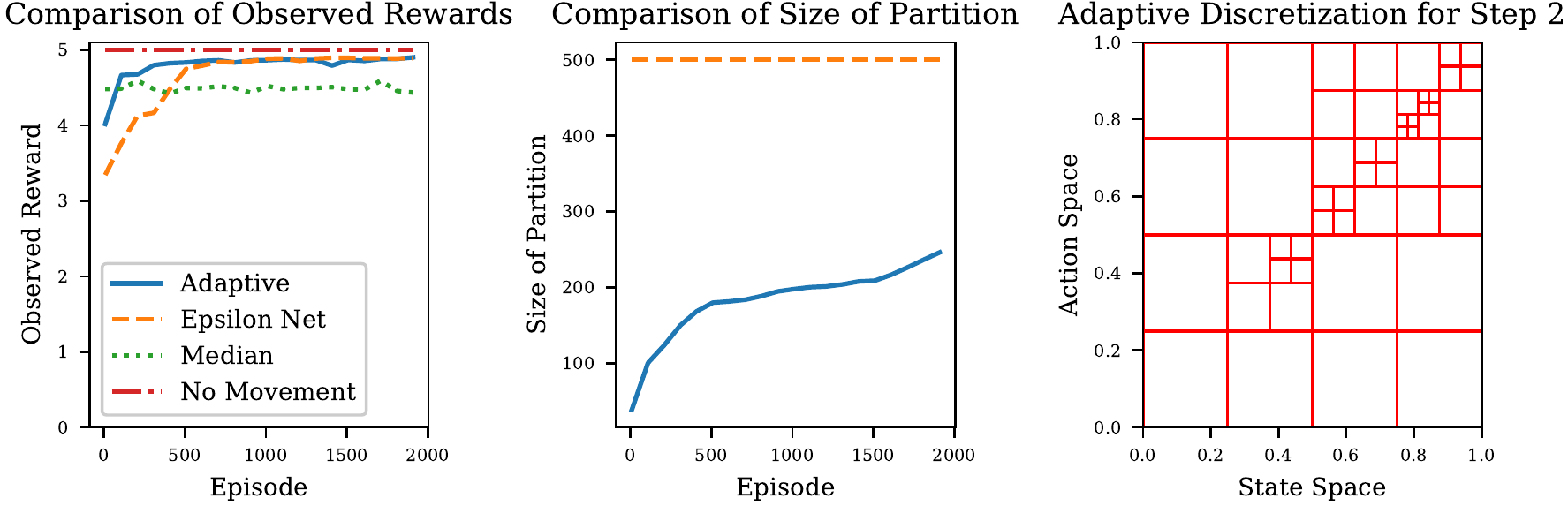}
\caption{
\textit{(Duplicate of Figure~\ref{fig:ambulance_beta})}  Comparison of the algorithms on the ambulance problem with Beta$(5,2)$ arrival distribution and reward function $r(x, a) = 1 - |x - a|$ (see Section~\ref{section:experiments_set_up}).  Clearly, the no movement heuristic is the optimal policy but the adaptive algorithm learns a fine partition across the diagonal of the space where the optimal policy lies.
}
\label{fig:ambulance_beta_duplicate}
\end{figure*}

\begin{figure*}[th!]
\centering
\includegraphics[scale=0.65]{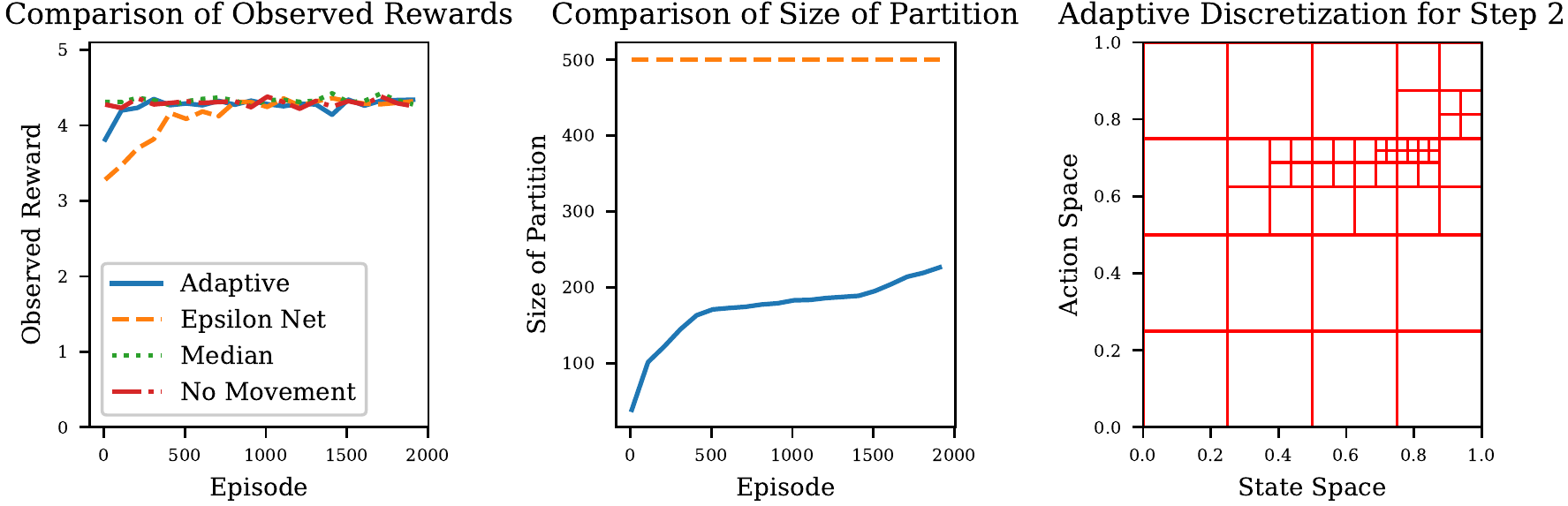}
\caption{
Comparison of the algorithms on the ambulance problem with Beta$(5,2)$ arrival distribution and reward function $r(x',x, a) = 1 - .25|x - a|-.75|x'-a|$ (see Section~\ref{section:experiments_set_up}).  All algorithms perform sub optimally but the adaptive algorithm is able to learn a mixed policy of traveling to the median.
}
\end{figure*}

\begin{figure*}[th!]
\centering
\includegraphics[scale=0.65]{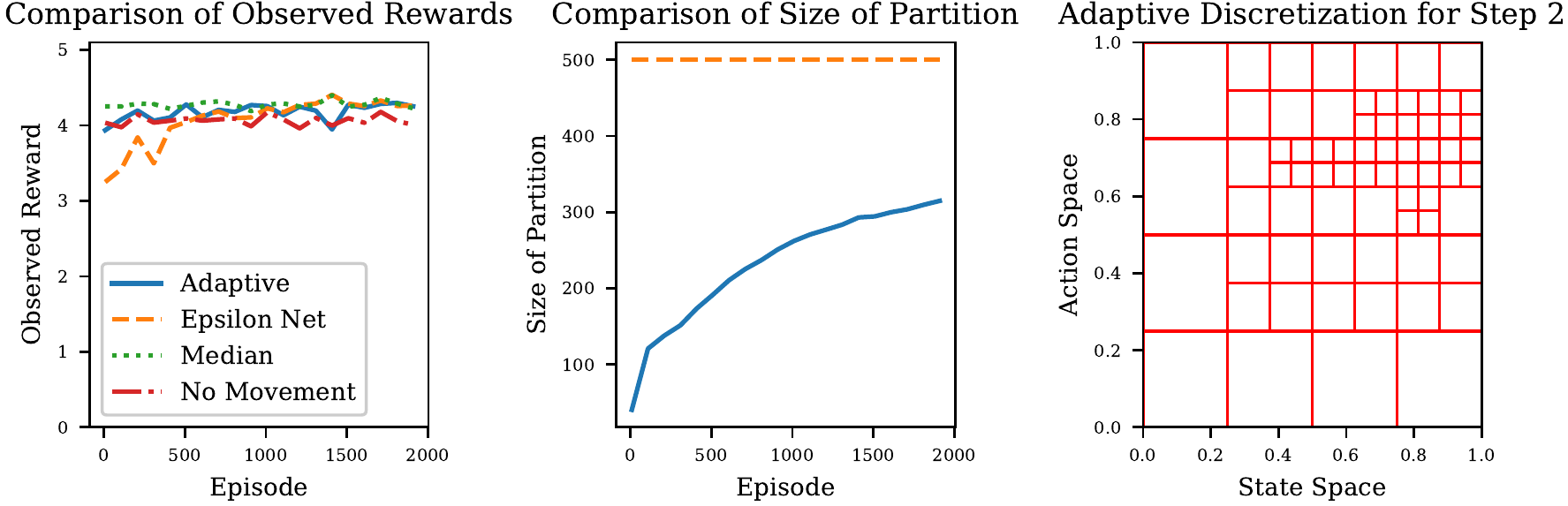}
\caption{
Comparison of the algorithms on the ambulance problem with Beta$(5,2)$ arrival distribution and reward function $r(x',x, a) = 1 - |x'-a|$ (see Section~\ref{section:experiments_set_up}).  The median policy performs best, and the adaptive algorithm is beginning to learn a finer partition around the median of the arrival distribution ($\simeq 0.7$).
}
\end{figure*}

\FloatBarrier
	
\end{document}